\documentclass{article}

     \usepackage[final]{neurips_2020}

\usepackage[utf8]{inputenc} 
\usepackage[T1]{fontenc}    
\usepackage{hyperref}       
\usepackage{url}            
\usepackage{booktabs}       
\usepackage{amsfonts}       
\usepackage{nicefrac}       
\usepackage{microtype}      
\usepackage{color}
\usepackage{algorithm}
\usepackage{algorithmic}
\usepackage{amsmath}
\usepackage{amsthm}
\usepackage{mathabx}
\usepackage{graphicx}
\usepackage{wrapfig}
\usepackage{natbib}
\usepackage{wrapfig}
\usepackage{multicol}
\title{Sinkhorn Natural Gradient for Generative Models}

\author{%
	Zebang Shen$^*$ \quad Zhenfu Wang$^\dagger$ \quad Alejandro Ribeiro$^*$  \quad Hamed Hassani$^*$\\
	$^*$Department of Electrical and Systems Engineering \quad $^\dagger$Department of Mathematics\\
	University of Pennsylvania\\
	\texttt{\{zebang@seas,zwang423@math,aribeiro@seas,hassani@seas\}.upenn.edu}
}

\newcommand{\OTgamma}{{\mathrm{OT}_\gamma}}
\newcommand{\OT}{{\mathrm{OT}}}

\DeclareMathAlphabet\mathbfcal{OMS}{cmsy}{b}{n}

\title{Sinkhorn Natural Gradient for Generative Models}
\input{required.tex}

\begin{document}
\maketitle

\begin{abstract}
	We consider the problem of minimizing a functional over a parametric family of probability measures, where the parameterization is characterized via a push-forward structure. 
	An important application of this problem is in training generative adversarial networks.  
	In this regard, we propose a novel Sinkhorn Natural Gradient (SiNG) algorithm which acts as a steepest descent method on the probability space endowed with the Sinkhorn divergence.
	We show that the Sinkhorn information matrix (SIM), a key component of SiNG, has an explicit expression and can be evaluated accurately in complexity that scales logarithmically with respect to the desired accuracy. This is in sharp contrast to  existing natural gradient methods that can only be carried out approximately.
	Moreover, in practical applications when only Monte-Carlo type integration is available, we design an empirical estimator for SIM and provide the stability analysis.
	In our experiments, we quantitatively compare SiNG with state-of-the-art SGD-type solvers on generative tasks to demonstrate its efficiency and efficacy of our method.
\end{abstract}

\section{Introduction} \label{section_introuction}
Consider the minimization of a functional $\FM$ over a parameterized family probability measures 
$\{\alpha_\theta\}$:
\begin{equation}\label{eqn_main}
	\min_{\theta\in\Theta} \left\{ F(\theta) \defi \FM(\alpha_\theta) \right\},
\end{equation}
where $\Theta\subseteq\RBB^d$ is the feasible domain of the parameter $\theta$.  
We assume that the measures $\alpha_\theta$ are defined over  a common ground set $\XM\subseteq\RBB^q$ with the following structure:  $\alpha_\theta = {T_\theta}_\sharp\mu$, where $\mu$ is a fixed and known measure and $T_\theta$ is a push-forward mapping. 
More specifically, $\mu$ is a simple measure on a latent space $\ZM\subseteq\RBB^{\bar q}$, such as the standard Gaussian measure $\mu = \NM(\zeroB_{\bar{q}}, \IB_{\bar{q}})$, and the parameterized map $T_\theta:\ZM\rightarrow\XM$ transforms the measure $\mu$ to $\alpha_\theta$. This type of push-forward  parameterization is commonly used in deep generative models, where $T_\theta$  represents a neural network parametrized by weights $\theta$ \citep{goodfellow2014generative,salimans2018improving,genevay2018learning}.
Consequently, methods to efficiently and accurately solve problem~\eqref{eqn_main} are of great importance in machine learning.

The de facto solvers for problem \eqref{eqn_main} are generic nonconvex optimizers such as Stochastic Gradient Descent (SGD) and its variants, Adam \citep{kingma2014adam}, Amsgrad \citep{reddi2019convergence}, RMSProp \citep{hinton2012neural}, etc.
These optimization algorithms directly work on the parameter space and are agnostic to the fact that $\alpha_\theta$'s are probability measures.  Consequently, SGD type solvers suffer from the complex optimization landscape induced from the neural-network mappings $T_\theta$.


An alternative to SGD type methods is the natural gradient method, which is originally motivated from Information Geometry \citep{amari1998natural,amari1987differential}. Instead of simply using the Euclidean structure of the parameter space $\Theta$ in the usual SGD, the natural gradient method endows the parameter space with a ``natural" metric structure by pulling back a known metric on the probability space and then searches the steepest descent direction of $F(\theta)$ in the ``curved" neighborhood of $\theta$. In particular,  the natural gradient update is invariant to reparametrization. 
This allows natural gradient to avoid the undesirable saddle point or local minima that are artificially created by the highly nonlinear maps $T_\theta$.
The classical Fisher-Rao Natural Gradient (FNG) \citep{amari1998natural} as well as its many variants \citep{martens2015optimizing,thomas2016energetic,song2018accelerating} endows the probability space with the KL divergence and admits update direction in closed form.
However, the update rules of these methods all require the evaluation of the score function of the variable measure.
Leaving aside its existence, this quantity is in general difficult to compute for push-forward measures, which limits the application of FNG type methods in the generative models.
Recently, \citet{li2018natural} propose to replace the KL divergence in FNG by the Wasserstein distance and propose the Wasserstein Natural Gradient (WNG) algorithm.
WNG shares the merit of reparameterization invariance as  FNG while avoiding the requirement of the score function.
However, the Wasserstein information matrix (WIM) is very difficult to compute as it does not attain a closed form expression when the dimension $d$ of parameters is greater than 1, rendering WNG impractical.



Following the line of natural gradient, in this paper, we propose Sinkhorn Natural Gradient (SiNG), an algorithm that performs the steepest descent of the objective functional $\FM$ on the probability space with the Sinkhorn divergence as the underlying metric.
Unlike FNG, SiNG requires only to sample from the variable measure $\alpha_\theta$.
Moreover, the Sinkhorn information matrix (SIM), a key component in SiNG, can be computed in logarithmic time in contrast to WIM in WNG.
Concretely, we list our contributions as follows:

\begin{enumerate}
	\item We derive the Sinkhorn Natural Gradient (SiNG) update rule as the exact direction that minimizes the objective functional $\FM$ within the Sinkhorn ball of radius $\epsilon$ centered at the current measure. 
	In the asymptotic case $\epsilon\rightarrow 0$, we show that the SiNG direction only depends on the Hessian of the Sinkhorn divergence and the gradient of the function $F$, while the effect of the Hessian of $F$ becomes negligible.
	Further, we prove that SiNG is invariant to reparameterization in its continuous-time limit (i.e. using the infinitesimal step size).
	\item We explicitly derive the expression of the Sinkhorn information matrix (SIM), i.e. the Hessian of the Sinkhorn divergence with respect to the parameter $\theta$.
	We then show the SIM can be computed using logarithmic (w.r.t. the target accuracy) function operations and integrals with respect to $\alpha_\theta$.
	\item When only Monte-Carlo integration w.r.t. $\alpha_\theta$ is available, we propose to approximate SIM with its empirical counterpart (eSIM), i.e. the Hessian of the empirical Sinkhorn divergence. Further, we prove stability of eSIM. Our analysis relies on the fact that the Fr\'echet derivative of Sinkhorn potential with respect to the parameter $\theta$ is continuous with respect to the underlying measure $\mu$. Such result can be of general interest.
\end{enumerate}



In our experiments, we pretrain the discriminators for the celebA and cifar10 datasets.
Fixing the discriminator, we compare SiNG with state-of-the-art SGD-type solvers in terms of the generator loss.
The result shows the remarkable superiority of SiNG in both efficacy and efficiency.

\vspace{-.1cm}
\noindent{\bf Notation:}
Let $\XM\subseteq\RBB^q$ be a compact ground set.
We use $\MM_1^+(\XM)$ to denote the space of probability measures on $\XM$ and use $\CM(\XM)$ to denote the family of continuous functions mapping from $\XM$ to $\RBB$.
For a function $f\in\CM(\XM)$, we denote its $L^\infty$ norm by $\|f\|_\infty \defi \max_{x\in\XM} |f(x)|$ and its gradient by $\nabla f$.\\
For a functional on general vector spaces, the Fr\'echet derivative is formally defined as follows.
Let V and W be normed vector spaces, and $U\subseteq V$ be an open subset of $V$.
A function $\FM:U \rightarrow W$ is called Fr\'echet differentiable at $x\in U$ if there exists a bounded linear operator $A:V\to W$ such that
\begin{equation} \label{definition_frechet_derivative}
\lim _{\|h\|\to 0}{\frac {\|\FM(x+h)-\FM(x)-Ah\|_{W}}{\|h\|_{V}}}=0.
\end{equation}
If there exists such an operator $A$, it will be unique, so we denote $D\FM(x) = A$ and call it the \emph{Fr\'echet derivative}.
From the above definition, we know that $D \FM :U\rightarrow T(V, W)$ where $T(V, W)$ is the family of bounded linear operators from $V$ to $W$.
Given $x \in U$, the linear map $D\FM(x)$ takes one input $y\in V$ and outputs $z \in W$. This is denoted by $z = D\FM(x)[y]$.
We then define the operator norm of $D\FM$ at $x$ as $\|D\FM(x)\|_{op}\defi\max_{h\in V}\frac{\|D \FM (x)[h]\|_W}{\|h\|_V}$.
Further, the second-order Fr\'echet derivative of $\FM$ is denoted as $D^2\FM:U\rightarrow L^2(V\times V, W)$, where $L^2(V\times V, W)$ is the family of all continuous bilinear maps from $V$ to $W$.
Given $x \in U$, the bilinear map $D^2\FM(x)$ takes two inputs $y_1, y_2 \in V$ and outputs $z \in W$. We denote this by $z = D^2\FM(x)[y_1, y_2]$.
If a function $\FM$ has multiple variables, we use $D_if$ to denote the Fr\'echet derivative with its $i^{th}$ variable and use $D^2_{ij}\FM$ to denote the corresponding second-order terms.
Finally, $\circ$ denotes the composition of functions.

\section{Related Work on Natural Gradient}
The Fisher-Rao natural gradient (FNG) \citep{amari1998natural} is a now classical algorithm for the functional minimization over a class of parameterized probability measures.
However, unlike SiNG, 
FNG as well as its many variants \citep{martens2015optimizing,thomas2016energetic,song2018accelerating} requires to evaluate the score function $\nabla_{\theta} \log p_{\theta}$ ($p_{\theta}$ denotes the p.d.f. of $\alpha_\theta$).
Leaving aside its existence issue, the score function for the generative model $\alpha_\theta$ is difficult to compute as it involves $T_\theta^{-1}$, the inversion of the push-forward mapping, and $\det(J{T_\theta^{-1}})$, the determinant of the Jacobian of $T_\theta^{-1}(z)$.
One can possibly recast the computation of the score function as a dual functional minimization problem over all continuous functions on $\XM$ \citep{essid2019adaptive}. 
However, such functional minimization problem itself is difficult to solve.
As a result, FNG has limited applicability in our problem of interest.

Instead of using the KL divergence, \citet{li2018natural} propose to measure the distance between (discrete) probability distributions using the optimal transport and develop the Wasserstein Natural Gradient (WNG). WNG inherits FNG's merit of reparameterization invariance. However, WNG requires to compute the Wasserstein information matrix (WIM), which does not attain a closed form expression when $d>1$, rendering WNG impractical \citep{li2019wasserstein, li2020ricci}.
As a workaround, one can recast a single WNG step to a dual functional maximization problem via the Legendre duality.
While itself remains challenging and can hardly be globally optimized, \citet{li2019affine} simplify the dual subproblem by restricting the optimization domain to an affine space of functions (a linear combinations of several bases).
Clearly, the quality of this solver depends heavily on the accuracy of this affine approximation.
Alternatively, \citet{arbel2019kernelized} restrict the dual functional optimization to a Reproducing Kernel Hilbert Space (RKHS).
By adding two additional regularization terms, the simplified dual subproblem admits a closed form solution.
However, in this way, the gap between the original WNG update and its kernelized version cannot be properly quantified without overstretched assumptions.

\section{Preliminaries} \label{section_preliminary}
We first introduce the entropy-regularized optimal transport distance and then its debiased version, i.e. the Sinkhorn divergence.
Given two probability measures $\alpha, \beta \in \MM_1^+(\XM)$, the entropy-regularized optimal transport distance  $\OTgamma(\alpha, \beta):\MM_1^+(\XM)\times\MM_1^+(\XM)\rightarrow\RBB_+$ is defined as 
\begin{equation}
\OTgamma(\alpha, \beta) = \min_{\pi\in\Pi(\alpha, \beta)} \langle c, \pi\rangle + \gamma \rm{KL}(\pi||\alpha\otimes\beta).
\label{eqn_OTepsilon}
\end{equation}
Here, $\gamma> 0$ is a fixed regularization parameter, $\Pi(\alpha, \beta)$ is the set of joint distributions over $\XM^2$ with marginals $\alpha$ and $\beta$, and we use $\langle c, \pi\rangle$ to denote $\langle c, \pi\rangle = \int_{\XM^2} c(x, y)\dB\pi(x, y)$.
We also use $\rm{KL}(\pi||\alpha\otimes\beta)$ to denote the \emph{Kullback-Leibler divergence} between the candidate transport plan $\pi$ and the product measure $\alpha\otimes\beta$.

Note that $\OTgamma(\alpha, \beta)$ is not a valid metric as there exists $\alpha \in \MM_1^+(\XM)$ such that $\OTgamma(\alpha, \alpha)\neq 0$ when $\gamma\neq 0$.
To remove this bias, consider the \emph{Sinkhorn divergence} $\SM(\alpha, \beta):\MM_1^+(\XM)\times\MM_1^+(\XM)\rightarrow\RBB_+$ introduced in \cite{peyre2019computational}:
\begin{equation} \label{s-o}
\SM(\alpha, \beta) \defi \OTgamma(\alpha, \beta) - \frac{\OTgamma(\alpha, \alpha)}{2} - \frac{\OTgamma(\beta, \beta)}{2},
\end{equation}
which can be regarded as a \emph{debiased version} of $\OTgamma(\alpha, \beta)$.
Since $\gamma$ is fixed throughout this paper, we omit the subscript $\gamma$ for simplicity.
It has been proved that $\SM(\alpha, \beta)$ is nonnegative, bi-convex and metrizes the convergence in law for a compact $\XM$ and a Lipschitz metric $c$ \cite{peyre2019computational}.
\paragraph{The Dual Formulation and Sinkhorn Potentials.} 
The entropy-regularized optimal transport problem $\OTgamma$, given in \eqref{eqn_OTepsilon}, is convex with respect to the joint distribution $\pi$: Its objective is a sum of a linear functional and the convex KL-divergence, and the feasible set $\Pi(\alpha, \beta)$ is convex.
Consequently, there is no gap between the primal problem \eqref{eqn_OTepsilon} and its Fenchel dual.
Specifically, define 
\begin{equation} \label{eqn_OTepsilon_dual_two_variable}
	\HM_2(f, g; \alpha, \beta) \defi \langle f, \alpha\rangle 
	+ \langle g, \beta\rangle  - \gamma\langle \exp(\frac{1}{\gamma}(f\oplus g - c)) - 1, \alpha\otimes\beta\rangle,
\end{equation}
where we denote $\big(f\oplus g\big)(x, y) = f(x) + g(y)$.
We have 
\begin{equation}	\label{eqn_OTepsilon_dual}
\OTgamma(\alpha, \beta) = \max_{f, g\in\CM(\XM)} \bigl\{\HM_2(f, g; \alpha, \beta) \bigr\} = \langle f_{\alpha, \beta}, \alpha\rangle +  \langle g_{\alpha, \beta}, \beta\rangle,
\end{equation}
where $f_{\alpha, \beta}$ and  $g_{\alpha, \beta}$, called the \emph{Sinkhorn potentials} of $\OTgamma(\alpha, \beta)$, are the maximizers of \eqref{eqn_OTepsilon_dual}. 


\paragraph{Training Adversarial Generative Models.}
We briefly describe how \eqref{eqn_main} captures the generative adversarial model (GAN):
In training a GAN, the objective functional in \eqref{eqn_main} itself is defined through a maximization subproblem
$\FM(\alpha_\theta) = \max_{\xi\in\Xi} \GM(\xi; \alpha_{\theta})$.
Here $\xi\in\Xi\subseteq\RBB^{\bar d}$ is some dual adversarial variable encoding an adversarial discriminator or ground cost.
For example, in the ground cost adversarial optimal transport formulation of GAN \citep{salimans2018improving, genevay2018learning}, we have 
$\GM(\xi; \alpha_\theta) = \SM_{c_\xi}(\alpha_\theta, \beta)$.
Here, with a slight abuse of notation, $\SM_{c_\xi}(\alpha_\theta, \beta)$ denotes the Sinkhorn divergence between the parameterized measure $\alpha_\theta$ and a given target measure $\beta$.
Notice that the symmetric ground cost $c_\xi$ in $\SM_{c_\xi}$ is no longer fixed to any pre-specified distance like $\ell_1$ or $\ell_2$ norm.
Instead, $c_\xi$ is encoded by a parameter $\xi$ so that $\SM_{c_\xi}$ can distinguish $\alpha_\theta$ and $\beta$ in an adaptive and adversarial manner.
By plugging the above $\FM(\alpha_\theta)$ to \eqref{eqn_main}, we recover the generative adversarial model proposed in \citep{genevay2018learning}:
\begin{equation} \label{eqn_experiment_gan}
\min_{\theta\in\Theta} \max_{\xi\in\Xi} \SM_{c_\xi}(\alpha_\theta, \beta).
\end{equation}
\section{Methodology} \label{section_methodology}
In this section, we derive the Sinkhorn Natural Gradient (SiNG) algorithm as a steepest descent method in the probability space endowed with the Sinkhorn divergence metric.
Specifically,  SiNG updates the parameter $\theta^t$ by 
\begin{equation}
	\theta^{t+1} := \theta^t + \eta\cdot \dB^t
\end{equation}
where $\eta > 0$ is the step size and the update direction $\dB^t$ is obtained by solving the following problem. 
Recall the objective $F$ in \eqref{eqn_main} and the Sinkhorn divergence $\SM$ in \eqref{s-o}.
Let $\dB^t = \lim_{\epsilon\rightarrow0}\frac{\Delta\theta^t_\epsilon}{\sqrt{\epsilon}}$, where 
\begin{equation} \label{eqn_natural_gradient_subproblem}
	\begin{aligned}
		\Delta\theta^t_\epsilon \defi \argmin_{\Delta\theta\in\RBB^d}  F(\theta^t+\Delta\theta) 
		\quad\mathrm{s.t.} \quad \|\Delta\theta\|\leq \epsilon^{c_1}, \SM(\alpha_{\theta^t + \Delta\theta}, \alpha_{\theta^t})\leq \epsilon + \epsilon^{c_2}.
	\end{aligned}
\end{equation}
Here the exponent $c_1$ and $c_2$ can be arbitrary real satisfying $1 < c_2 < 1.5$, $c_1<0.5$ and $3c_1 - 1\geq c_2$.
Proposition \ref{proposition_update_direction} depicts a simple expression of $\dB^t$.
Before proceeding to derive this expression, we note that $\Delta\theta = 0$ globally minimizes the non-negative function $\SM(\alpha_{\theta^t + \Delta\theta}, \alpha_{\theta^t})$, which leads to the following first and second order optimality criteria:
\begin{equation} \label{eqn_optimality_of_Sinkhorn_divergence}
	\nabla_\theta \SM(\alpha_{\theta}, \alpha_{\theta^t})_{\vert {\theta} = {\theta^t}} = 0 \quad \mathrm{and}\quad\HB(\theta^t)\defi\nabla^2_\theta \SM(\alpha_{\theta}, \alpha_{\theta^t})_{\vert {\theta} = {\theta^t}} \succcurlyeq 0.
\end{equation}
This property is critical in deriving the explicit formula of the Sinkhorn natural gradient.
From now on, the term $\HB(\theta^t)$, which is a key component of SiNG, will be referred to as the \emph{Sinkhorn information matrix (SIM)}.

\begin{proposition}\label{proposition_update_direction}
	Assume that the minimum eigenvalue of $\HB(\theta^t)$ is strictly positive (but can be arbitrary small) and that $\nabla^2_\theta F(\theta)$ and $\HB(\theta)$ are continuous w.r.t. $\theta$.
	The SiNG direction has the following  explicit expression
	\begin{equation} \label{eqn_update_direction}
		\dB^t = -\frac{\sqrt{2}}{\sqrt{\langle\HB(\theta^t)^{-1}\nabla_\theta F(\theta^t), \nabla_\theta F(\theta^t)\rangle}}\cdot\HB(\theta^t)^{-1}\nabla_\theta F(\theta^t).
	\end{equation}
\end{proposition}
Interestingly, the SiNG direction does not involve the Hessian of $F$. This is due to a Lagrangian-based argument that we sketch here.
Note that the continuous assumptions on $\nabla^2_\theta F(\theta)$ and $\HB(\theta)$ enable us to approximate the objective and the constraint in \eqref{eqn_natural_gradient_subproblem} via the second-order Taylor expansion.
\begin{proof}[Proof sketch for Proposition~\ref{proposition_update_direction}]
	The second-order Taylor expansion of the Lagrangian of \eqref{eqn_natural_gradient_subproblem} is
	\begin{equation}
	\bar G(\Delta\theta) = F(\theta^t) + \langle \nabla_\theta F(\theta^t), \Delta\theta\rangle + \frac{1}{2}\langle \nabla^2_\theta F(\theta^t) \Delta\theta, \Delta\theta\rangle + \frac{\lambda}{2} \langle\HB(\theta^t)\Delta\theta, \Delta\theta\rangle -  \lambda\epsilon - \lambda \epsilon^{c_2},
	\end{equation}
	where $\lambda\geq 0$ is the dual variable.
	Since the minimum eigenvalue of $\HB(\theta^t)$ is strictly positive, for a sufficiently small $\epsilon$, by taking $\lambda = \OM(\frac{1}{\sqrt{\epsilon}})$, we have that $\HB(\theta^t) + \frac{1}{\lambda} \nabla^2_\theta F(\theta^t)$ is also positive definite.
	In such case, a direct computation reveals that $\bar G$ is minimized at
	\begin{equation}
		\widebar{\Delta \theta^*} = -\frac{1}{\lambda}\left(\HB(\theta^t) + \frac{1}{\lambda} \nabla^2_\theta F(\theta^t)\right)^{-1}\nabla_\theta F(\theta^t).
	\end{equation}
	Consequently, the term involving $\nabla^2_\theta F(\theta^t)$ vanishes when $\epsilon$ approaches zero and we obtain the result.
	
	The above argument is made precise in Appendix \ref{appendix_proof_of_proposition_update_direction}.
\end{proof}
\begin{remark}
	Note that our derivation also applies to the Fisher-Rao natural gradient or the Wasserstein natural gradient: If we replace the Sinkhorn divergence by the KL divergence (or the Wasserstein distance), the update direction $\dB^t \simeq \left[\HB(\theta^t)\right]^{-1}\nabla_\theta F(\theta^t)$ still holds, where $\HB(\theta^t)$ is the Hessian matrix of the KL divergence (or the Wasserstein distance). This observation works for a general functional as a local metric \cite{thomas2016energetic} as well. 
\end{remark}
The following proposition states that SiNG is invariant to reparameterization in its continuous time limit ($\eta\rightarrow 0$).
The proof is stated in Appendix \ref{section_proof_proposition_invariance}.
\begin{proposition} \label{proposition_invariance}
	Let $\Phi$ be an invertible and smoothly differentiable function and denote  a re-parameterization $\phi = \Phi(\theta)$.
	Define $\tilde{\HB}(\bar \phi) \defi \nabla^2_\phi \SM(\alpha_{\Phi^{-1}(\phi)}, \alpha_{\Phi^{-1}(\bar \phi)})_{\vert \phi = \bar \phi}$ and $\tilde F(\bar \phi) \defi F(\Phi^{-1}(\bar \phi))$. 
	Use $\dot \theta$ and $\dot \phi$ to denote the time derivative of $\theta$ and $\phi$ respectively.
	Consider SiNG in its continuous-time limit under these two parameterizations:
	\begin{equation}
		\dot \theta_s = - \HB(\theta_s)^{-1}\nabla F(\theta_s) \quad \mathrm{and} \quad \dot \phi_s = - \tilde \HB(\phi_s)^{-1}\nabla \tilde F(\phi_s) \quad \mathrm{with} \quad \phi_0 = \Phi(\theta_0).
	\end{equation}
	Then $\theta_s$ and $\phi_s$ are related by the equation $\phi_s = \Phi(\theta_s)$ at all time $s\geq 0$.
\end{proposition}
The SiNG direction is a ``curved" negative gradient of the loss function $F(\theta)$ and the ``curvature" is exactly given by the Sinkhorn Information Matrix (SIM), i.e. the Hessian $\HB(\theta^t) = \nabla^2_\theta \SM(\alpha_{\theta}, \alpha_{\theta^t})_{\vert {\theta} = {\theta^t}}$ of the Sinkhorn divergence. An important question is whether SIM is computationally tractable.
In the next section, we derive its explicit expression and describe how it can be efficiently computed.
This is in sharp contrast to the Wasserstein information matrix (WIM) as in the WNG method proposed in \cite{li2018natural}, which  does not attain an explicit form for $d>1$ ($d$ is the parameter dimension).

While computing  the update direction $\mathbf{d}_t$ involves the inversion of $\HB(\theta^t)$, it can be computed using the classical conjugate gradient algorithm, requiring only  a matrix-vector product.
Consequently,  our Sinkhorn Natural Gradient (SiNG) admits a simple and elegant implementation based on modern auto-differential mechanisms such as PyTorch.
{We will elaborate this point in Appendix \ref{appendix_pytorch}.}

\section{Sinkhorn Information Matrix}
In this section, we describe the explicit expression of the \emph{Sinkhorn information matrix (SIM)} and show that it can be computed very efficiently using simple function operations (e.g. $\log$ and $\exp$) and integrals with respect to $\alpha_\theta$ (with complexity logarithmic in terms of the reciprocal of the target accuracy).
The computability of SIM and hence SiNG is the key contribution of our paper.
In the case when we can only compute the integration with respect to $\alpha_\theta$ in a Monte Carlo manner, 
an empirical estimator of SIM (eSIM) is proposed in the next section with a delicate stability analysis.\\
Since  $\SM(\cdot, \cdot)$ is a linear combination  of terms like $\OTgamma(\cdot, \cdot)$--see \eqref{s-o}, we can focus on the term $\nabla^2_\theta \OTgamma(\alpha_{\theta}, \alpha_{\theta^t})_{\vert {\theta} = {\theta^t}}$ in $\HB(\theta^t)$ and 
the other term $\nabla^2_\theta \OTgamma(\alpha_{\theta}, \alpha_{\theta})_{\vert {\theta} = {\theta^t}}$ can be handled similarly.
Having these two terms, SIM is computed as $\HB(\theta^t) = [\nabla_\theta^2 \mathrm{OT}_\gamma(\alpha_\theta, \alpha_{\theta^t}) + \nabla_\theta^2 \mathrm{OT}_\gamma(\alpha_\theta, \alpha_{\theta})]_{\vert {\theta} = {\theta^t}}$.

Recall that the entropy regularized optimal transport distance $\OTgamma$ admits an equivalent dual concave-maximization form \eqref{eqn_OTepsilon_dual}.
Due to the concavity of $\HM_2$ w.r.t. $g$ in \eqref{eqn_OTepsilon_dual_two_variable}, the corresponding optimal $g_f = \argmax_{g\in\CM(\XM)} \HM_2(f, g; \alpha, \beta)$ can be explicitly computed for any fixed $f\in\CM(\XM)$:
Given a function $\bar f\in\CM(\XM)$ and a measure $\alpha\in\MM_1^+(\XM)$, define the Sinkhorn mapping as
\begin{equation} \label{eqn_sinkhorn_mapping}
	\AM\big(\bar f, \alpha\big)(y) \defi -\gamma\log\int_\XM\exp\left(-\frac{1}{\gamma}c(x, y) + \frac{1}{\gamma}\bar f(x)\right)\dB\alpha(x).
\end{equation}
The first-order optimality of $g_f$ writes $g_f = \AM(f, \alpha)$. Then, \eqref{eqn_OTepsilon_dual} can be simplified to the following problem with a single potential variable: 
\begin{equation} \label{eqn_OTgamma_dual_single_variable}
	\OTgamma(\alpha_\theta, \beta) = \max_{f\in\CM(\XM)} \left\{\HM_1(f, \theta) \defi \langle f, \alpha_\theta\rangle + \langle \AM \big(f, \alpha_\theta\big), \beta\rangle  \right\},
\end{equation}
where we emphasize the impact of $\theta$ to $\HM_1$ by writing it explicitly as a variable for $\HM_1$.
Moreover, in $\HM_1$ the dependence on $\beta$ is dropped as $\beta$ is fixed.
We also denote the optimal solution to the R.H.S. of \eqref{eqn_OTgamma_dual_single_variable} by $f_\theta$ which is one of the Sinkhorn potentials for $\OTgamma(\alpha_\theta, \beta)$.

The following proposition describes the explicit expression of $\nabla^2_\theta \OTgamma(\alpha_{\theta}, \alpha_{\theta^t})_{\vert {\theta} = {\theta^t}}$ based on the above dual representation. The proof is provided in Appendix \ref{appendix_proof_of_proposition_SIM_expression}.
\begin{proposition} \label{proposition_SIM_expression}
	Recall the definition of the dual-variable function $\HM_1:\CM(\XM)\times \Theta\rightarrow \RBB$ in \eqref{eqn_OTgamma_dual_single_variable} and the definition of the second-order Fr\'echet derivative at the end of Section \ref{section_introuction}.
	For a parameterized push-forward measure $\alpha_\theta = {T_\theta}_\sharp \mu$ and a fixed measure $\beta\in\MM_1^+(\XM)$, we have
	\begin{equation}
		\nabla^2_\theta \OTgamma(\alpha_{\theta}, \beta) = - D^2_{11} \HM_1(f_\theta, \theta)\circ (Df_\theta, Df_\theta) + D^2_{22}\HM_1(f_\theta, \theta),
	\end{equation}
	where $Df_\theta$ denotes the Fr\'echet derivative of the Sinkhorn potential $f_\theta$ w.r.t. the parameter $\theta$.
\end{proposition}
\begin{remark}[SIM for $1d$-Gaussian]
	It is in general difficult to give closed form expression of the SIM.
	However, in the simplest case when $\alpha_{\theta}$ is a one-dimensional Gaussian distribution with a parameterized mean, i.e. $\alpha_{\theta} = \NM(\mu(\theta), \sigma^2)$, SIM can be explicitly computed as $\nabla_{\theta}^2 \SM(\alpha_{\theta}, \beta) = 2\nabla^2_\theta \mu(\theta)$ due to the closed form expression of the entropy regularized optimal transport between Gaussian measures \citep{janati2020entropic}.
\end{remark}
Suppose that we have the Sinkhorn potential $f_\theta$ and its the Fr\'echet derivative $Df_\theta$.
Then the terms $D^2_{ij}\HM_1(f, \theta), i,j = 1,2$ can all be evaluated using a constant amount of simple function operations, e.g. $\log$ and $\exp$, since we know the explicit expression of $\HM_1$.
Consequently, it is sufficient to have estimators $f_\theta^\epsilon$ and $g_\theta^\epsilon$ of $f_\theta$ and $Df_\theta$ respectively, such that $\|f_\theta^\epsilon - f_\theta\|_\infty\leq \epsilon$ and $\|g_\theta^\epsilon - Df_\theta\|_{op}\leq \epsilon$ for an arbitrary target accuracy $\epsilon$.
This is because the high accuracy approximation of $f_\theta$ and $Df_\theta$ imply the high accuracy approximation of $\nabla^2_\theta \OTgamma(\alpha_{\theta}, \beta)$ due to the Lipschitz continuity of the terms $D^2_{ij}\HM_1(f, \theta), i,j = 1,2$. We derive these expressions and their Lipschitz continuity in Appendix \ref{appendix_SIM}.

For the Sinkhorn Potential $f_\theta$, its estimator $f_\theta^\epsilon$ can be efficiently computed using the Sinkhorn-Knopp algorithm \cite{sinkhorn1967concerning}.
We provide more details on this in Appendix \ref{appendix_proposition_sinkhorn_potential}.
\begin{proposition}[Computation of the Sinkhorn Potential $f_\theta$ -- (Theorem 7.1.4 in \citep{lemmens2012nonlinear} and Theorem B.10 in \citep{NIPS2019_9130}] \label{proposition_sinkhorn_potential}
	Assume that the ground cost function $c$ is bounded, i.e. $ 0 \leq c(x, y)\leq M_c, \forall x, y\in\XM$.
	Denote $\lambda\defi\frac{\exp(M_c/\gamma) -1 }{\exp(M_c/\gamma) + 1 }<1$ and define
	\begin{equation} \label{eqn_main_fix_point}
	\BM\big(f, \theta\big) \defi \AM\big(\AM\big(f, \alpha_\theta\big), \beta\big).
	\end{equation}
Then the fixed point iteration $f^{t+1} = \BM\big(f^t, \theta\big)$	converges linearly:
	$\|f^{t+1} - f_\theta\|_\infty=\OM(\lambda^t)$.
\end{proposition}
For the Fr\'echet derivative $Df_\theta$, we construct its estimator in the following proposition.
\begin{proposition}[Computation of the Fr\'echet derivative $Df_\theta$] \label{proposition_frechet_derivative}
	Let $f_\theta^\epsilon$ be an approximation of $f_\theta$ such that $\|f_\theta^\epsilon-f_\theta\|_\infty\leq  \epsilon$. Choose a large enough $l$, for instance $l = \lceil\log_{\lambda}\frac{1}{3}\rceil/2$.
	Define $\EM\big(f, \theta\big) = \BM\big(\cdots\BM\big(f, \theta\big)\cdots,\theta\big)$, the $l$ times composition of $\BM$ in its first variable. 
	Then the sequence 
	\begin{equation}
		g_{\theta}^{t+1} = D_1\EM\big(f_\theta^\epsilon, \theta\big)\circ g_{\theta}^{t} + D_2\EM\big(f_\theta^\epsilon, \theta\big)
	\end{equation}
	converges linearly to a $\epsilon$-neighborhood of $Df_\theta$, i.e.  $\|g_{\theta}^{t+1} - Df_\theta\|_{op} = \OM(\epsilon + (\frac{2}{3})^t\|g_{\theta}^{0} - Df_\theta\|_{op})$.
\end{proposition}
We deferred the proof to the above proposition to Appendix \ref{appendix_proof_of_proposition_frechet_derivative}.
The high-accuracy estimators $f_\theta^\epsilon$ and $g_\theta^\epsilon$ derived in the above propositions can both be obtained using $\OM(\log\frac{1}{\epsilon})$ function operations and integrals.
With the expression of SIM and the two propositions discussing the efficient computation of $f_\theta$ and $Df_\theta$, we obtain the following theorem.
\begin{theorem}[Computability of SIM] \label{thm_computability_of_SIM}
	For any given target accuracy $\epsilon>0$, there exists an estimator $\HB_\epsilon(\theta)$, such that $\|\HB_\epsilon(\theta) - \HB(\theta)\|_{op}\leq \epsilon$, and the estimator can be computed using $\OM(\log\frac{1}{\epsilon})$ simple function operations and integrations with respect to $\alpha_\theta$.
\end{theorem}
This result shows a significantly broader applicability of SiNG than WNG, as the latter  can only be used in limited situations due to the intractability of computing WIM. 

\section{Empirical Estimator of SIM}
In the previous section, we derived an explicit expression for the Sinkhorn information matrix (SIM) and described how it can be computed efficiently.
In this section, we provide an empirical estimator for SIM (eSIM) in the case where the integration w.r.t. $\alpha_\theta$ can only be computed in a Monte-Carlo manner. 
Moreover, we prove the stability of eSIM by showing that the Fr\'echet derivative of the Sinkhorn potential with respect to the parameter $\theta$ is continuous with respect to the underlying measure $\mu$, which is interesting on its own.

Recall that the parameterized measure has the structure $\alpha_\theta = {T_{\theta}}_{\sharp}\mu$, where $\mu\in\MM_1^+(\ZM)$ is some probability measure on the latent space $\ZM\subseteq\RBB^{\bar q}$ and $T_\theta:\ZM\rightarrow\XM$ is some push-forward mapping parameterized by $\theta\in\Theta$.
We use $\bar{\mu}$ to denote an empirical measure of $\mu$ with $n$ Dirac measures: $\bar{\mu} = \frac{1}{n}\sum_{i=1}^{n}\delta_{z_i}$ with $z_i\stackrel{\mathrm{iid}}{\sim}\mu$ and we use $\bar{\alpha}_\theta$ to denote the corresponding empirical measure of $\alpha_\theta$: $\bar{\alpha}_\theta = {T_{\theta}}_\sharp\bar\mu = \frac{1}{n}\sum_{i=1}^{n}\delta_{T_{\theta}(z_i)}$.
Based on the above definition, we propose the following empirical estimator for the Sinkhorn information matrix (eSIM)  
\begin{equation}
	\bar \HB(\theta^t) = \nabla^2_\theta \SM(\bar \alpha_{\theta}, \bar \alpha_{\theta^t})_{\vert {\theta} = {\theta^t}}.
\end{equation}
The following theorem shows stability of eSIM. The proof is provided in Appendix \ref{appendix_eSIM}. 
\begin{theorem} \label{theorem_consistency}
	Define the bounded Lipschitz metric of measures $d_{bl}:\MM_1^+(\XM)\times \MM_1^+(\XM)\rightarrow\RBB_+$ by
	\begin{equation}
	d_{bl}(\alpha, \beta) \defi \sup_{\|\xi\|_{bl}\leq 1} |\langle \xi, \alpha\rangle - \langle \xi, \beta \rangle|,
	\end{equation}
	where we denote $\|\xi\|_{bl} \defi \max\{\|\xi\|_\infty, \|\xi\|_{Lip} \}$ with $\|\xi\|_{Lip}\defi \max_{x, y\in\XM}\frac{|\xi(x)-\xi(y)|}{\|x-y\|}$.
	Assume that the ground cost function is bounded and Lipschitz continuous. Then
	\begin{equation}
		\|\bar \HB(\theta^t) - \HB(\theta^t)\|_{op} = {\OM(d_{bl}(\mu, \bar\mu))}.
	\end{equation}
\end{theorem}
In the rest of this subsection, we analyze the structure of $\bar \HB(\theta^t)$ and describe how it can be efficiently computed.
Similar to the previous section, we focus  on the term $\nabla^2_\theta \OTgamma(\bar \alpha_{\theta}, \beta)$ with $\bar \alpha_\theta = \frac{1}{n}\sum_{i=1}^{n}\delta_{T_{\theta}(z_i)}$ and $ \beta = \frac{1}{n}\sum_{i=1}^{n} \delta_{y_i}$ for arbitrary $y_i\in\XM$.

First, notice that the output of the Sinkhorn mapping \eqref{eqn_sinkhorn_mapping} is determined solely by the function values of the input $\bar{f}$ at the support of $\alpha$.
Using $\fB = [\fB_1, \ldots, \fB_n] \in \RBB^n$ with $\fB_i = \bar f(x_i)$ to denote the value extracted from $\bar f$ on $\supp(\bar \alpha)$,
we define for a discrete probability measures $\bar \alpha = \frac{1}{n}\sum_{i=1}^{n}\delta_{x_i}$ the discrete Sinkhorn mapping $\bar\AM\big(\fB, \bar \alpha\big):\RBB^n\times\MM_1^+(\XM)\rightarrow \CM(\XM)$ as
\begin{equation}
	\bar\AM\big(\fB, \bar \alpha\big) (y) \defi -\gamma\log \Big(\frac{1}{n}\sum_{i=1}^{n}\exp\Big(-\frac{1}{\gamma}c(x_i, y) + \frac{1}{\gamma} \fB_i\Big) \Big) = \AM\big(\bar f, \bar\alpha\big) (y),
\end{equation}
where the last equality should be understood as two functions being identical.
Since both $\bar \alpha_\theta$ and $ \beta$ in $\OTgamma(\bar \alpha_{\theta},  \beta)$ are discrete, \eqref{eqn_OTgamma_dual_single_variable} can be reduced to 
\begin{equation}\label{eqn_OTepsilon_dual_discrete_single}
	\OTgamma(\bar \alpha_\theta, \beta) = \max_{\fB\in \RBB^n} \left\{\bar \HM_1(\fB, \theta) = \frac{1}{n} \fB^\top \oneB_n + \frac{1}{n} \sum_{i=1}^{n} \bar \AM \big(\fB, \bar \alpha_\theta\big)(y_i) \right\}.
\end{equation}

Now, let $\fB_\theta$ be the solution to the above problem.
We can compute the first order gradient of $\OTgamma(\bar \alpha_{\theta}, \beta)$ with respect to $\theta$ by
\begin{equation}
	\nabla_\theta \OTgamma(\bar \alpha_{\theta}, \beta) = J_{\fB_\theta}^\top\cdot\nabla_1 \bar\HM_1(\fB_\theta, \theta)  + \nabla_2 \bar\HM_1(\fB_\theta, \theta) .
\end{equation}
Here $J_{\fB_\theta} = \frac{\partial \fB_\theta}{\partial \theta}\in\RBB^{n\times d}$ denotes the Jacobian matrix of $\fB_\theta$ with respect to $\theta$ and $\nabla_i \bar\HM_1$ denotes the gradient of $\bar\HM_1$ with respect to its $i^{th}$ variable for $i=1,2$.
Importantly, the optimality condition of $\fB_\theta$ implies $\nabla_1 \bar\HM_1(\fB_\theta, \theta) = \zeroB_n$.
Further, we compute the second order gradient of $\OTgamma(\bar \alpha_{\theta}, \beta)$ with respect to $\theta$ by (we omit the parameter $(\fB_\theta, \theta)$ of $\bar \HM_1$)
\begin{equation}
\nabla^2_\theta \OTgamma(\bar \alpha_{\theta}, \beta) = 
T_{\fB_\theta}\times_1 \nabla_{1} \bar\HM_1 + J_{\fB_\theta}^\top\cdot\nabla_{11} \bar\HM_1\cdot J_{\fB_\theta} + J_{\fB_\theta}^\top\cdot\nabla_{12} \bar\HM_1
+ \nabla_{21} \bar\HM_1^\top\cdot J_{\fB_\theta} + \nabla_{22} \bar\HM_1,
\end{equation}
where $T_{\fB_\theta}=\frac{\partial^2 \fB_\theta}{\partial \theta^2}\in\RBB^{n\times d\times d}$ is a tensor denoting the second-order Jacobian matrix of $\fB_\theta$ with respect to $\theta$ and  $\times_1$ denotes the tensor product along its first dimension.
Using the fact that $\nabla_1 \bar\HM_1(\fB_\theta, \theta) = \zeroB_n$, we drop the first term and simplify $\nabla^2_\theta \OTgamma(\bar \alpha_{\theta}, \beta)$ to (again we omit the parameter $(\fB_\theta, \theta)$ of $\bar \HM_1$)
\begin{equation}
\nabla^2_\theta \OTgamma(\bar \alpha_{\theta}, \beta) = 
J_{\fB_\theta}^\top\cdot\nabla_{11} \bar\HM_1\cdot J_{\fB_\theta} + J_{\fB_\theta}^\top\cdot\nabla_{12} \bar\HM_1 
+ \nabla_{21} \bar\HM_1^\top\cdot J_{\fB_\theta} + \nabla_{22} \bar\HM_1.
\end{equation}
As we have the explicit expression of $\bar\HM_1$, we can explicitly compute $\nabla_{ij} \bar\HM_1$ given that we have the Sinkhorn potential $\fB_\theta$.
Further, if we can  compute $J_{\fB_\theta}$, we are then able to compute $\nabla^2_\theta \OTgamma(\bar \alpha_{\theta}, \beta)$.
The following propositions can be viewed as discrete counterparts of Proposition \ref{proposition_sinkhorn_potential} and Proposition \ref{proposition_frechet_derivative} respectively.
Both $\fB_\theta$ and $J_{\fB_\theta}$ can be well-approximated using a number of finite dimensional vector/matrix operations which is logarithmic in the desired accuracy.
Besides, given these two quantities, one can easily check that $\nabla_{ij} \bar\HM_1$ can be evaluated within $\OM((n+d)^2)$ arithmetic operations.
Consequently, we can compute an $\epsilon$-accurate approximation of eSIM in time $\OM((n+d)^2\log\frac{1}{\epsilon})$.
\begin{proposition}[Computation of the Sinkhorn Potential $\fB_\theta$]
	Assume that the ground cost function $c$ is bounded, i.e. $0 \leq c(x, y)\leq M_c, \forall x, y\in\XM$.
	Denote $\lambda\defi\frac{\exp(M_c/\gamma) -1}{ \exp(M_c/\gamma) + 1 }<1$ and define
	\begin{equation} 
		\bar\BM \big(\fB, \theta\big) \defi \bar\AM\big(\gB, \beta\big) \ \mathrm{with}\ \gB = [\bar\AM \big(\fB, \bar \alpha_{\theta}\big)(y_1), \ldots, \bar\AM\big(\fB, \bar \alpha_{\theta}\big)(y_n)]\in\RBB^n.
	\end{equation}
	Then the fixed point iteration $\fB^{t+1} = \bar \BM\big(\fB^t, \theta\big)$ converges linearly:
	$\|\fB^{t+1} - \fB_\theta\|_\infty=\OM(\lambda^t)$
\end{proposition}
\begin{proposition}[Computation of the Jacobian $J_{\fB_\theta}$]
	Let $\fB_\epsilon$ be an approximation of $\fB_\theta$ such that $\|\fB_\epsilon-\fB_\theta\|_\infty\leq  \epsilon$. Pick $l = \lceil\log_{\lambda}\frac{1}{3}\rceil/2$. Define $\bar \EM\big(\fB, \theta\big) = \bar\BM\big(\cdots\bar\BM\big(\fB, \theta\big)\cdots, \theta\big)$, the $l$ times composition of $\bar\BM$ in its first variable. 
	Then the sequence of matrices
	\begin{equation}
		\JB^{t+1} = J_{1}{\bar\EM\big(\fB_\epsilon, \theta\big)} \cdot\JB^{t} + J_2\bar\EM\big(\fB_\epsilon, \theta\big),
	\end{equation}
	converges linearly to an $\epsilon$ neighbor of $J_{\fB_\theta}$: $\|\JB^{t+1} - J_{\fB_\theta}\|_{op}= \OM(\epsilon + (\frac{2}{3})^t\|\JB^0 - J_{\fB_\theta}\|_{op})$.
	Here $J_{i}\bar\EM$ denotes the Jacobian matrix of $\bar\EM$ with respect to its $i^{th}$ variable.
\end{proposition}
The SiNG direction $\mathbf{d}_t$ involves the inversion of $\bar\HB(\theta^t)$.
This can be (approximately) computed using the classical conjugate gradient (CG) algorithm, using only matrix-vector products.
Combining eSIM and CG, we describe a simple and elegant PyTorch-based implementation for SiNG in Appendix \ref{appendix_pytorch}, 

%

\section{Experiment}
In this section, we compare SiNG with other SGD-type solvers by training generative models.
We did not compare with WNG \cite{li2018natural} since WNG can only be implemented for the case where the parameter dimension $d$ is $1$.
We also tried to implement KWNG \cite{arbel2019kernelized}, which however diverges in our setting. In particular, we encounter the case when the KWNG direction has negative inner product with the euclidean gradient direction, leading to its divergence. As we discussed in the related work, the gap between KWNG and WNG cannot be quantified with reasonable assumptions, which explains our observation.
In all the following experiments, we pick the push-forward map $T_\theta$ to be the generator network in DC-GAN \citep{radford2015unsupervised}.
For more detailed experiment settings, please see Appendix \ref{appendix_experiment}.

\subsection{Squared-$\ell_2$-norm as Ground Metric}
We first consider the distribution matching problem, 
where our goal is to minimize the Sinkhorn divergence between the parameterized generative model $\alpha_{\theta}  = {T_{\theta}}_\sharp \mu$ and a given target distribution $\beta$,
\begin{equation} \label{eqn_squared_l2_ground_cost}
	\min_{\theta \in \Theta }  F(\theta) = \SM(\alpha_\theta, \beta).
\end{equation}
\begin{wrapfigure}{r}{.3\textwidth}
	\centering
	\raisebox{0pt}[\dimexpr\height-0.2\baselineskip\relax]{\includegraphics[width=0.3\textwidth]{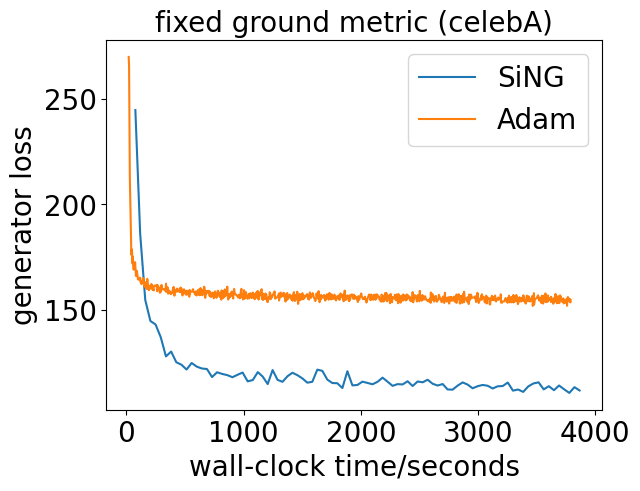}}
\end{wrapfigure}
Here, $T_\theta$ is a neural network describing the push-forward map with its parameter summarized in $\theta$ and $\mu$ is a zero-mean isometric Gaussian distribution.
In particular, the metric on the ground set $\XM$ is set to the vanilla squared-$\ell_2$ norm, i.e. $c(x, y) = \|x - y\|^2$ for $x, y\in\XM$.
Our experiment considers a specific instance of problem \eqref{eqn_squared_l2_ground_cost} where we take the measure $\beta$ to be the distribution of the images in the CelebA dataset.
We present the comparison of the generator loss (the objective value) vs time plot in right figure. The entropy regularization parameter $\gamma$ is set to $0.01$ for both the objective and the constraint.
We can see that SiNG is much more efficient at reducing the objective value than ADAM given the same amount of time.


\subsection{Squared-$\ell_2$-norm with an Additional Encoder as Ground Metric}
\begin{figure}[h]
	\centering
	\includegraphics[width=0.45\textwidth]{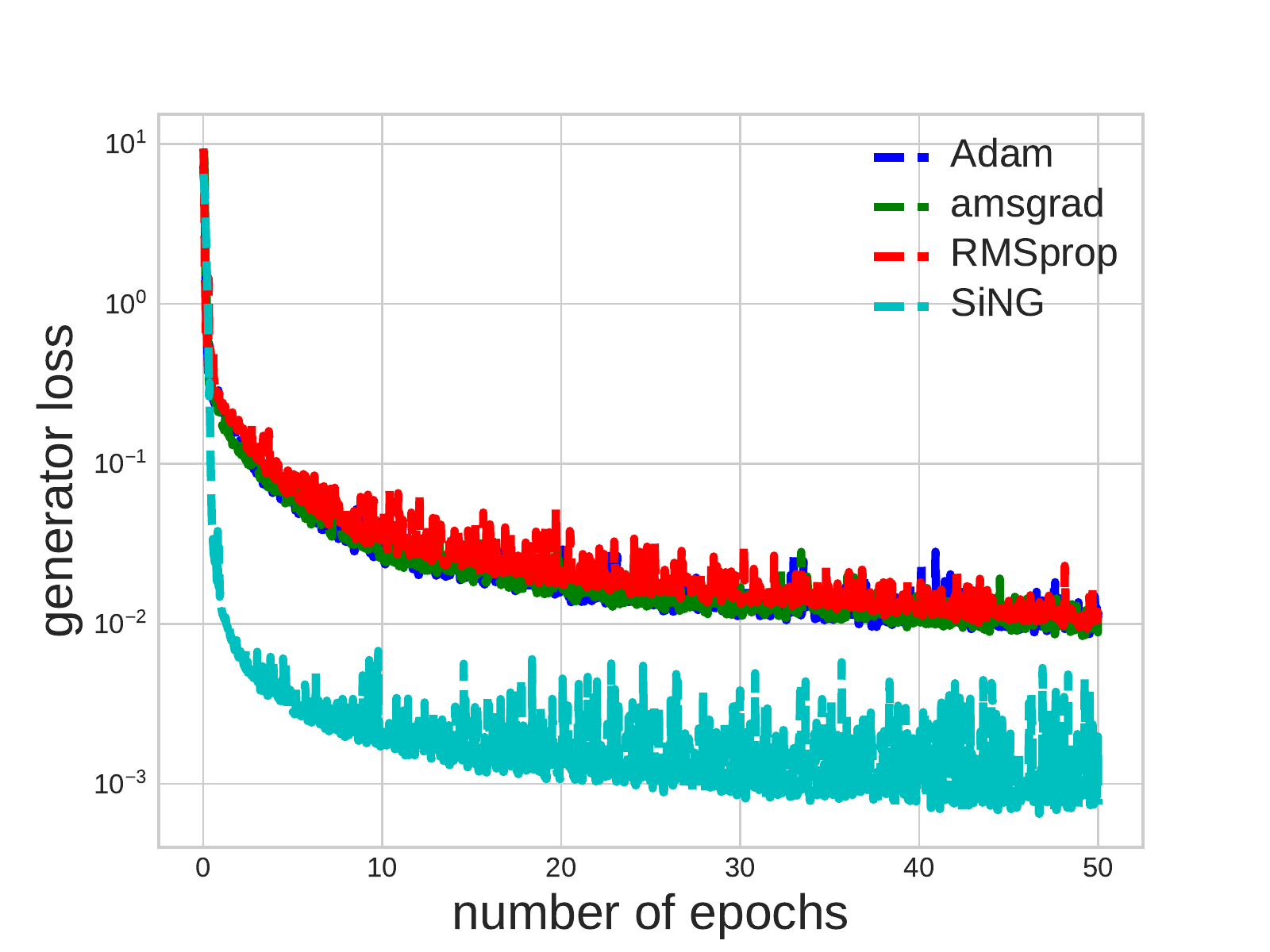}
	\includegraphics[width=0.45\textwidth]{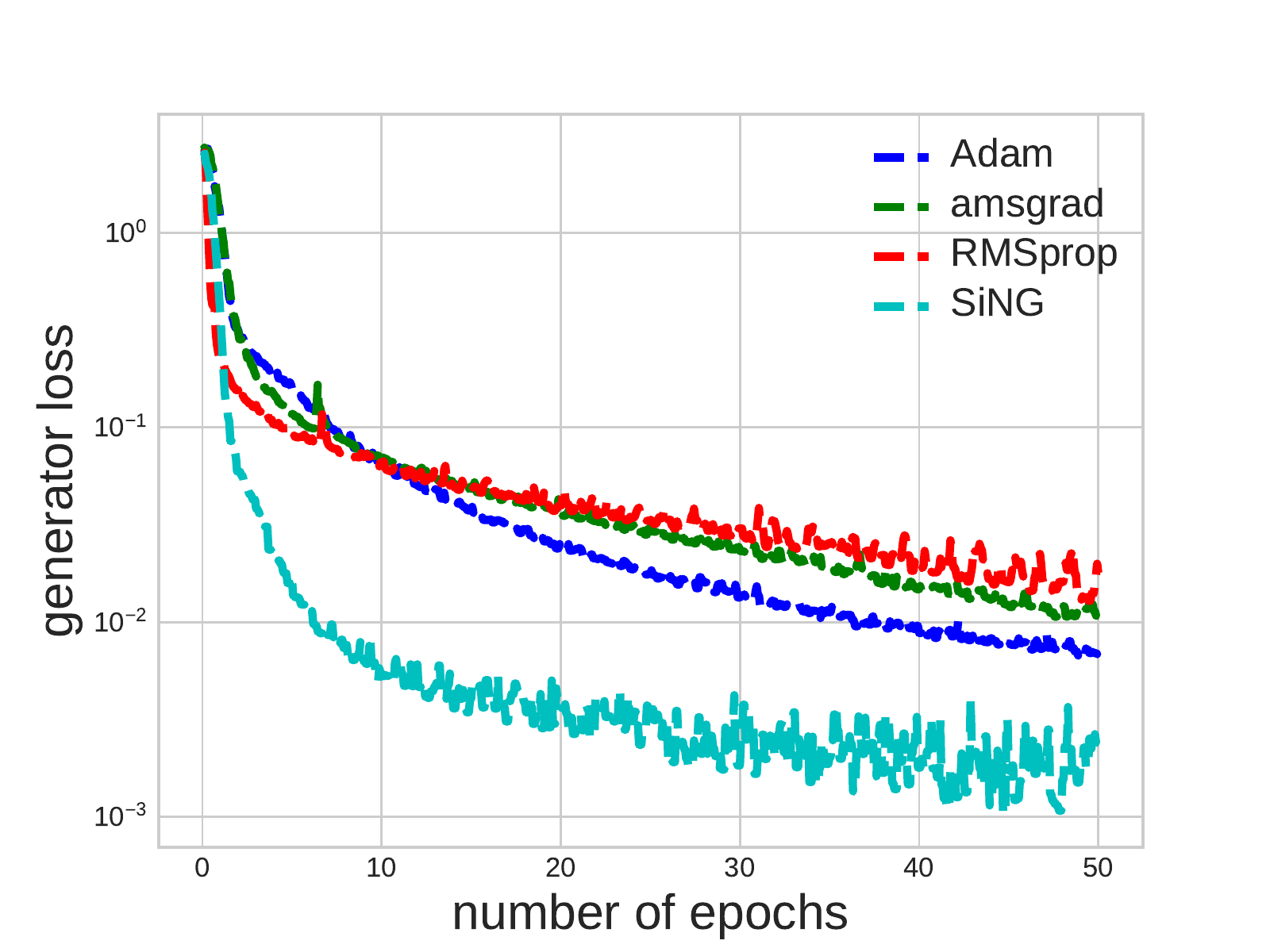}
	\caption{Generator losses on CelebA (left) and Cifar10 (right).}
	\label{fig_result}
\end{figure}
We then consider a special case of problem \eqref{eqn_experiment_gan}, where the metric on the ground set $\XM$ is set to squared-$\ell_2$-norm with a \emph{fixed} parameterized encoder (i.e. we fix the variable $\xi$ in the $\max$ part of \eqref{eqn_experiment_gan}): $c_\xi(x, y) = \|\phi_\xi(x) - \phi_\xi(y)\|^2$.
Here $\phi_\xi(\cdot):\XM\rightarrow \RBB^{\hat q}$ is a neural network encoder that outputs an embedding of the input in a high dimensional space ($\hat q > q$, where we recall $q$ is the dimension of the ground set $\XM$).
In particular, we set $\phi_\xi(\cdot)$ to be the discriminator network in DC-GAN without the last classification layer \citep{radford2015unsupervised}.
Two specific instances are considered: we take the measure $\beta$ to be the distribution of the images in either the CelebA or the Cifar10 dataset.
The parameter $\xi$ of the encoder $\phi$ is obtained in the following way: we first use SiNG to train a generative model by alternatively taking a SiNG step on $\theta$ and taking an SGD step on $\xi$.
After sufficiently many iterations (when the generated image looks real or specifically 50 epochs), we fix the encoder $\phi_{\xi}$.
We then set the objective functional \eqref{eqn_main} to be
	$\FM(\alpha_\theta) = \SM_{c_\xi}(\alpha_\theta, \beta)$ (see \eqref{eqn_experiment_gan}),
and compare SiNG and SGD-type algorithms in the minimization of $\FM$ under a consensus random initialization.
We report the comparison in Figure \ref{fig_result}, where we observe the significant improvement from SiNG in both accuracy and efficiency. 
Such phenomenon is due to the fact that SiNG is able to use geometry information by considering SIM while other method does not.
Moreover, the pretrained ground cost $c_\xi$ may capture some non-trivial metric structure of the images and consequently geometry-faithfully method like our SiNG can thus do better.

%
%

\subsection{Training GAN with SiNG}
\begin{figure}[h]
	\centering
	\begin{tabular}{c @{ } c}
		\includegraphics[width=0.45\textwidth]{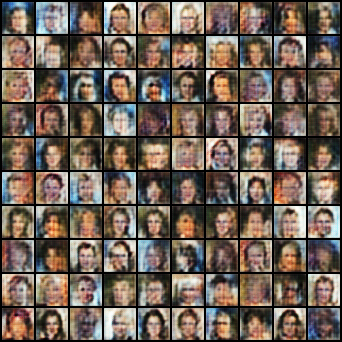}   &   
		\includegraphics[width=0.45\textwidth]{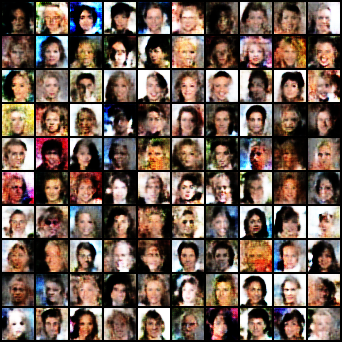}
	\end{tabular}
	\caption{Comparison of the visual quality of the images generated by Adam (left) and SiNG (right).}
	\label{fig_gan}
\end{figure}
Finally, we showcase the the advantage of training a GAN model using SiNG over SGD-based solvers.
Specifically, we consider the GAN model \eqref{eqn_experiment_gan}.
The entropy regularization of the Sinkhorn divergence objective is set to $\gamma=100$ as suggested in Table 2 of \citep{genevay2018learning}.
The regularization for the constraint is set to $\gamma=1$ in SiNG. We used {ADAM} as the optimizer for the discriminators (with step size $10^{-3}$ and batch size 4000).
The result is reported in Figure \ref{fig_gan}.
We can see that the images generated using SiNG are much more vivid than the ones obtained using SGD-based optimizers.
We remark that our main goal has been to showcase that SiNG is more efficient in reducing the objective value compared to SGD-based solvers, and hence, 
we have used a relatively simpler DC-GAN type generator and discriminator (details given in the supplementary materials). 
If more sophisticated ResNet type generators and discriminators are used, the image quality can be further improved. 
\clearpage
\section{Broader Impact}
We propose the Sinkhorn natural gradient (SiNG) algorithm for minimizing an objective functional over a parameterized family of generative-model type measures. While our results do not immediately lead to broader societal impacts (as they are mostly theoretical), they can lead to new potential positive impacts.
SiNG admits explicit update rule which can be efficiently carried out in an exact manner under both continuous and discrete settings.
Being able to exploit the geometric information provided in the Sinkhorn information matrix, we observe the remarkable advantage of SiNG over existing state-of-the-art SGD-type solvers.
Such algorithm is readily applicable to many types of existing generative adversarial models and possibly helps the development of the literature.
\section*{Acknowledgment}
This work is supported by NSF CPS-1837253.
\bibliographystyle{abbrvnat}
\bibliography{refs}
\clearpage
\appendix
\section{Appendix Section for Methodology} \label{appendix_methodology}
\subsection{Proof of Proposition \ref{proposition_update_direction}} \label{appendix_proof_of_proposition_update_direction}
Denote the Lagrangian function by
\begin{equation}
G_{\lambda}(\Delta\theta) = F(\theta^t+\Delta\theta) + \lambda \left(\SM(\alpha_{\theta^t + \Delta\theta}, \alpha_{\theta^t})- \epsilon -  \epsilon^{c_2}\right).
\end{equation}
We have the following inequality which characterize a lower bound of the solution to \eqref{eqn_natural_gradient_subproblem} (recall that $1 < c_2 < 1.5$, $c_1<0.5$ and $3c_1 - 1\geq c_2$) ,
\begin{equation}\label{eqn_minmax_maxmin}
\begin{aligned}
	\min_{\Delta\theta\in\RBB^d} &\ F(\theta^t+\Delta\theta)\\
	\mathrm{s.t.} &\ \|\Delta \theta\|\leq \epsilon^{c_1}\ \\
	&\ \SM(\alpha_{\theta^t + \Delta\theta}, \alpha_{\theta^t})\leq \epsilon + \epsilon^{c_2}
\end{aligned}
   = \min_{\|\Delta \theta\|\leq \epsilon^{c_1}}\max_{\lambda\geq 0} G_{\lambda}(\Delta\theta) \geq \max_{\lambda\geq 0} \min_{\|\Delta \theta\|\leq \epsilon^{c_1}} G_{\lambda}(\Delta\theta).
\end{equation}
We now focus on the R.H.S. of the above inequality.
Denote the second-order Taylor expansion of the Lagrangian $G_{\lambda}$ by $\bar{G}_{\lambda}$:
\begin{equation*}
\bar G_{\lambda}(\Delta\theta) = F(\theta^t) + \langle \nabla_\theta F(\theta^t), \Delta\theta\rangle + \frac{1}{2}\langle \nabla^2_\theta F(\theta^t) \Delta\theta, \Delta\theta\rangle + \frac{\lambda}{2} \langle\HB(\theta^t)\Delta\theta, \Delta\theta\rangle -  \lambda\epsilon - \lambda \epsilon^{c_2},
\end{equation*}
where we used the optimality condition \eqref{eqn_optimality_of_Sinkhorn_divergence} of $\SM(\alpha, \alpha^t)$ so that the first-order term of $\SM(\alpha, \alpha^t)$ vanishes. Besides, $\HB(\theta)$ is defined in \eqref{eqn_optimality_of_Sinkhorn_divergence}.
{The error of such approximation can be bounded as }
\begin{equation}
G_{\lambda}(\Delta\theta) - \bar G_{\lambda}(\Delta\theta) = \OM((\lambda+1) \|\Delta\theta\|^3).
\end{equation}
Further, for any fixed $\lambda$, denote $\Delta\theta_\lambda^* = \argmin_{\|\Delta \theta\|\leq \epsilon^{c_1}} G_{\lambda}(\Delta\theta)$.

We can then derive the following lower bound on the minimization subproblem of the R.H.S. of \eqref{eqn_minmax_maxmin}:
\begin{equation*}
\begin{aligned}
\max_{\lambda\geq 0}\min_{\|\Delta \theta\|\leq \epsilon^{c_1}} G_{\lambda}(\Delta\theta) =&\ \max_{\lambda\geq 0} \bar G_{\lambda}(\Delta\theta_{\lambda}^*) - \OM((\lambda+1) \|\Delta\theta_{\lambda}^*\|^3) \\
\geq&\ \max_{\lambda\geq 0} \bar G_{\lambda}(\Delta\theta_{\lambda}^*) - \OM((\lambda+1) \epsilon^{3c_1})  \\
\geq&\ \max_{\lambda\geq 0}\min_{\|\Delta \theta\|\leq \epsilon^{c_1}} \bar G_{\lambda}(\Delta\theta) - \OM((\lambda+1) \epsilon^{3c_1}),
\end{aligned}
\end{equation*}
Note that for sufficiently large $\lambda$, $\HB(\theta^t) + \frac{1}{\lambda} \nabla^2_\theta F(\theta^t) \succ 0 $ by recalling the positive definiteness of $\HB(\theta^t)$.
In this case, as a convex program, $\min_{\|\Delta \theta\|\leq \epsilon^{c_1}} \bar G_{\lambda}(\Delta\theta)$ admits the closed form solution:
Denote $\widebar{\Delta \theta_{\lambda}^*} = \argmin \bar G_{\lambda}(\Delta\theta)$. We have
\begin{equation}
\widebar{\Delta \theta_{\lambda}^*} = -\frac{1}{\lambda}\left(\HB(\theta^t) + \frac{1}{\lambda} \nabla^2_\theta F(\theta^t)\right)^{-1}\nabla_\theta F(\theta^t)\ \mathrm{and} \ \bar G(\widebar{\Delta \theta_{\lambda}^*}) = F(\theta^t) - \frac{\bar a}{2\lambda} - \lambda\epsilon - \lambda\epsilon^{c_2},
\end{equation}
where we denote $\bar a\defi \langle\left[\HB(\theta^t)  + \frac{1}{\lambda} \nabla^2_\theta F(\theta^t)\right]^{-1}\nabla_\theta F(\theta^t), \nabla_\theta F(\theta^t)\rangle > 0$.

For sufficiently small $\epsilon$, by taking $\lambda = \sqrt{\frac{a}{2\epsilon}}$ with $a \defi \langle\left[\HB(\theta^t) \right]^{-1}\nabla_\theta F(\theta^t), \nabla_\theta F(\theta^t)\rangle > 0$ (note that $\|\widebar{\Delta \theta_{\lambda}^*}\| = \OM(\sqrt{\epsilon}) < \epsilon^{c_1}$ and is hence feasible for $c_1<0.5$), the R.H.S. of \eqref{eqn_minmax_maxmin} has the following lower bound (recall that we have $3c_1-1\geq c_2$)
\begin{equation}
\max_{\lambda\geq 0} \min_{\|\Delta \theta\|\leq \epsilon^{c_1}} G_{\lambda}(\Delta\theta) \geq F(\theta^t) - (\frac{\bar{a}}{\sqrt{2a}}+\sqrt{\frac{a}{2}})\sqrt{\epsilon} - \OM(\epsilon^{c_2-0.5}).
\end{equation}

This result leads to the following lower bound on \eqref{eqn_natural_gradient_subproblem}:
\begin{equation}
\lim_{\epsilon\rightarrow 0}\frac{F(\theta^t+\Delta\theta^t_\epsilon) - F(\theta^t)}{\sqrt{\epsilon}} \geq -\sqrt{2\langle\left[\HB(\theta^t)\right]^{-1}\nabla_\theta F(\theta^t), \nabla_\theta F(\theta^t)\rangle},
\end{equation}
where $\Delta\theta^t_\epsilon$ is the solution to \eqref{eqn_natural_gradient_subproblem}.
Finally, observe that the equality is achieved by taking $\Delta\theta^t_\epsilon = -\frac{\sqrt{2\epsilon} \left(\HB(\theta^t) \right)^{-1}\nabla_\theta F(\theta^t)}{\sqrt{\langle\left[\HB(\theta^t)\right]^{-1}\nabla_\theta F(\theta^t), \nabla_\theta F(\theta^t)\rangle}}$:
\begin{equation}
\lim_{\epsilon\rightarrow 0}\frac{F(\theta^t+\Delta\theta^t_\epsilon) - F(\theta^t)}{\sqrt{\epsilon}} =  \lim_{\epsilon\rightarrow 0} \frac{1}{\sqrt{\epsilon}}\langle\nabla F(\theta^t), \Delta\theta^t_\epsilon	\rangle = -\sqrt{2\langle\left[\HB(\theta^t)\right]^{-1}\nabla_\theta F(\theta^t), \nabla_\theta F(\theta^t)\rangle},
\end{equation}
and $\Delta\theta^t_\epsilon$ is feasible for sufficiently small $\epsilon$ (note that we have $\frac{1}{2}\langle\HB(\theta^t)\Delta\theta^t_\epsilon, \Delta\theta^t_\epsilon\rangle = \epsilon$):
\begin{equation}
\SM(\alpha_{\theta^t + \Delta\theta^t_\epsilon}, \alpha_{\theta^t}) \leq \frac{1}{2}\langle\HB(\theta^t)\Delta\theta^t_\epsilon, \Delta\theta^t_\epsilon\rangle + \OM(\epsilon^{1.5}) < \epsilon + \epsilon^{c_2},
\end{equation}
and $\|\Delta\theta^t_\epsilon\| = \OM(\sqrt{\epsilon}) < \epsilon^{c_1}$ for $c_1<0.5$.
This leads to our conclusion.

\subsection{Proof of Proposition \ref{proposition_invariance}} \label{section_proof_proposition_invariance}
Our goal is to show that the continuous-time limit of $\Phi(\theta_s)$ satisfies the same differential equation as $\phi_s$ provided that $\Phi(\theta_0) = \phi_0$.
To do so, first compute the differential equation of ${\Phi(\theta_s)}$
\begin{equation}
\frac{\partial \Phi(\theta_s)}{\partial s} = \nabla_\theta{\Phi}(\theta_s)\dot{\theta}_s = - \nabla_\theta{\Phi}(\theta_s) \HB(\theta_s)^{-1}\nabla F(\theta_s),
\end{equation}
where $\nabla_\theta{\Phi}(\theta_s)$ is the Jacobian matrix of $\Phi(\theta)$ w.r.t. $\theta$ at $\theta = \theta_s$.
We then compute the differential equation of $\phi_s$ (note that $\nabla_\phi{\Phi^{-1}}(\phi_s)$ is the Jacobian matrix of $\Phi^{-1}(\phi)$ w.r.t. $\phi$ at $\phi = \phi_s$)
\begin{align}
\dot \phi_s &\ = -\left[\nabla^2_\phi \SM(\alpha_{\Phi^{-1}(\phi)}, \alpha_{\Phi^{-1}(\phi_s)})_{\vert \phi = \phi_s}\right]^{-1} \nabla_\phi F(\Phi^{-1}(\phi))_{\vert \phi = \phi_s}\notag\\
&\ = -\left[\nabla_\phi{\Phi^{-1}}(\phi_s)^\top \nabla^2_\theta \SM(\alpha_\theta, \alpha_{\theta^s})_{\vert \theta = \theta_s} \nabla_\phi{\Phi^{-1}}(\phi_s)\right]^{-1} \nabla_\phi{\Phi^{-1}}(\phi_s)^\top \nabla F(\theta)_{\vert \theta = \theta_s} \label{eqn_appendix_invariance_proof_1}\\
&\ = - \left[\nabla_\phi{\Phi^{-1}}(\phi_s)\right]^{-1} \nabla^2_\theta \SM(\alpha_\theta, \alpha_{\theta^s})_{\vert \theta = \theta_s} \nabla F(\theta)_{\vert \theta = \theta_s} \notag\\
&\ = - \nabla_\theta{\Phi}(\theta_s) \HB(\theta_s)^{-1}\nabla F(\theta_s) \label{eqn_appendix_invariance_proof_2}\\
&\ = \frac{\partial \Phi(\theta_s)}{\partial s}.\notag
\end{align}
Here we use the following lemma in \eqref{eqn_appendix_invariance_proof_1}. We use $\Phi^{-1}(\phi_s) = \theta_s$ and the inverse function theorem $\nabla_\theta{\Phi}(\theta_s)= \left[\nabla_\phi{\Phi^{-1}}(\phi_s)\right]^{-1}$ in \eqref{eqn_appendix_invariance_proof_2}.
\begin{lemma}\label{lemma_appendix_invariance}
	\begin{equation}
	\nabla^2_\phi \SM(\alpha_{\Phi^{-1}(\phi)}, \alpha_{\Phi^{-1}(\phi_s)})_{\vert \phi = \phi_s} = \nabla_\phi{\Phi^{-1}}(\phi_s)^\top \nabla^2_\theta \SM(\alpha_\theta, \alpha_{\theta^s})_{\vert \theta = \theta_s} \nabla_\phi{\Phi^{-1}}(\phi_s)
	\end{equation}
\end{lemma}
\begin{proof}
	This lemma can be proved with simple computations.
	We compute only for the terms in $\nabla^2_\theta\OTgamma(\alpha_\theta, \alpha_{\theta^s})$ as example. The terms in $\nabla^2_\theta\OTgamma(\alpha_\theta, \alpha_{\theta})$ can be computed similarly.
	Recall the expression
	\begin{equation}
	\begin{aligned}
	\nabla^2_\theta \OTgamma(\alpha_{\theta}, \beta) =\quad  &D^2_{11} \HM_1(f_\theta, \theta)\circ (Df_\theta, Df_\theta) + &D^2_{12} \HM_1(f_\theta, \theta)\circ(Df_\theta, \IM_d) \\
	+ &D^2_{21}\HM_1(f_\theta, \theta)\circ(\IM_d, Df_\theta) + &D^2_{22}\HM_1(f_\theta, \theta)\circ(\IM_d, \IM_d).
	\end{aligned}
	\end{equation}
	We compute
	\begin{equation}
	\begin{aligned}
	\nabla^2_\phi \OTgamma(\alpha_{\Phi^{-1}(\phi)}, \beta) =\  &D^2_{11} \HM_1(f_{\Phi^{-1}(\phi)}, {\Phi^{-1}(\phi)})\circ (Df_{\Phi^{-1}(\phi)}\circ J_{\Phi^{-1}}(\phi), Df_{\Phi^{-1}(\phi)}\circ J_{\Phi^{-1}}(\phi)) \\
	&+ D^2_{12} \HM_1(f_{\Phi^{-1}(\phi)}, {\Phi^{-1}(\phi)})\circ(Df_{\Phi^{-1}(\phi)}\circ J_{\Phi^{-1}}(\phi), J_{\Phi^{-1}}(\phi)) \\
	&+ D^2_{21}\HM_1(f_{\Phi^{-1}(\phi)}, {\Phi^{-1}(\phi)})\circ(J_{\Phi^{-1}}(\phi), Df_{\Phi^{-1}(\phi)}\circ J_{\Phi^{-1}}(\phi)) \\
	&+ D^2_{22}\HM_1(f_{\Phi^{-1}(\phi)}, {\Phi^{-1}(\phi)})\circ(J_{\Phi^{-1}}(\phi), J_{\Phi^{-1}}(\phi)).
	\end{aligned}
	\end{equation}
	Plugging $\Phi^{-1}(\phi_s) = \theta_s$ to the above equality, we have
	\begin{equation}
	\nabla^2_\phi \OTgamma(\alpha_{\Phi^{-1}(\phi)}, \beta)_{\vert \phi = \phi_s} = \nabla_\phi{\Phi^{-1}}(\phi_s)^\top \nabla^2_\theta \OTgamma(\alpha_{\theta}, \beta)_{\vert \theta = \theta_s} \nabla_\phi{\Phi^{-1}}(\phi_s).
	\end{equation}
\end{proof}
\clearpage
\section{Appendix on SIM} \label{appendix_SIM}
\subsection{Proof of Proposition \ref{proposition_SIM_expression}} \label{appendix_proof_of_proposition_SIM_expression}
We will derive the explicit expression of $\nabla^2_\theta \OTgamma(\alpha_{\theta}, \alpha_{\theta^t})_{\vert {\theta} = {\theta^t}}$ based on the dual representation \eqref{eqn_OTgamma_dual_single_variable}.
Recall the definition of the Fr\'echet derivative in Definition \ref{definition_frechet_derivative} and its chain rule $D(f\circ g)(x) = Df(g(x)) \circ Dg(x)$. We compute the first-order gradient by
\begin{equation}
	\nabla_\theta \OTgamma(\alpha_{\theta}, \beta) = \nabla_\theta \HM_1(f_\theta, \theta) = \underbrace{D_1 \HM_1(f_\theta, \theta)\circ Df_\theta}_{\GM_1(f_\theta, \theta)} + \underbrace{D_2 \HM_1(f_\theta, \theta)}_{\GM_2(f_\theta, \theta)},
\end{equation}
where $D_i\HM_1$ denote the Fr\'echet derivative of $\HM_1$ with respect to its $i^{th}$ variable.
Importantly, the optimality condition of \eqref{eqn_OTgamma_dual_single_variable} implies that $D_1 \HM_1(f_\theta, \theta)[g] = 0, \forall g\in\CM(\XM)$.\\
Further, in order to compute the second order gradient of $\OTgamma(\alpha_{\theta}, \beta)$ with respect to $\theta$, we first compute the gradient of $\GM_i, i=1,2$:
\begin{align}
	\nabla_\theta \GM_1(f_\theta, \theta) & = D_1 \HM_1(f_\theta, \theta)\circ D^2f_\theta + D^2_{11} \HM_1(f_\theta, \theta)\circ (Df_\theta, Df_\theta) + D^2_{12} \HM_2(f_\theta, \theta)\circ(Df_\theta, \IM_d), \label{eqn_frechet_derivative_i}\\
	\nabla_\theta \GM_2(f_\theta, \theta) & = D^2_{21}\HM_1(f_\theta, \theta)\circ(\IM_d, Df_\theta) + D^2_{22}\HM_1(f_\theta, \theta)\circ(\IM_d, \IM_d).
\end{align}
Using the fact that $D_1 \HM_1(f_\theta, \theta)[g] = 0, \forall g\in\CM(\XM)$, we can drop the first term in the R.H.S. of \eqref{eqn_frechet_derivative_i}.
Combining the above results, we have
\begin{equation}
	\begin{aligned}
	\nabla^2_\theta \OTgamma(\alpha_{\theta}, \beta) =\quad  &D^2_{11} \HM_1(f_\theta, \theta)\circ (Df_\theta, Df_\theta) + &D^2_{12} \HM_1(f_\theta, \theta)\circ(Df_\theta, \IM_d) \\
	+ &D^2_{21}\HM_1(f_\theta, \theta)\circ(\IM_d, Df_\theta) + &D^2_{22}\HM_1(f_\theta, \theta)\circ(\IM_d, \IM_d).
	\end{aligned}
\end{equation}
Moreover, we can further simplify the above expression by noting that for any $g\in T(\RBB^d, \CM(\XM))$, i.e. any bounded linear operators from $\RBB^d$ to $\CM(\XM)$,
\begin{equation}
	\nabla_\theta  \left(D_1 \HM_1(f_\theta, \theta)\circ g \right) = D^2_{11}\HM_1(f_\theta, \theta)\circ(g, Df_\theta) + D^2_{12}\HM_1(f_\theta, \theta)\circ(g, \IM_d) = 0.
\end{equation}
Plugging in $g = Df_\theta$ in the above equality we have 
\begin{equation} \label{eqn_appendix_D_11_D12}
	D^2_{11}\HM_1(f_\theta, \theta)\circ(Df_\theta, Df_\theta) = - D^2_{12}\HM_1(f_\theta, \theta)\circ(Df_\theta, \IM_d).
\end{equation}
Consequently we derive (we omit the identity operator $(\IM_d, \IM_d)$ for the second term)
\begin{equation} \label{eqn_appendix_two_term_OT_gamma}
\nabla^2_\theta \OTgamma(\alpha_{\theta}, \beta) = - D^2_{11} \HM_1(f_\theta, \theta)\circ (Df_\theta, Df_\theta) + D^2_{22}\HM_1(f_\theta, \theta),
\end{equation}
where we note that $D^2_{12}\HM_1(f_\theta, \theta)\circ(Df_\theta, \IM_d)$ is symmetric from \eqref{eqn_appendix_D_11_D12} and
\begin{equation}
	D^2_{21}\HM_1(f_\theta, \theta)\circ(\IM_d, Df_\theta) = \left[D^2_{12} \HM_1(f_\theta, \theta)\circ(Df_\theta, \IM_d)\right]^\top = D^2_{12} \HM_1(f_\theta, \theta)\circ(Df_\theta, \IM_d).
\end{equation}
These two terms can be computed explicitly and involve only simple function operations like $\exp$ and $\log$ and integration with respect to $\alpha_\theta$ and $\beta$, as discussed in the following.

\subsubsection{Explicit Expression of $\nabla^2_\theta \OTgamma(\alpha_{\theta}, \beta)$}
Denote $\AB_1 = D^2_{11} \HM_1(f_\theta, \theta)\circ (Df_\theta, Df_\theta)$ as the first term of \eqref{eqn_appendix_two_term_OT_gamma}.
We note that $\AB_1\in\RBB^{d\times d}$ is a matrix and hence is a bilinear operator.
If we can compute $h_1^\top\AB_1 h_2$ for any two directions $h_1, h_2\in\RBB^d$, we are able to compute entries of $\AB_1$ by taking $h_1$ and $h_2$ to be the canonical bases.
We compute this quantity $h_1^\top\AB_1 h_2$ as follows.

For a fixed $y\in\XM$, denote $\TM_y:\XM\times\CM(\XM)\rightarrow\RBB$ by $$\TM_y(x, f) \defi \exp(-c(x, y)/\gamma) \exp(f(x)/\gamma).$$
Denote $g_1 = Df_\theta[h_1] \in\CM(\XM)$ for some direction $h_1\in \RBB^d$ (recall that $Df_\theta\in T(\RBB^d, \CM(\XM))$, where $T(V, W)$ is the family of bounded linear operators from set $V$ to set $W$).
Use the chain rule of Fr\'echet derivative to compute 
\begin{equation}
	\big(D_1\AM(f, \alpha_\theta) [g_1]\big)  (y) = - \frac{\int_\XM \TM_y(x, f)g_1(x)\dB\alpha_\theta(x)}{\int_\XM \TM_y(x, f)\dB\alpha_\theta(x)}.
\end{equation}
Let $h_2\in\RBB^d$ be another direction and denote $g_2 = Df_\theta[h_2] \in\CM(\XM)$.
We compute
\begin{align}
	&\left(D^2_{11}\AM(f, \alpha_\theta)[g_1,  g_2]\right)(y) \notag\\
	= &\frac{\int_{\XM} \TM_y(x, f) g_1(x) g_2(x)\dB\alpha_\theta(x)}{\gamma\int_\XM \TM_y(x, f)\dB\alpha_\theta(x)} - \frac{\int_{\XM^2} \TM_y(x, f)\TM_y(x', f) g_1(x)g_2(x')\dB\alpha_\theta(x)\dB\alpha_\theta(x')}{\gamma\left[\int_\XM \TM_y(x, f)\dB\alpha_\theta(x)\right]^2}. \label{eqn_appendix_D11A}
\end{align}
Moreover, for any two directions $h_1, h_2\in\RBB^d$, we compute $D_{11}^2\HM_1(f, \theta)\big[Df_\theta[h_1], Df_\theta[h_2]\big]$ by
\begin{equation} \label{eqn_appendix_D_11_H}
	D_{11}^2 \HM_1(f, \theta)\big[Df_\theta[h_1], Df_\theta[h_2]\big] = \int_\XM \left(D^2_{11}\AM(f_\theta, \alpha_\theta)\big[Df_\theta[h_1],  Df_\theta[h_2]\big]\right)(y) \dB \beta(y),
\end{equation}
which by plugging in \eqref{eqn_appendix_D11A} yields closed a form expression with only simple function operations like $\exp$ and $\log$ and integration with respect to $\alpha_\theta$ and $\beta$.


We then compute the second term of \eqref{eqn_appendix_two_term_OT_gamma}.
Using the change-of-variable formula, we have
\begin{equation}
	\AM(f, {T_\theta}_\sharp\mu)(y) = 	-\gamma\log\int_\ZM\exp\left(-\frac{1}{\gamma}c(T_\theta(z), y) + \frac{1}{\gamma} f(T_\theta(z))\right)\dB\mu(z).
\end{equation}
For any $f\in\CM(\XM)$, the first-order Fr\'echet derivative of $\HM_1(f, \theta)$ w.r.t. its second variable is given by
\begin{align*}
	D_2\HM_1(f, \theta) = \int_\ZM& \langle\nabla_\theta T_\theta(z), \nabla f\big(T_\theta(z)\big)\rangle \dB \mu (z)\\
	&+ \int_\XM\frac{\int_\ZM \TM_y\big(T_\theta(z), f\big)\big\langle\nabla_\theta T_\theta(z), \nabla_{1} c\big(T_\theta(z), y\big) - \nabla f\big(T_\theta(z)\big)\big\rangle \dB \mu(z)}{\int_\ZM \TM_y\big(T_\theta(z), f\big)\dB \mu(z)} \dB \beta(y).
\end{align*}
Denote $u_z(\theta, f) = \nabla_{1} c\big(T_\theta(z), y\big) - \nabla f\big(T_\theta(z)\big)$.
The second-order Fr\'echet derivative is given by
\begin{align}
	&D^2_{22}\HM_1(f, \theta) \label{eqn_appendix_D_22_H}\\
	=& \int_\ZM \nabla^2_\theta T_\theta(z)\times_1\nabla f\big(T_\theta(z)\big) + \nabla_\theta T_\theta(z)^\top\nabla^2 f\big(T_\theta(z)\big)\nabla_\theta T_\theta(z) \dB \mu (z) \notag\\
	&+ \frac{1}{\gamma} \int_\XM\frac{\int_\ZM \TM_y\big(T_\theta(z), f\big)\nabla_\theta T_\theta(z)^\top u_z(\theta, f)u_z(\theta, f)^\top \nabla_\theta T_\theta(z)\dB \mu(z)}{\int_\ZM \TM_y\big(T_\theta(z), f\big)\dB \mu(z)} \dB \beta(y) \notag\\
	&+ \int_\XM\frac{\int_\ZM \TM_y\big(T_\theta(z), f\big)\nabla^2_\theta T_\theta(z)\times_1 u_z(\theta, f)\dB \mu(z)}{\int_\ZM \TM_y\big(T_\theta(z), f\big)\dB \mu(z)} \dB \beta(y)\notag\\
	&+ \int_\XM\frac{\int_\ZM \TM_y\big(T_\theta(z), f\big)\nabla_\theta T_\theta(z)^\top [\nabla_{11} c(T_\theta(z), y) - \nabla^2 f\big(T_\theta(z)\big)]\nabla_\theta T_\theta(z)\dB \mu(z)}{\int_\ZM \TM_y\big(T_\theta(z), f\big)\dB \mu(z)} \dB \beta(y) \notag\\
	&+ \frac 1 \gamma \int_\XM\frac{\int_\ZM \TM_y\big(T_\theta(z), f\big)\nabla_\theta T_\theta(z)^\top u_z(\theta, f)\dB \mu(z)\left[\int_\ZM \TM_y\big(T_\theta(z), f\big)\nabla_\theta T_\theta(z)^\top u_z(\theta, f)\dB \mu(z)\right]^\top}{\left[\int_\ZM \TM_y\big(T_\theta(z), f\big)\dB \mu(z)\right]^2} \dB \beta(y). \notag
\end{align}
Here $\nabla_\theta T_\theta(z)\in\RBB^{q\times d}$ and $\nabla^2_\theta T_\theta(z)\in\RBB^{q\times d\times d}$ denote the first and second order Jacobian of $T_\theta(z)$ w.r.t. to $\theta$; $\times_1$ denotes the tensor product along the first dimension; $\nabla f \in\RBB^{q}$ and $\nabla^2 f\in\RBB^{q\times q}$ denote the first and second order gradient of $f$ w.r.t. its input; $\nabla_{1} c\in\RBB^{q}$ and $\nabla_{11} c\in\RBB^{q\times q}$ denote the first and second order gradient of $c$ w.r.t. its first input. By plugging in $f = f_\theta$, we have the explicit expression of the second term of \eqref{eqn_appendix_two_term_OT_gamma}.
\subsection{More details in Proposition \ref{proposition_sinkhorn_potential}} \label{appendix_proposition_sinkhorn_potential}
First, we recall some existing results about the Sinkhorn potential $f_\theta$.
\begin{assumption}\label{ass_bounded_c}
	The ground cost function $c$ is bounded and we denote $M_c\defi\max_{x,y\in\XM} c(x, y)$.
\end{assumption}
It is known that, under the above boundedness assumption on the ground cost function $c$, $f_\theta$ is a solution to the generalized DAD problem (eq. (7.4) in \citep{lemmens2012nonlinear}), which is the fixed point to the operator $\BM:\CM(\XM)\times\Theta\rightarrow\CM(\XM)$ defined as
\begin{equation} 
	\BM(f, \theta) \defi \AM\big(\AM(f, \alpha_\theta), \beta\big).
\end{equation}
Further, the Birkhoff-Hopf Theorem (Sections A.4 and A.7 in \citep{lemmens2012nonlinear}) states that $\exp(\BM/\gamma)$ is a contraction operator under the Hilbert metric with a contraction factor $\lambda^2$ where $\lambda\defi\frac{\exp(M_c/\gamma)-1}{\exp(M_c/\gamma)+1}<1$ (see also Theorem B.5 in \citep{NIPS2019_9130}):
For strictly positive functions $u, u'\in\CM(\XM)$, define the Hilbert metric as 
\begin{equation} \label{eqn_hilbert_metric}
	d_H(u, u') \defi \log \max_{x, y\in\XM} \frac{u(x) u'(y)}{u'(x) u(y)}.
\end{equation}
For any measure $\alpha\in\MM_1^+(\XM)$, we have
\begin{equation}\label{eqn_contraction_under_hilbert_metric}
	d_H(\exp(\AM(f, \alpha_\theta)/\gamma), \exp(\AM(f', \alpha_\theta)/\gamma))\leq \lambda d_H(\exp(f/\gamma), \exp(f'/\gamma)).
\end{equation}
Consequently, by applying the fixed point iteration
\begin{equation}
	f^{t+1} = \BM(f^t, \theta),
\end{equation}
also known as the Sinkhorn-Knopp algorithm,
one can compute $f_\theta$ in logarithmic time: $\|f^{t+1} - f_\theta\|_\infty=\OM(\lambda^t)$ (Theorem. 7.1.4 in \citep{lemmens2012nonlinear} and Theorem B.10 in \citep{NIPS2019_9130}).

While the above discussion shows that the output of the Sinkhorn-Knopp algorithm well approximates the Sinkhorn potential $f_\theta$, it would be useful to discuss more about the boundedness property of the sequence $\{f^t \}$ produced by the above Sinkhorn-Knopp algorithm.
We first show that under bounded initialization $f^0$, the entire sequence $\{f^t \}$ is bounded.
\begin{lemma} \label{lemma_appendix_bounded_function_sequence_from_sinkhorn_knopp}
	Suppose that we initialize the Sinkhorn-Knopp algorithm with $f^0\in\CM(\XM)$ such that $\|f^0\|_\infty\leq M_c$.
	One has  $\|f^t\|_\infty\leq M_c$, for $t=1, 2, 3, \cdots$. 
\end{lemma}
\begin{proof}
	For $\|f\|_\infty \leq M_c$ and any measure $\alpha\in\MM_1^+(\XM)$, we have
	\begin{equation*}
		\|\AM(f, \alpha)\|_\infty = \gamma \|\log \int_\XM \exp\{-c(x, \cdot)/\gamma \}\exp\{f(x)/\gamma \}\dB\alpha(x)\|_\infty \leq \gamma \log \exp(M_c/\gamma) \leq M_c.
	\end{equation*}
	One can then check the lemma via induction.
\end{proof}
We then show that the sequence $\{f^t \}$ has bounded first, second and third-order gradients under the following assumptions on the ground cost function $c$.
\begin{assumption}\label{ass_bounded_infty_c_gradient}
	The cost function $c$ is $G_c$-Lipschitz continuous with respect to one of its inputs: For all $x, x' \in \XM$, $$|c(x, y) - c(x', y)|\leq G_c\|x - x'\|.$$
\end{assumption}
\begin{assumption}\label{ass_bounded_infty_c_hessian}
	The gradient of the cost function $c$ is $L_c$-Lipschitz continuous: for all $x, x' \in \XM$, $$\|\nabla_1 c(x, y) - \nabla_1 c(x', y)\|\leq L_c\|x - x'\|.$$
\end{assumption}
\begin{assumption}\label{ass_lipschitz_c_hessian}
	The Hessian matrix of the cost function $c$ is $L_{2, c}$-Lipschitz continuous: for all $x, x' \in \XM$, $$\|\nabla^2_{11} c(x, y) - \nabla^2_{11} c(x', y)\|\leq L_{2, c}\|x - x'\|.$$
\end{assumption}
\begin{lemma} \label{lemma_appendix_sinkhorn_potential_boundedness}
	Assume that the initialization $f^0\in\CM(\XM)$ satisfies $\|f^0\|_\infty\leq M_c$.\\
	(i.) Under Assumptions \ref{ass_bounded_c} and \ref{ass_bounded_infty_c_gradient}, $\exists G_f$ such that $\|\nabla f^t\|_{2, \infty}\leq G_f, \forall t>0$.\\
	(ii.) Under Assumptions \ref{ass_bounded_c} - \ref{ass_bounded_infty_c_hessian}, $\exists L_f$ such that $\|\nabla^2 f^t(x)\|\leq L_f, \forall t>0$.\\
	(iii.) Under Assumptions \ref{ass_bounded_c} - \ref{ass_lipschitz_c_hessian}, $\exists L_{2, f}$ such that $\|\nabla^2 f^t(x) - \nabla^2 f^t(y)\|_{op}\leq L_{2, f}\|x - y\|, \forall t>0$.\\
	(iv). For $\|f\|_\infty\leq M_c$, the function $\BM(f, \theta)(x)$ is $G_f$-Lipschitz continuous.
\end{lemma}
\begin{proof}
	We denote $k(x, y) \defi \exp\{-c(x, y)/\gamma\}$ in this proof.
	
	(i) Under Assumptions \ref{ass_bounded_c} and \ref{ass_bounded_infty_c_gradient}, $k$ is $[G_c/\gamma]$-Lipschitz continuous w.r.t. its first variable.
	For $f\in\CM(\XM)$ such that $\|f\|_\infty\leq M_c$, we bound
	\begin{align*}
		|\AM(f, \alpha)(x) - \AM(f, \alpha)(y)| &= \gamma | \log\int_\XM [k(z, y) - k(z, x)]\exp\{f(z)/\gamma\}\dB \alpha(z) |\\
		&\leq \gamma \exp(M_c/\gamma) G_c/\gamma \|x - y\|_2 = \exp(M_c/\gamma) G_c \|x - y\|_2.
	\end{align*}
	Using Lemma \ref{lemma_appendix_bounded_function_sequence_from_sinkhorn_knopp}, we know that $\{f^t\}$ is $M_c$-bounded and hence $$\|\nabla f^{t+1}\|_{2, \infty} \leq G_f = \exp(2M_c/\gamma) G^2_c.$$

	(ii) 
	Under Assumption \ref{ass_bounded_c}, $k(x, y)\geq \exp(-M_c/\gamma)$.
	We compute
	\begin{align*}
		\nabla \big(\AM(f, \alpha)\big)(x) = \frac{\int_\XM k(z, x)\exp\{f(z)/\gamma\} \nabla_{1}c(x, z)\dB \alpha(z)}{\int_\XM k(z, x)\exp\{f(z)/\gamma\}\dB \alpha(z)}. && \#\ \frac{g_1(x)}{g_2(x)}
	\end{align*}
	Let $g_1:\RBB^q\rightarrow \RBB^q$ and $g_2:\RBB^q\rightarrow\RBB$ be the numerator and denominator of the above expression.
	If we have (a) $\|g_1\|_{2, \infty}\leq G_1$, (b) $\|g_1(x) - g_1(y)\|\leq L_1\|x - y\|$ and (c) $\|g_2\|_{\infty}\leq G_2$, (d) $|g_2(x) - g_2(y)|\leq L_2\|x - y\|$, (e) $g_2\geq \bar G_2 >0$, we can bound
	\begin{equation} \label{eqn_Lipschitz_continuous_fraction_of_functions}
		\|\frac{g_1(x)}{g_2(x)} - \frac{g_1(y)}{g_2(y)}\| = \|\frac{g_1(x)g_2(y)- g_1(y)g_2(x)}{g_2(x)g_2(y)}\|
		\leq \frac{G_2L_1 + G_1L_2}{{\bar G_2}^2}\|x - y\|,
	\end{equation}
	which means that $\nabla \big(\AM(f, \alpha)\big)$ is $L$-Lipschitz continuous with $L = \frac{G_2L_1 + G_1L_2}{{\bar G_2}^2}$.
	We now prove (a)-(e).
	\begin{itemize}
		\item[(a)] $\|\int_\XM k(z, x)\exp\{f(z)/\gamma\} \nabla_{1}c(x, z)\dB \alpha(z)\|_{2,\infty} \leq \exp(M_c/\gamma)\cdot G_c$ (Assumption \ref{ass_bounded_infty_c_gradient}).
		\item[(b)] Note that for any two bounded and Lipschitz continuous functions $h_1:\XM\rightarrow\RBB$ and $h_2:\XM\rightarrow\RBB^q$, their product is also Lipschitz continuous:
		\begin{equation}
			\|h_1(x)\cdot h_2(x) - h_1(y)\cdot h_2(y)\| \leq [|h_1|_\infty\cdot G_{h_2}+ \|h_2\|_{2,\infty}\cdot G_{h_1}]\|x-y\|,
		\end{equation}
		where $G_{h_i}$ denotes the Lipschitz constant of $h_i$, $i = 1, 2$.
		Hence for $g_1$, we have
		\begin{equation*}
			\|g_1(x) - g_1(y)\|\leq \exp(M_c/\gamma)\cdot(L_c + G^2_c/\gamma)\cdot\|x-y\|,
		\end{equation*}
		since $k(x, y)\leq 1$, $\|\nabla_{1} k(x, y)\|\leq G_c/\gamma$, $\|\nabla_{1} c(x, y)\|\leq G_c$, $\|\nabla^2_{11}c(x, y)\|_{op}\leq L_c$.
		\item[(c)] $\|\int_\XM k(z, \cdot)\exp\{f(z)/\gamma\}\dB \alpha(z)\|_{\infty}\leq \exp(M_c/\gamma)$.
		\item[(d)] $|\int_\XM [k(z, x) - k(z, y)]\exp\{f(z)/\gamma\}\dB \alpha(z)|\leq \exp(M_c/\gamma) \cdot G_c/\gamma\cdot\|x - y\|$.
		\item[(e)] $\int_\XM k(z, x)\exp\{f(z)/\gamma\}\dB \alpha(z) \geq \exp(-2M_c/\gamma) > 0$.
	\end{itemize}
	Combining the above points, we prove the existence of $L_f$.
	
	For (iii), compute that
	\begin{align*}
		&\ \nabla^2 \big(\AM(f, \alpha)\big)(x)\\
		 =\ &\ \frac{\int_\XM k(z, x)\exp\{f(z)/\gamma\}\nabla_{1} c(x, z) \nabla_{1} c(x, z)^\top\dB \alpha(z)}{\int_\XM k(z, x)\exp\{f(z)/\gamma\}\dB \alpha(z)} && \#1\\
		 &\ + \frac{\int_\XM k(z, x)\exp\{f(z)/\gamma\}\nabla^2_{11}c(x, z)\dB \alpha(z)}{\int_\XM k(z, x)\exp\{f(z)/\gamma\}\dB \alpha(z)} && \#2\\
		&\ - \frac{\int_\XM k(z, x)\exp\{f(z)/\gamma\}\nabla_{1} c(x, z)\dB \alpha(z)\left[\int_\XM k(z, x)\exp\{f(z)/\gamma\}\nabla_{1} c(x, z)\dB \alpha(z)\right]^\top}{\left[\int_\XM k(z, x)\exp\{f(z)/\gamma\}\dB \alpha(z)\right]^2}. && \#3
	\end{align*}
	We now analyze $\#1$-$\#3$ individually.
	\begin{itemize}
		\item[$\#1$] Note that for any two bounded and Lipschitz continuous functions $h_1:\XM\rightarrow\RBB$ and $h_2:\XM\rightarrow\RBB^{q\times q}$, their product is also Lipschitz continuous:
		\begin{equation}
		\|h_1(x)\cdot h_2(x) - h_1(y)\cdot h_2(y)\|_{op} \leq [|h_1|_\infty\cdot G_{h_2}+ \|h_2\|_{op,\infty}\cdot G_{h_1}]\|x-y\|,
		\end{equation}
		where $G_{h_i}$ denotes the Lipschitz constant of $h_i$, $i = 1, 2$.
		
		Take $h_1(x) = k(z', x)\exp\{f(z')/\gamma\}/\int_\XM k(z, x)\exp\{f(z)/\gamma\}\dB \alpha(z)$.
		$h_1$ is bounded since $k(z', x)\leq 1$ and $\int_\XM k(z, x)\exp\{f(z)/\gamma\}\dB \alpha(z) \geq \exp(-2M_c/\gamma) > 0$.
		$h_1$ is Lipschitz continuous since we additionally have $k(z', x)$ being Lipschitz continuous (see \eqref{eqn_Lipschitz_continuous_fraction_of_functions}).
		
		Take $h_2(x) = \nabla_{1} c(x, z) \nabla_{1} c(x, z)^\top$. $h_2$ is bounded since $\|\nabla_{1} c(x, z)\|\leq G_c$ (Assumption \ref{ass_bounded_infty_c_gradient}). $h_2$ is Lipschitz continuous due to Assumption \ref{ass_bounded_infty_c_hessian}.
		\item[$\#2$] Following the similar argument as $\#1$, we have the result. Note that $h_2(x) = \nabla^2_{11}c(x, z)$ is Lipschitz continuous due to Assumption \ref{ass_lipschitz_c_hessian}.
		\item[$\#3$] We follow the similar argument as $\#1$ by taking 
		\[
		h_1(x) = \frac{ k(z', x)\exp\{f(z')/\gamma\} k(z', x)\exp\{f(z')/\gamma\}}{\left[\int_\XM k(z, x)\exp\{f(z)/\gamma\}\dB \alpha(z)\right]^2},
		\]
		and taking
		\[
		h_2(x) = \nabla_{1} c(x, z) [\nabla_{1} c(x, z)]^\top.
		\]
	\end{itemize}
		Combining the above points, we prove the existence of $L_{2,f}$.

	
	(iv) As a composition of $\AM$, we also have that $\BM(f, \theta)$ is $G_f$-Lipschitz continuous (see $G_f$ in (i)).
\end{proof}
Moreover, based on the above continuity results, we can show that the first-order gradient $\nabla f_\theta^\epsilon$ (and second-order gradient $\nabla^2 f_\theta^\epsilon$) also converges to $\nabla f_\theta$ (and $\nabla^2 f_\theta$) in time logarithmically depending on $1/\epsilon$.
\begin{lemma} \label{lemma_convergence_of_gradient}
	Under Assumptions \ref{ass_bounded_c}-\ref{ass_bounded_infty_c_hessian}, the Sinkhorn-Knopp algorithm, i.e. the fixed point iteration
	\begin{equation}
	f^{t+1} = \BM(f^t, \theta),
	\end{equation}
	computes $\nabla f_\theta$ in logarithm time: $\|\nabla f^{t+1} - \nabla  f_\theta\|_{2, \infty}= \epsilon$ with $t = \OM(\log \frac{1}{\epsilon})$.
\end{lemma}
\begin{proof}
	For a fix point $x\in\XM$ and any direction $h\in\RBB^q$, we have
	\begin{equation*}
		f^t(x+\eta\cdot h) - f^t(x) = \eta [\nabla f^t(x)]^\top h + \frac{\eta^2}{2} h^\top \nabla^2 f^t(x + \tilde{\eta}_1\cdot h) h,
	\end{equation*}
	where $\eta>0$ is some constant to be determined later and $0\leq \tilde{\eta}_1\leq \eta$ is obtained from the mean value theorem.
	Similarly, we have for $0\leq \tilde{\eta}_2\leq \eta$
	\begin{equation*}
		f_\theta(x+\eta\cdot h) - f_\theta(x) = \eta [\nabla f_\theta(x)]^\top h + \frac{\eta^2}{2} h^\top \nabla^2 f_\theta(x + \tilde{\eta}_2\cdot h) h.
	\end{equation*}
	We can then compute
	\begin{equation*}
		|[\nabla f^t(x) - \nabla f_\theta(x)]^\top h| \leq \frac{2}{\eta}\|f^t - f_\theta\|_\infty + \eta L_f \|h\|^2.
	\end{equation*}
	Take $h = \nabla f^t(x) - \nabla f_\theta(x)$ and $\eta = \frac{2}{L_f}$. We derive from the above inequality
	\begin{equation*}
		\|\nabla f^t(x) - \nabla f_\theta(x)\|^2 \leq 2L_f\|f^t - f_\theta\|_\infty.
	\end{equation*}
	Consequently, if we have $2L_f\|f^t - f_\theta\|_\infty\leq \epsilon^2$, we can prove that $\|\nabla f^t - \nabla f_\theta\|_{2, \infty}\leq \epsilon$ since $x$ is arbitrary. This can be achieve in logarithmic time using the Sinkhorn-Knopp algorithm.
\end{proof}
\begin{lemma} \label{lemma_convergence_of_Hessian}
	Under Assumptions \ref{ass_bounded_c}-\ref{ass_lipschitz_c_hessian}, the Sinkhorn-Knopp algorithm, i.e. the fixed point iteration
	\begin{equation}
		f^{t+1} = \BM(f^t, \theta),
	\end{equation}
	computes $\nabla^2 f_\theta$ in logarithm time: $\|\nabla^2 f^{t+1} - \nabla^2  f_\theta\|_{op, \infty}=\epsilon$ with $t = \OM(\log \frac{1}{\epsilon})$.
\end{lemma}
\begin{proof}
	This follows a similar argument as Lemma \ref{lemma_convergence_of_gradient} by noticing that the third order gradient of $f^t$ (and $f_\theta$) is bounded due to Assumption \ref{ass_lipschitz_c_hessian}.
\end{proof}

\subsection{Proof of Proposition \ref{proposition_frechet_derivative}} \label{appendix_proof_of_proposition_frechet_derivative}
We now construct a sequence $\{g^t\}$ to approximate the Fr\'echet derivative of the Sinkhorn potential $Df_\theta$ such that for all $t\geq T(\epsilon)$ with some integer function $T(\epsilon)$ of the target accuracy $\epsilon$, we have $\|g^t_{\theta} - Df_\theta\|_{op} \leq \epsilon$. 
In particular, we show that such $\epsilon$-accurate approximation can be achieved using a logarithmic amount of simple function operations and integrations with respect to $\alpha_\theta$.

For a given target accuracy $\epsilon>0$, denote $\bar \epsilon = \epsilon/{L}_{l}$, where ${L}_{l}$ is a constant defined in Lemma \ref{lemma_appendix_lipschitz_D2EM}.
First, Use the Sinkhorn-Knopp algorithm to compute $f_\theta^{\bar \epsilon}$, an approximation of $f_\theta$ such that $\|f_\theta^{\bar \epsilon}-f_\theta\|_\infty\leq \bar \epsilon$. This computation can be done in $\OM(\log \frac{1}{\epsilon})$ from Proposition \ref{proposition_sinkhorn_potential}.

Denote $\EM(f, \theta) = \BM^l(f, \theta) = \BM\big(\cdots\BM(f, \theta),\cdots, \theta\big)$, the $l$ times composition of $\BM$ in its first variable. Pick $l = \lceil\log_{\lambda}\frac{1}{3}\rceil/2$. From the contraction of $\AM$ under the Hilbert metric \eqref{eqn_contraction_under_hilbert_metric}, we have
\begin{align*}
\|\EM(f, \theta) - \EM(f', \theta)\|_\infty &\ \leq \gamma d_H(\exp(\EM(f, \theta)/\gamma), \exp(\EM(f', \theta)/\gamma))\\
&\ \leq \gamma\lambda^{2l} d_H(\exp(f/\gamma), \exp(f'/\gamma))\leq 2\lambda^{2l}\|f - f'\|_\infty \leq \frac{2}{3}\|f - f'\|_\infty,
\end{align*}
where we use $\|f - f'\|_\infty\leq d_H(\exp(f), \exp(f'))\leq 2\|f - f'\|_\infty$ in the first and third inequalities.
Consequently, $\EM[f, \theta]$ is a contraction operator w.r.t. $f$ under the $l_\infty$ norm, which is equivalent to
\begin{equation} \label{eqn_contraction_EM}
\|D^1\EM(f, \theta)\|_{op} \leq \frac{2}{3}.
\end{equation}

Now, given arbitrary initialization $g_{\theta}^0:\Theta\rightarrow T(\RBB^d, \CM(\XM))$\footnote{Recall that $T(\RBB^d, \CM(\XM))$ is the family of bounded linear operators from $\RBB^d$ to $\CM(\XM)$}, construct iteratively 
\begin{equation}
g_{\theta}^{t+1} = D_1\EM(f_\theta^{\bar \epsilon}, \theta)\circ g_{\theta}^{t} + D_2\EM(f_\theta^{\bar \epsilon}, \theta),
\end{equation}
where $\circ$ denotes the composition of (linear) mappings.
In the following, we show that 
\begin{equation*}
	\|g_{\theta}^{t+1} - Df_\theta\|_{op}\leq 3\epsilon + (\frac{2}{3})^t\|g_{\theta}^0 - Df_\theta\|_{op}.
\end{equation*}
First, note that $f_\theta$ is a fixed point of $\EM(\cdot, \theta)$
\begin{equation*}
	f_\theta = \EM(f_\theta, \theta).
\end{equation*}
Take the Fr\'echet derivative w.r.t. $\theta$ on both sides of the above equation.
Using the chain rule, we compute
\begin{equation}
Df_\theta = D_1\EM(f_\theta, \theta)\circ Df_\theta + D_2\EM(f_\theta, \theta).
\end{equation}
For any direction $h\in\RBB^d$, we bound the difference of the directional derivatives by
\begin{align*}
&\ \|g_{\theta}^{t+1}[h] - Df_\theta[h]\|_\infty \\
\leq&\ \|D_1\EM(f_\theta, \theta)\big[ Df_\theta[h]\big] - D_1\EM(f_\theta^{\bar \epsilon}, \theta)\big[g_{\theta}^{t}[h]\big]\|_\infty + \|D_2\EM(f_\theta^{\bar \epsilon}, \theta)[h] - D_2\EM(f_\theta, \theta)[h]\|_\infty\\
\leq&\ \frac{2}{3}\|Df_\theta[h] - g_{\theta}^{t}[h]\|_\infty + {L}_{l}\big(\|f_\theta^{\bar \epsilon} - f_\theta\|_\infty + \|\nabla f_\theta^{\bar \epsilon} - \nabla f_\theta\|_\infty\big)\|h\|_\infty \\
\leq&\ \frac{2}{3}\|Df_\theta - g_{\theta}^{t}\|_{op}\|h\|_\infty + \epsilon\|h\|_\infty,
\end{align*}
where in the second inequality we use the bound on $D_1\EM$ in \eqref{eqn_contraction_EM} and the {${L}_{l}$-Lipschitz continuity of $D_2\EM$} with respect to its first argument (recall that $f_\theta^{\bar \epsilon}$ is obtained from the Sinkhorn-Knopp algorithm and hence $\|f_\theta^{\bar \epsilon}\|_\infty\leq M_c$ from Lemma \ref{lemma_appendix_bounded_function_sequence_from_sinkhorn_knopp} and $\|\nabla f_\theta^{\bar \epsilon}\|_{2, \infty}\leq G_f$ from (i) of Lemma \ref{lemma_appendix_sinkhorn_potential_boundedness}).
The above inequality is equivalent to 
\begin{equation*}
\|g_{\theta}^{t+1} - Df_\theta\|_{op} - 3\epsilon \leq \frac{2}{3}\left(\|Df_\theta - g_{\theta}^{t}\|_{op} - 3\epsilon \right) \Rightarrow \|g_{\theta}^{t+1} - Df_\theta\|_{op}\leq 3\epsilon + (\frac{2}{3})^t\|g_{\theta}^0 - Df_\theta\|_{op}.
\end{equation*}
Therefore, after $T(\epsilon) = \OM(\log\frac{1}{\epsilon})$ iterations, we find $g_{\theta}^{T(\epsilon)}$ such that $\|g_{\theta}^{T(\epsilon)} - Df_\theta\|_{op} \leq 4\epsilon$.

\begin{assumption}[Boundedness of $\nabla_\theta T_\theta(x)$] \label{ass_Lipschitz_continuity_T}
	There exists some $G_T>0$ such that for any $x\in\XM$ and $\theta\in\Theta$, $\|\nabla_\theta T_\theta(x)\|_{op}\leq G_T$.
\end{assumption}

\begin{lemma}[Lipschitz continuity of $D_2\EM$] \label{lemma_appendix_lipschitz_D2EM}
	Under Assumptions \ref{ass_bounded_c} - \ref{ass_bounded_infty_c_hessian} and \ref{ass_Lipschitz_continuity_T},
	$D_2\EM$ is Lipschitz continuous with respect to its first variable:
	For $f, f'\in\CM(\XM)$ such that $\|f\|_\infty \leq M_c$ ($\|f'\|_\infty \leq M_c$) and $\|\nabla f\|_\infty\leq G_f$ ($\|\nabla f'\|_\infty\leq G_f$), and $\theta\in\Theta$ there exists some $L_l$ such that
	\begin{equation}
		\|D_2\EM(f, \theta) - D_2\EM(f', \theta)\|_{op} \leq L_l \big(\|f - f'\|_\infty + \|\nabla f - \nabla f'\|_{2, \infty}\big).
	\end{equation}
\end{lemma}
\begin{proof}
	Recall that $\EM(\cdot, \theta) = \BM^l(\cdot, \theta)$. Using the chain rule of Fr\'echet derivative, we compute
\begin{equation} \label{eqn_appendix_proof_D2E_expression}
	D_2\BM^{l}(f, \theta) = D_1\BM\big(\BM^{l-1}(f, \theta), \theta\big)\circ D_2\BM^{l-1}(f, \theta) + D_2\BM\big(\BM^{l-1}(f, \theta), \theta\big).
\end{equation}
We bound the two terms on the R.H.S. individually.

\paragraph{Analyze the first term of \eqref{eqn_appendix_proof_D2E_expression}.}
For a given $f$, use $A_f$ and $B_f$ to denote two linear operators depending on $f$.
We have $\|A_f\circ B_f - A_{f'}\circ B_{f'}\|_{op} = \OM(\|f - f'\|_\infty + \|\nabla f - \nabla f'\|_{2, \infty})$ if both $A_f$ and $B_f$ are bounded, $\|A_f - A_{f'}\|_{op} = \OM(\|f - f'\|_\infty + \|\nabla f - \nabla f'\|_{2, \infty})$, and $\|B_f - B_{f'}\|_{op} = \OM(\|f - f'\|_\infty + \|\nabla f - \nabla f'\|_{2, \infty})$:
\begin{align}
	\|A_{f}\circ B_{f} - A_{f'}\circ B_{f'}\|_{op} \leq \|A_{f}\circ B_{f} - A_{f}\circ B_{f'}\|_{op} + \|A_{f}\circ B_{f'} - A_{f'}\circ B_{f'}\|_{op}\notag \\
	\leq [\max_f \|B_{f}\|_{op}\cdot L_{A} +  \max_f \|A_{f}\|_{op}\cdot L_{B}]\big[\|f - f'\|_\infty + \|\nabla f - \nabla f'\|_{2, \infty}\big], \label{eqn_appendix_proof_composition_of_linear_operators}
\end{align}
where $L_{A}$ and $L_{B}$ denote the constants of operators $A_f$ and $B_f$ such that 
\begin{align*}
	\|A_f - A_{f'}\|\leq L_{A}\big[\|f - f'\|_\infty + \|\nabla f - \nabla f'\|_{2, \infty}\big]\\
	\|B_f - B_{f'}\|\leq L_{B}\big[\|f - f'\|_\infty + \|\nabla f - \nabla f'\|_{2, \infty}\big].
\end{align*}
We now take $$A_f = D_1\BM\big(\BM^{l-1}(f, \theta), \theta\big) \ \mathrm{and}\ B_f = D_2\BM^{l-1}(f, \theta).$$
$\|A_f\|_{op}$ is bounded from the following lemma.
\begin{lemma} \label{appendix_lemma_property_1}
	$\BM(f, \theta)$ is $1$-Lipschitz continuous with respect to its first variable.
\end{lemma}
\begin{proof}
	We compute that for any measure $\kappa$ and any function $g \in \CM(\XM)$,
	\begin{equation}\label{eqn_Gateaux_derivative}
	D_1\AM(f, \kappa)[g] = \frac{\int_\XM\exp\{-\frac{1}{\gamma}\big(c(x, y) - f(x)\big)\}g(x)\dB\kappa(x)}{\int_\XM\exp\{-\frac{1}{\gamma}\big(c(x, y) - f(x)\big)\}\dB\kappa(x)}.
	\end{equation}
	Note that 
	\begin{equation}\label{eqn_proof_bound_0}
	\|D_1\AM(f, \kappa)[g]\|_\infty \leq \|\frac{\int_\XM\exp\{-\frac{1}{\gamma}\big(c(x, y) - f(x)\big)\}\dB\kappa(x)}{\int_\XM\exp\{-\frac{1}{\gamma}\big(c(x, y) - f(x)\big)\}\dB\kappa(x)}\|_\infty\cdot\|g\|_\infty = \|g\|_\infty,
	\end{equation}
	and consequently we have $\|D_1\AM(f, \kappa)\|_{op} \leq 1$.
	Further, since $\BM$ is the composition of $\AM$ in its first variable, we have that $\|D_1\BM(f, \theta)\|_{op} \leq 1$.
\end{proof}
$\|B_f\|_{op}$ is bounded from the following lemma.
\begin{lemma}\label{appendix_lemma_property_4}
	Assume that $f\in\CM(\XM)$ satisfies $\|f\|_\infty\leq M_c$ and $\|\nabla f\|_{2, \infty}\leq G_f$.
	Under Assumptions \ref{ass_bounded_infty_c_gradient} and \ref{ass_Lipschitz_continuity_T}, $\forall l\geq 1, \|D_2\BM^{l}(f, \theta)\|_{op}$ is $M_l$-bounded, with $M_l = l\cdot\exp(3M_c/\gamma)\cdot G_T\cdot(G_c + G_f)$. 
\end{lemma}
\begin{proof}
	In this proof, we denote $\tilde{\AM}(f, \theta) \defi \AM(f, \alpha_\theta)$ to make the dependence of $\AM$ on $\theta$ explicit.
	Using the chain rule of Fr\'echet derivative, we compute
	\begin{equation}
	D_2\BM^{l}(f, \theta) = D_1\BM\big(\BM^{l-1}(f, \theta), \theta\big)\circ D_2\BM^{l-1}(f, \theta) + D_2\BM\big(\BM^{l-1}(f, \theta), \theta\big).
	\end{equation}
	We will use $M_l$ to denote the upper bound of $\|D_2\BM^{l}(f, \theta)\|_{op}$. Consequently we have
	\begin{align*}
	M_l \leq&\ \|D_1\BM\big(\BM^{l-1}(f, \theta), \theta\big)\|_{op} \|D_2\BM^{l-1}(f, \theta)\|_{op} + \|D_2\BM\big(\BM^{l-1}(f, \theta), \theta\big)\|_{op} \\
	\leq&\ M_{l-1}+ \|D_2\BM\big(\BM^{l-1}(f, \theta), \theta\big)\|_{op},
	\end{align*}
	where we use Lemma \ref{appendix_lemma_property_1} in the second inequality.
	Recall that $\BM(f, \theta) = \AM(\tilde{\AM}(f, \theta), \beta)$.
	Again using the chain rule of the Fr\'echet derivative, we compute
	\begin{equation}
	D_2\BM(f, \theta) = D_1\AM\big(\tilde{\AM}(f, \theta), \beta\big)\circ D_2\tilde{\AM}(f, \theta),
	\end{equation}
	and hence
	\begin{equation}
	\|D_2\BM(f, \theta)\|_{op} \leq \|D_1\AM\big(\tilde{\AM}(f, \theta), \beta\big)\|_{op}\cdot\|D_2\tilde{\AM}(f, \theta)\|_{op} \leq \|D_2\tilde{\AM}(f, \theta)\|_{op},
	\end{equation}
	where we use \eqref{eqn_proof_bound_0} in the second inequality.
	We now bound $\|D_2\tilde{\AM}(f, \theta)\|_{op}$. Denote
	$$\omega_y(x) \defi \exp(-c(x, y)/\gamma) \exp(f(x)/\gamma).$$
	We have $\exp(-2M_c/\gamma)\leq \omega_y(x)\leq \exp(M_c/\gamma)$ from $\|f\|_\infty\leq M_c$ and Assumption \ref{ass_bounded_c}.
	For any direction $h\in\RBB^q$ (note that $D_2\tilde{\AM}(f, \theta)[h]:\XM\rightarrow\RBB$) and any $y\in\XM$, we compute
	\begin{equation*}
	\big(D_2\tilde{\AM}(f, \theta)[h]\big)(y) = \frac{\int_\XM \omega_y(T_\theta(x)) \langle [\nabla_\theta T_\theta(x)]^\top\left[-\nabla_1 c(T_\theta(x), y) + \nabla f(T_\theta(x))\right], h\rangle\dB\mu(x)}{\int_\XM \omega_y(T_\theta(x)) \dB\mu(x)},
	\end{equation*}
	where $\nabla_\theta T_\theta(x)$ denotes the Jacobian matrix of $T_\theta(x)$ w.r.t. $\theta$.
	Consequently we bound
	\begin{align*}
	\|D_2\tilde{\AM}(f, \theta)[h]\|_\infty \leq&\ \exp(3M_c/\gamma)\|\nabla_\theta T_\theta(x)\|_{op}\cdot [\|\nabla_1 c\big(T_\theta(x), y\big)\| + \|\nabla f\big(T_\theta(x)\big)\|]\cdot\|h\| \\
	\leq&\ \exp(3M_c/\gamma)\cdot G_T\cdot(G_c + G_f) \|h\|,
	\end{align*}
	which implies
	\begin{equation} \label{eqn_appendix_D2A_bounded}
	\|D_2\tilde{\AM}(f, \theta)\|_{op} \leq \exp(3M_c/\gamma)\cdot G_T\cdot(G_c + G_f).
	\end{equation}
\end{proof}
To show the Lipschitz continuity of $A_f$, i.e. $\|A_f - A_{f'}\|\leq L_{A}\|f - f'\|_\infty$, we first establish the following continuity lemmas of $D_1\BM(\cdot, \theta)$ and $\BM^{l-1}(\cdot, \theta)$.
\begin{lemma} \label{lemma_appendix_lipschitz_continuity_D_1_B}
	For $f\in\CM(\XM)$ such $\|f\|_\infty \leq M_c$, 
	$D_1\BM(f, \theta)$ is $L$-Lipschitz continuous with respect to its first variable with $L = 2L_\AM$.
\end{lemma}
\begin{proof}
	Use the chain rule of Fr\'echet derivative to compute
	\begin{equation}
	D_1\BM(f, \theta) = \underbrace{D_1\AM\big(\AM(f, \alpha_\theta), \beta\big)}_{U_f}\circ \underbrace{D_1\AM(f, \alpha_\theta)}_{V_f}.
	\end{equation}
	We analyze the Lipschitz continuity of $\|D_1\BM(f, \theta)\|_{op}$ following the same logic as \eqref{eqn_appendix_proof_composition_of_linear_operators}:
	\begin{itemize}
		\item The $1$-boundedness of $U_f$ and $V_f$ is from Lemma \ref{appendix_lemma_property_1}.
		\item The $L_\AM$-Lipschitz continuity of $V_f$ is from Lemma \ref{lemma_appendix_bound_2_lemma}.
		\item The $L_\AM$-Lipschitz continuity of $U_f$ is from Lemmas \ref{appendix_lemma_property_1} and \ref{lemma_appendix_bound_2_lemma}.
	\end{itemize}
	Consequently, we have that $D_1\BM(f, \theta)$ is $2L_\AM$-Lipschitz continuous w.r.t. its first variable.
\end{proof}
\begin{lemma} \label{appendix_lemma_property_3}
	$\forall l, \BM^{l}(f, \theta)$ is $1$-Lipschitz continuous with respect to its first variable.
\end{lemma}
\begin{proof}
	Use the chain rule of Fr\'echet derivative to compute
	\begin{equation}
	D_1\BM^{l}(f, \theta) = D_1\BM\big(\BM^{l-1}(f, \theta), \theta\big)\circ D_1\BM^{l-1}(f, \theta).
	\end{equation}
	Consequently $\|D_1\BM^{l}(f, \theta)\|_{op} \leq \|D_1\BM(f, \theta)\|_{op}^l$.
	Further, we have $\|D_1\BM(f, \theta)\|_{op} \leq 1$ from Lemma \ref{appendix_lemma_property_1} which leads to the result.
\end{proof}
We have that $A_f$ is Lipschitz continuous since (i) $A_f$ is the composition of Lipschitz continuous operators $D_1\BM(\cdot, \theta)$ and $\BM^{l-1}(f\cdot, \theta)$ and (ii) for $\|f\|_\infty\leq M_c$, $\forall l\geq 0, \|\BM^l(f, \theta)\|_\infty\leq M_c$ (the argument is similar to Lemma \ref{lemma_appendix_bounded_function_sequence_from_sinkhorn_knopp}).

We prove $\|B_f - B_{f'}\|\leq L_{l}\big[\|f - f'\|_\infty + \|\nabla f - \nabla f'\|_{2, \infty}\big]$ via induction.
The following lemma establishes the base case for $D_2\BM(f, \theta)$ (when $l=2$).
Note that the boundedness of $\|f\|_\infty$ ($\|f'\|_\infty$) and $\|\nabla f\|_\infty$ ($\|\nabla f'\|_\infty$) remains valid after the operator $\BM$ (Lemma \ref{lemma_appendix_bounded_function_sequence_from_sinkhorn_knopp} and (i) of Lemma \eqref{lemma_appendix_sinkhorn_potential_boundedness}).
\begin{lemma} \label{appendix_lemma_property_5}
	There exists constant $L_{1}$ such that for $\|f\|_\infty \leq M_c$ ($\|f'\|_\infty \leq M_c$) and $\|\nabla f\|_\infty\leq G_f$ ($\|\nabla f'\|_\infty\leq G_f$)
	\begin{equation}
		\|D_2\BM(f, \theta) - D_2\BM(f', \theta)\|_{op} \leq L_{1}\big[\|f - f'\|_\infty + \|\nabla f - \nabla f'\|_{2, \infty}\big].
	\end{equation}
\end{lemma}
\begin{proof}
	In this proof, we denote $\tilde{\AM}(f, \theta) \defi \AM(f, \alpha_\theta)$ to make the dependence of $\AM$ on $\theta$ explicit.
	Recall that $\BM(f, \theta) = \AM(\tilde{\AM}(f, \theta), \beta)$.
	Use the chain rule of Fr\'echet derivative to compute
	\begin{equation}
	D_2\BM(f, \theta) = \underbrace{D_1\AM\big(\AM(f, \alpha_\theta), \beta\big)}_{U_f}\circ \underbrace{D_2\tilde \AM(f, \theta)}_{V_f}.
	\end{equation}
	We analyze the Lipschitz continuity of $\|D_2\BM(f, \theta)\|_{op}$ following the same logic as \eqref{eqn_appendix_proof_composition_of_linear_operators}:
	\begin{itemize}
		\item The $1$-boundedness of $U_f$ is from Lemma \ref{appendix_lemma_property_1}.
		\item The $\exp(3M_c/\gamma)\cdot G_T\cdot(G_c + G_f)$-boundedness of $V_f$ is from \eqref{eqn_appendix_D2A_bounded}.
		\item The $L_A$-Lipschitz continuity of $U_f$ is from Lemmas \ref{appendix_lemma_property_1} and \ref{lemma_appendix_bound_2_lemma} and the fact that for $\|f\|_\infty\leq M_c$, $\|\AM(f, \theta)\|_\infty\leq M_c$ (the argument is similar to Lemma \ref{lemma_appendix_bounded_function_sequence_from_sinkhorn_knopp}).
		\item Denote
		$$\TM_y(x, f) \defi \exp(-c(x, y)/\gamma) \exp(f(x)/\gamma).$$
		We compute
		\begin{align*}
		V_f = \frac{\int_\ZM \TM_y\big(T_\theta(z), f\big)  [\nabla_\theta T_\theta(z)]^\top\left[-\nabla_1 c(T_\theta(z), y) + \nabla f\big(T_\theta(z)\big)\right]\dB\mu(z)}{\int_\ZM \TM_y\big(T_\theta(z), f\big) \dB\mu(z)}, && \#\frac{P_f}{Q_f}
		\end{align*}
		Denote the numerator by $P_f$ and the denominator by $Q_f$.
		Following the similar idea as \eqref{eqn_Lipschitz_continuous_fraction_of_functions},
		we show that both $\|P_f\|_{op}$ and $\|Q_f\|_{\infty}$ are bounded, $Q_f$ is Lipschitz continuous w.r.t. $f$, $Q_f$ is positive and bounded from below, and $\|P_f - P_{f'}\|_{op} \leq L_v [\|f - f'\|_\infty + \|\nabla f - \nabla f'\|_{2,\infty}]$ for some constant $L_v$.
		
		\begin{itemize}
			\item The boundedness of $\|P_f\|_{op}$ is from the boundedness of $f$, Assumptions \ref{ass_Lipschitz_continuity_T}, \ref{ass_bounded_infty_c_gradient}, and the boundedness of $\nabla f$.
			\item The boundedness of $\|Q_f\|_{\infty}$ is from the boundedness of $f$.
			\item Use $DQ_f$ to denote the Fr\'echet derivative of $Q_f$ w.r.t. $f$. For any function $g\in\CM(\XM)$,
			\begin{equation}
				DQ_f[g] = \int_\XM \TM_y(x, f) g(x)/\gamma \dB \alpha_\theta(x),
			\end{equation}
			where we recall that $\alpha_\theta = {T_\theta}_\sharp \mu$. Further, we have $\|DQ_f[g]\|_\infty \leq \exp(M_c/\gamma)/\gamma \|g\|_\infty$, which implies the Lipschitz continuity of $Q_f$ (for $\|f\|_\infty \leq M_c$).
			\item We prove that for $\|f\|_\infty \leq M_c$ ($\|f'\|_\infty \leq M_c$) and $\|\nabla f\|_\infty\leq G_f$ ($\|\nabla f'\|_\infty\leq G_f$), $$\|P_f - P_{f'}\|_{op} \leq L_v [\|f - f'\|_\infty + \|\nabla f - \nabla f'\|_{2,\infty}].$$
			For a fixed $z\in\ZM$,
			denote $$p^z_f \defi\TM_y\big(T_\theta(z), f\big)  [\nabla_\theta T_\theta(z)]^\top\left[-\nabla_1 c(T_\theta(z), y) + \nabla f\big(T_\theta(z)\big)\right].$$
			Note that $P_f = \int_\ZM p^z_f\dB \mu(z)$.
			For any direction $h\in\RBB^d$, we bound
			\begin{align*}
			&\|p^z_f[h] - p^z_{f'}[h]\|_{op} \\
			\leq& \|D_2\TM_y\big(T_\theta(z), f\big)\|_{op}\|f - f'\|_\infty\cdot  \max_y|[\nabla_\theta T_\theta(z) h]^\top\left[-\nabla_1 c(T_\theta(z), y) + \nabla f\big(T_\theta(z)\big)\right]| \\
			&+ [\max_y \TM_y\big(T_\theta(z), f\big)]\cdot \|\nabla_\theta T_\theta(z) h\| \|\nabla f\big(T_\theta(z)\big) - \nabla f'\big(T_\theta(z)\big)\|\\
			\leq& \exp(M_c/\gamma)/\gamma\cdot G_T\cdot(G_c + G_f)\cdot\|f - f'\|_\infty\cdot\|h\| + \exp(M_c/\gamma)\cdot G_T\cdot\|h\|\cdot\|\nabla f - \nabla f'\|_{2,\infty}.
			\end{align*}
			Consequently, we have that there exists a constant $L_v$ such that
			\begin{equation*}
			\|p^z_f[h] - p^z_{f'}[h]\|_{\infty} \leq L_v [\|f - f'\|_\infty + \|\nabla f - \nabla f'\|_{2,\infty}]\cdot\|h\|.
			\end{equation*}
		\end{itemize}
	\end{itemize}
\end{proof}
The above lemma shows the base case for the induction.
Now suppose that the inequality $\|D_2\BM^{k}(f, \theta) - D_2\BM^{k}(f', \theta)\|_{op} \leq L_{k}\big[\|f - f'\|_\infty + \|\nabla f - \nabla f'\|_{2, \infty}\big]$ holds.\\
For the case of $k+1$, we compute the Fr\'echet derivative
\begin{equation*}
	D_2 \BM^{k+1}(f, \theta) = D_1\BM\big(\BM^{k}(f, \theta), \theta\big)\circ D_2\BM^{k}(f, \theta) + D_2\BM\big(\BM^{k}(f, \theta), \theta\big),
\end{equation*}
and hence we can bound
\begin{align}
	&\ \|D_2 \BM^{k+1}(f, \theta) - D_2 \BM^{k+1}(f', \theta)\|_{op} \notag \\
	\leq&\ \|D_1\BM\big(\BM^{k}(f, \theta), \theta\big)\circ \big(D_2\BM^{k}(f, \theta) - D_2\BM^{k}(f', \theta)\big)\|_{op} \notag \\
	&\ + \|\bigg(D_1\BM\big(\BM^{k}(f, \theta), \theta\big) - D_1\BM\big(\BM^{k}(f', \theta), \theta\big)\bigg) \circ D_2\BM^{k}(f', \theta)\|_{op}\notag\\
	&\ + \|D_2\BM\big(\BM^{k}(f, \theta), \theta\big) - D_2\BM\big(\BM^{k}(f', \theta), \theta\big)\|_{op}\notag\\
	\leq&\  \|D_2\BM^{k}(f, \theta) - D_2\BM^{k}(f', \theta)\|_{op} \\
	&\ + L_\AM\|\BM^{k}(f, \theta) - \BM^{k}(f', \theta)\|_\infty \|D_2\BM^{k}(f', \theta)\|_{op} \notag \\
	&\ + L_1\big[\|\BM^{k}(f, \theta) - \BM^{k}(f', \theta)\|_\infty + \|\nabla \BM^{k}(f, \theta) - \nabla \BM^{k}(f', \theta)\|_{2, \infty}\big] \notag \\
	\leq&\ L_k [\|f - f'\|_\infty + \|\nabla f - \nabla f'\|_{2, \infty}] + L_\AM\cdot M_k \cdot\|f - f'\|_\infty \notag  \\
	&\ + L_1\|f - f'\|_\infty + L_1\|\nabla \BM^{k}(f, \theta) - \nabla \BM^{k}(f', \theta)\|_{2, \infty}\notag \\
	\leq&\ (L_k+ L_1+L_\AM M_k) [\|f - f'\|_\infty + \|\nabla f - \nabla f'\|_{2, \infty}]+ L_1\|\nabla \BM^{k}(f, \theta) - \nabla \BM^{k}(f', \theta)\|_{2, \infty}. \label{eqn_appendix_D_2_BM_continuity}
\end{align}
Here in the third inequality, we use the induction for the first term, Lemma \ref{appendix_lemma_property_4} for the second term.
Notice that $\nabla\AM(f, \theta)$ is Lipschitz continuous w.r.t. $f$: Denote $k(x, y) \defi \exp\{-c(x, y)/\gamma\}$. For any fixed $x\in\XM$,
\begin{align*}
	\nabla \big(\AM(f, \alpha)\big)(x) = \frac{\int_\XM k(z, x)\exp\{f(z)/\gamma\} \nabla_{1}c(x, z)\dB \alpha(z)}{\int_\XM k(z, x)\exp\{f(z)/\gamma\}\dB \alpha(z)}, && \#\ \frac{g_1(f)}{g_2(f)}
\end{align*}
where we denote the numerator and denominator of the above expression by $g_1:\CM(\XM)\rightarrow \RBB^q$ and $g_2:\CM(\XM)\rightarrow\RBB$.
From the boundedness of $g_1$ and $g_2$, the Lipschitz continuity of $g_1$ and $g_2$ w.r.t. to $f$, and the fact that $g_2$ is positive and bounded away from zero, we conclude that there exists some constant $L_{\AM, f}$ such that for any $x\in\XM$ (this follows similarly as \eqref{eqn_Lipschitz_continuous_fraction_of_functions})
\begin{equation} \label{eqn_appendix_gradient_bounded_by_input_function}
	\|\nabla \big(\AM(f, \alpha)\big)(x) - \nabla \big(\AM(f', \alpha)\big)(x)\| \leq L_{\AM, f} \|f - f'\|_\infty.
\end{equation}
Recall that $\BM^k$ is the compositions of operators in the form of $\AM$.
Consequently, we have that 
\begin{equation*}
	\|\nabla \BM^{k}(f, \theta) - \nabla \BM^{k}(f', \theta)\|_{2, \infty} \leq L_{\AM, f} \|f - f'\|_\infty.
\end{equation*}
Plugging this result into \eqref{eqn_appendix_D_2_BM_continuity}, we prove that the induction holds for $k+1$:
\begin{equation*}
	\|D_2 \BM^{k+1}(f, \theta) - D_2 \BM^{k+1}(f', \theta)\|_{op} \leq (L_k+ L_1+L_\AM M_k+ L_1 L_{\AM, f}) [\|f - f'\|_\infty + \|\nabla f - \nabla f'\|_{2, \infty}].
\end{equation*}
Consequently, for any finite $l$, we have $\|B_f - B_{f'}\|\leq L_{l}\big[\|f - f'\|_\infty + \|\nabla f - \nabla f'\|_{2, \infty}\big]$, where $L_{l} = l\cdot(L_1+L_\AM M_k+ L_1 L_{\AM, f})$.
\begin{lemma} \label{lemma_appendix_bound_2_lemma}
	Under Assumption \ref{ass_bounded_c}, for $f\in\CM(\XM)$ such $\|f\|_\infty \leq M_c$, there exists constant $L_\AM$ such that
	$D_1\AM(f, \alpha_\theta)$ is $L_\AM$-Lipschitz continuous with respect to its first variable.
\end{lemma}
\begin{proof}
	Let $g\in\CM(\XM)$ any function.
	Denote $\TM_y(x, f) \defi \exp(-c(x, y)/\gamma) \exp(f(x)/\gamma)$.
	For a fixed point $y\in\XM$ and any function $g\in\CM(\XM)$, we compute that
	\begin{align*}
		\big(D_1\AM(f, \theta)[g]\big)(y) = \frac{\int_{\XM}\TM_y(x, f) g(x) \dB \alpha_\theta(x)}{\int_{\XM}\TM_y(x, f) \dB \alpha_\theta(x)}, && \#\ \frac{g_1(f)}{g_2(f)}
	\end{align*}
	where we denote the numerator and denominator of the above expression by $g_1:\CM(\XM)\rightarrow \RBB^q$ and $g_2:\CM(\XM)\rightarrow\RBB$.
	From the boundedness of $g_1$ and $g_2$, the Lipschitz continuity of $g_1$ and $g_2$ w.r.t. to $f$, and the fact that $g_2$ is positive and bounded away from zero, we conclude that there exists some constant $L_{\AM}$ such that for any $x\in\XM$ (this follows similarly as \eqref{eqn_Lipschitz_continuous_fraction_of_functions}).

\end{proof}

\paragraph{Analyze the second term of \eqref{eqn_appendix_proof_D2E_expression}.}
We bound the second term of \eqref{eqn_appendix_proof_D2E_expression} using Lemma \ref{appendix_lemma_property_5}:
\begin{align*}
	&\ \|D_2\BM(\BM^{l-1}(f, \theta), \theta) - D_2\BM(\BM^{l-1}(f', \theta), \theta)\|_{op} \\
	\leq&\ L_1 [\|\BM^{l-1}(f, \theta) - \BM^{l-1}(f', \theta)\|_\infty + \|\nabla \BM^{l-1}(f, \theta) - \nabla \BM^{l-1}(f', \theta)\|_{2, \infty}] \\
	\leq&\ L_1 [\|f - f'\|_\infty + L_{\AM, f}\|f - f'\|_{\infty}] = L_1\cdot(1+L_{\AM, f})\|f - f'\|_\infty,
\end{align*}
where we use \eqref{eqn_appendix_gradient_bounded_by_input_function} in the second inequality.

Combing the analysis for the two terms of \eqref{eqn_appendix_proof_D2E_expression}, we conclude the result. 

\end{proof}

\subsection{Proof of Theorem \ref{thm_computability_of_SIM}} \label{appendix_proof_thm_computability_of_SIM}
We prove that the approximation error of $\nabla^2_\theta\OT_\gamma(\alpha_\theta, \beta)$ using the estimated Sinkhorn potential $f_\theta^\epsilon$ and the estimated Fr\'echet derivative $g_\theta^\epsilon$ is of the order $$\OM(\|f_\theta^\epsilon - f_\theta\|_\infty + \|\nabla f_\theta^\epsilon - \nabla f_\theta\|_{2, \infty} + \|\nabla^2 f_\theta^\epsilon - \nabla^2 f_\theta\|_{op, \infty}+\|g_\theta^\epsilon - Df_\theta\|_{op}).$$
The other term $\nabla^2_\theta\OT_\gamma(\alpha_\theta, \alpha_\theta)$ is handled in a similar manner.

Recall the simplified expression of $\nabla^2_\theta\OT_\gamma(\alpha_\theta, \beta)$ in \eqref{eqn_appendix_two_term_OT_gamma}. Given the estimator $f_\theta^\epsilon$ ($g_\theta^\epsilon$) of $f_\theta$ ($Df_\theta$), we need to prove the following bounds of differences in terms of the estimation accuracy:
For any $h_1, h_2 \in \RBB^d$,
\begin{align} 
	&|D^2_{11}\HM_1(f_\theta, \theta)\big[Df_\theta[h_1], Df_\theta[h_2]\big] - D^2_{11}\HM_1(f_\theta^\epsilon, \theta)\big[g_\theta^\epsilon[h_1], g_\theta^\epsilon[h_2]\big]| \notag\\
	&\qquad =\OM\left(\|h_1\|\cdot\|h_2\|\cdot(\|f_\theta^\epsilon - f_\theta\|_\infty + \|g_\theta^\epsilon - Df_\theta\|_{op})\right), \label{eqn_appendix_proof_theorem_approximation_H_1}\\
	&\|D^2_{22}\HM_1(f_\theta, \theta) - D^2_{22}\HM_1(f_\theta^\epsilon, \theta)\|_{op} \notag \\
	&\qquad =\OM\left(\|f_\theta^\epsilon - f_\theta\|_\infty + \|\nabla f_\theta^\epsilon - \nabla f_\theta\|_{2, \infty} + \|\nabla^2 f_\theta^\epsilon - \nabla^2 f_\theta\|_{op, \infty}\right). \label{eqn_appendix_proof_theorem_approximation_H_2}
\end{align}
Note that from the definition of the operator norm the first results is equivalent to the bound in the operator norm.
Using Propositions \ref{proposition_sinkhorn_potential} and \ref{proposition_frechet_derivative} and Lemmas \ref{lemma_convergence_of_gradient}, \ref{lemma_convergence_of_Hessian}, we know that we can compute the estimators $f_\theta^\epsilon$ and $g_\theta^\epsilon$ such that $\|f_\theta^\epsilon - f_\theta\|_\infty\leq \epsilon$, $\|\nabla f_\theta^\epsilon - \nabla f_\theta\|_{2, \infty}\leq \epsilon$, and $\|\nabla^2 f_\theta^\epsilon - \nabla^2 f_\theta\|_{op, \infty}\leq \epsilon$, and $\|g_\theta^\epsilon - Df_\theta\|_{op}\leq \epsilon$ in logarithm time $\OM(\log \frac{1}{\epsilon})$.
Together with \eqref{eqn_appendix_proof_theorem_approximation_H_1} and \eqref{eqn_appendix_proof_theorem_approximation_H_2} proved above, we can compute an $\epsilon$-accurate estimation of $\nabla^2_\theta\OT_\gamma(\alpha_\theta, \beta)$ (in the operator norm) in logarithm time $\OM(\log \frac{1}{\epsilon})$.

\paragraph{Bounding \eqref{eqn_appendix_proof_theorem_approximation_H_1}.}
Recall the definition of $D^2_{11}\HM_1(f_\theta, \theta)\big[Df_\theta[h_1], Df_\theta[h_2]\big]$ in \eqref{eqn_appendix_D_11_H}.
Denote
\begin{align*}
	A_1 &\ = D^2_{11}\AM(f_\theta, \alpha_\theta), v_1 = Df_\theta[h_1], v_2 = Df_\theta[h_2],	\\
	A_2 &\ = D^2_{11}\AM(f_\theta^\epsilon, \alpha_\theta), u_1 = g_\theta^\epsilon[h_1], u_2 = g_\theta^\epsilon[h_2].
\end{align*}
Based on these definitions, we have 
\begin{align*}
&D^2_{11}\HM_1(f_\theta, \theta)\big[Df_\theta[h_1], Df_\theta[h_2]\big] = \int_\XM A_1[v_1, v_2](y) \dB \beta(y) \\
&D^2_{11}\HM_1(f_\theta^\epsilon, \theta)\big[g_\theta^\epsilon[h_1], g_\theta^\epsilon[h_2]\big] = \int_\XM A_2[u_1, u_2](y) \dB \beta(y).
\end{align*}
Using the triangle inequality, we have
\begin{align} \label{eqn_thm_computability_proof_1_1}
	\|A_1[v_1, v_2] - &A_2[u_1, u_2]\|_\infty \\
	&\leq \|A_1[v_1 - u_1, v_2]\|_\infty + \|A_1[u_1, v_2 - u_2]\|_\infty + \|(A_1 - A_2)[u_1, u_2]\|_\infty. \notag
\end{align}
We bound the three terms on the R.H.S. individually.

For the first term on the R.H.S. of \eqref{eqn_thm_computability_proof_1_1}, we recall the explicit expression of $A_1[v_1, v_2](y)$ in \eqref{eqn_appendix_D11A} as
\begin{align*}
A_1[v_1, v_2](y)  = \frac{\int_{\XM} \TM_y(x, f_\theta) v_1(x) v_2(x)\dB\alpha_\theta(x)}{\gamma\int_\XM \TM_y(x, f_\theta)\dB\alpha_\theta(x)} - \frac{\int_{\XM^2} \TM_y(x, f_\theta)\TM_y(x', f_\theta) v_1(x)v_2(x')\dB\alpha_\theta(x)\dB\alpha_\theta(x')}{\gamma\left[\int_\XM \TM_y(x, f_\theta)\dB\alpha_\theta(x)\right]^2}. 
\end{align*}
Here we recall $\TM_y(x, f) \defi \exp(-c(x, y)/\gamma) \exp(f(x)/\gamma)$.
We bound using the facts that $\TM_y(x, f_\theta)$ is bounded from above and bounded away from zero
\begin{align*}
	|A_1[v_1 - u_1, v_2](y)| \leq&\ |\frac{\int_{\XM} \TM_y(x, f_\theta) \big(v_1(x) - u_1(x)\big) v_2(x)\dB\alpha_\theta(x)}{\gamma\int_\XM \TM_y(x, f_\theta)\dB\alpha_\theta(x)}| \\
	&\quad + |\frac{\int_{\XM^2} \TM_y(x, f_\theta)\TM_y(x', f_\theta) \big(v_1(x) - u_1(x)\big)v_2(x')\dB\alpha_\theta(x)\dB\alpha_\theta(x')}{\gamma\left[\int_\XM \TM_y(x, f_\theta)\dB\alpha_\theta(x)\right]^2}| \\
	=&\ \OM(\|v_1 - u_1\|_\infty\cdot \|v_2\|_\infty).
\end{align*}
Further, we have $\|u_1 - v_1\|_\infty = \OM(\|Df_\theta - g_\theta^\epsilon\|_{op}\cdot\|h_1\|)$ and $\|v_1\|_\infty = \OM(\|h_2\|)$. Consequently, the first term on the R.H.S. of \eqref{eqn_thm_computability_proof_1_1} is of order $\OM(\|Df_\theta - g_\theta^\epsilon\|_{op}\cdot\|h_1\|\cdot\|h_2\|)$.

Following the same argument, we have the second term on the R.H.S. of \eqref{eqn_thm_computability_proof_1_1} is of order $\OM(\|Df_\theta - g_\theta^\epsilon\|_{op}\cdot\|h_1\|\cdot\|h_2\|)$.

To bound the third term on the R.H.S. of \eqref{eqn_thm_computability_proof_1_1}, denote
\begin{equation*}
	A_{11}[u_1, u_2] \defi \frac{\int_{\XM} \TM_y(x, f_\theta) u_1(x) u_2(x)\dB\alpha_\theta(x)}{\gamma\int_\XM \TM_y(x, f_\theta)\dB\alpha_\theta(x)}\ \mathrm{and} \
	A_{21}[u_1, u_2] \defi \frac{\int_{\XM} \TM_y(x, f_\theta^\epsilon) u_1(x) u_2(x)\dB\alpha_\theta(x)}{\gamma\int_\XM \TM_y(x, f_\theta^\epsilon)\dB\alpha_\theta(x)},
\end{equation*}
and denote
\begin{align*}
	A_{12}[u_1, u_2] \defi \frac{\int_{\XM} \TM_y(x, f_\theta)u_1(x)\dB\alpha_\theta(x)\int_{\XM}\TM_y(x', f_\theta) u_2(x')\dB\alpha_\theta(x')}{\gamma\left[\int_\XM \TM_y(x, f_\theta)\dB\alpha_\theta(x)\right]^2}, \\
	\mathrm{and} \ A_{22}[u_1, u_2] \defi \frac{\int_{\XM} \TM_y(x, f_\theta^\epsilon) u_1(x)\dB\alpha_\theta(x)\int_{\XM}\TM_y(x', f_\theta^\epsilon)u_2(x')\dB\alpha_\theta(x')}{\gamma\left[\int_\XM \TM_y(x, f_\theta^\epsilon)\dB\alpha_\theta(x)\right]^2}.
\end{align*}
We show that both $|\big(A_{11} - A_{21}\big)[u_1, u_2]|$ and $|\big(A_{12} - A_{22}\big)[u_1, u_2]|$ are of order $\OM(\|Df_\theta - g_\theta^\epsilon\|_{op}\cdot\|h_1\|\cdot\|h_2\|)$. This then implies $|\big(A_{1} - A_{2} \big)[u_1, u_2]| = \OM(\|Df_\theta - g_\theta^\epsilon\|_{op}\cdot\|h_1\|\cdot\|h_2\|)$.\\
With the argument similar to \eqref{eqn_Lipschitz_continuous_fraction_of_functions}, we obtain that $|\big(A_{11} - A_{21}\big)[u_1, u_2]| = \OM(\|Df_\theta - g_\theta^\epsilon\|_{op}\cdot\|u_1\|\cdot\|u_2\|)$ using the boundedness and Lipschitz continuity of the numerator and denominator of $A_{11}[u_1, u_2]$ w.r.t. to $f_\theta$ and the fact that the denominator is positive and bounded away from zero (see the discussion following \eqref{eqn_Lipschitz_continuous_fraction_of_functions}).
Further, since both $Df_\theta$ and $g_\theta^\epsilon$ are bounded linear operators, we have that $u_1 = \OM(h_1)$ and $u_2 = \OM(h_2)$.
Consequently, we prove that $|\big(A_{11} - A_{21}\big)[u_1, u_2]| = \OM(\|f_\theta - f_\theta^\epsilon\|_{op}\cdot\|h_1\|\cdot\|h_2\|)$.\\
Similarly, we can prove that $|\big(A_{12} - A_{22}\big)[u_1, u_2]| = \OM(\|f_\theta - f_\theta^\epsilon\|_{op}\cdot\|h_1\|\cdot\|h_2\|)$.

Altogether, we have proved \eqref{eqn_appendix_proof_theorem_approximation_H_1}.
\paragraph{Bounding \eqref{eqn_appendix_proof_theorem_approximation_H_2}.}
Recall that the expression of $D^2_{22}\HM_1(f, \theta)$ in \eqref{eqn_appendix_D_22_H}.
For a fixed $y\in\XM$ and a fixed $z'\in\ZM$, denote (recall that $u_z(\theta, f) = \nabla_{1} c\big(T_\theta(z), y\big) - \nabla f\big(T_\theta(z)\big)$)
\begin{align*}
	B_1(f) =&\ \nabla^2_\theta T_\theta(z')\times_1\nabla f\big(T_\theta(z')\big) \\
	B_2(f) =&\ \nabla_\theta T_\theta(z')^\top\nabla^2 f\big(T_\theta(z')\big)\nabla_\theta T_\theta(z') \\
	B_3(f) =&\ \frac{\int_\ZM \TM_y\big(T_\theta(z), f\big)\nabla_\theta T_\theta(z)^\top u_z(\theta, f)u_z(\theta, f)^\top \nabla_\theta T_\theta(z)\dB \mu(z)}{\int_\ZM \TM_y\big(T_\theta(z), f\big)\dB \mu(z)}\\
	B_4(f) =&\ \frac{\int_\ZM \TM_y\big(T_\theta(z), f\big)\nabla^2_\theta T_\theta(z)\times_1 u_z(\theta, f)\dB \mu(z)}{\int_\ZM \TM_y\big(T_\theta(z), f\big)\dB \mu(z)}\\
	B_5(f) =&\ \frac{\int_\ZM \TM_y\big(T_\theta(z), f\big)\nabla_\theta T_\theta(z)^\top \nabla_{11} c(T_\theta(z), y)\nabla_\theta T_\theta(z)\dB \mu(z)}{\int_\ZM \TM_y\big(T_\theta(z), f\big)\dB \mu(z)}\\
	B_6(f) =&\ -\frac{\int_\ZM \TM_y\big(T_\theta(z), f\big)\nabla_\theta T_\theta(z)^\top \nabla^2 f\big(T_\theta(z)\big)\nabla_\theta T_\theta(z)\dB \mu(z)}{\int_\ZM \TM_y\big(T_\theta(z), f\big)\dB \mu(z)}\\
	B_7(f) =&\ \frac{\int_\ZM \TM_y\big(T_\theta(z), f\big)\nabla_\theta T_\theta(z)^\top u_z(\theta, f)\dB \mu(z)\left[\int_\ZM \TM_y\big(T_\theta(z), f\big)\nabla_\theta T_\theta(z)^\top u_z(\theta, f)\dB \mu(z)\right]^\top}{\left[\int_\ZM \TM_y\big(T_\theta(z), f\big)\dB \mu(z)\right]^2}
\end{align*}
Based on these definitions, we have
\begin{equation*}
	D^2_{22}\HM_1(f, \theta) = \int_{\ZM}\sum_{i=1}^{2} B_i(f) \dB \mu(z') + \int_{\XM} \sum_{i=3}^{7} B_i(f) \dB \beta(y).
\end{equation*}
We bound the above seven terms individually.
\begin{assumption}\label{ass_boundedness_of_second_jacobian_T}
	For a fixed $z\in\ZM$ and $\theta\in\Theta$, use $\nabla^2_\theta T_\theta(z)\in T(\RBB^d\times\RBB^d\rightarrow\RBB^q)$\footnote{Recall that $T(U, W)$ is the family of bounded linear operators from $U$ to $W$.} to denote the second-order Jacobian of $T_\theta(z)$ w.r.t. $\theta$. Use $\times_1$ to denote the tensor product along the first dimension.
	For any two vectors $g, g'\in\RBB^d$, we assume that
	\begin{equation}
		\|\nabla^2_\theta T_\theta(z) \times_1 g - \nabla^2_\theta T_\theta(z) \times_1 g'\|_{op} = \OM(\|g - g'\|).
	\end{equation}
\end{assumption}
For the first term, using the boundedness of $\nabla^2_\theta T_\theta(z')$ (Assumption \ref{ass_boundedness_of_second_jacobian_T}), we have that $$\|B_1(f_\theta) - B_1(f_\theta^\epsilon)\|_{op} = \OM(\|\nabla f_\theta - \nabla f_\theta^\epsilon\|_{2, \infty}).$$
For the second term, using the boundedness of $\nabla_\theta T_\theta(z')$, we have that
\begin{equation*}
	\|B_2(f_\theta) - B_2(f_\theta^\epsilon)\|_{op} = \OM(\|\nabla^2 f_\theta - \nabla^2 f_\theta^\epsilon\|_{op, \infty}).
\end{equation*}
For the third term, note that $\|u_z(\theta, f_\theta) - u_z(\theta, f_\theta^\epsilon)\| = \OM(\|\nabla f_\theta - \nabla f_\theta^\epsilon\|_{2,\infty})$.
With the argument similar to \eqref{eqn_Lipschitz_continuous_fraction_of_functions}, we obtain that
\begin{equation}
	\|B_3(f_\theta) - B_3(f_\theta^\epsilon)\|_{op} = \OM(\|f_\theta - f_\theta^\epsilon\|_{\infty}+ \|\nabla f_\theta - \nabla f_\theta^\epsilon\|_{2, \infty}).
\end{equation}
This is from the boundedness and Lipschitz continuity of $\TM_y\big(T_\theta(z), f\big)$ w.r.t. to $f$, the boundedness and Lipschitz continuity of $u_z(\theta, f)$ w.r.t. $\nabla f$, and the fact that $\TM_y\big(T_\theta(z), f\big)$ is positive and bounded away from zero.

For the forth term, following the similar argument as the third term and using the boundedness of $\nabla^2_\theta T_\theta(z)$, we have that
\begin{equation}
	\|B_4(f_\theta) - B_4(f_\theta^\epsilon)\|_{op} = \OM(\|f_\theta - f_\theta^\epsilon\|_{\infty}+ \|\nabla f_\theta - \nabla f_\theta^\epsilon\|_{2, \infty}).	
\end{equation}
For the fifth term, following the similar argument as the third term and using the boundedness of $\nabla_\theta T_\theta(z)$ and $\nabla_{11} c(T_\theta(z), y)$, we have that
\begin{equation}
	\|B_5(f_\theta) - B_5(f_\theta^\epsilon)\|_{op} = \OM(\|f_\theta - f_\theta^\epsilon\|_{\infty}).	
\end{equation}
For the sixth term, following the similar argument as the third term and using the boundedness of $\nabla_\theta T_\theta(z)$, we have that 
\begin{equation}
	\|B_6(f_\theta) - B_6(f_\theta^\epsilon)\|_{op} = \OM(\|f_\theta - f_\theta^\epsilon\|_{\infty} + \|\nabla^2 f_\theta - \nabla^2 f_\theta^\epsilon\|_{op, \infty}).	
\end{equation}
For the last term, following the similar argument as the third term and using the boundedness of $\nabla_\theta T_\theta(z)$, we have that 
\begin{equation}
	\|B_7(f_\theta) - B_7(f_\theta^\epsilon)\|_{op} = \OM(\|f_\theta - f_\theta^\epsilon\|_{\infty} + \|\nabla f_\theta - \nabla f_\theta^\epsilon\|_{2, \infty}).	
\end{equation}
Combing the above results, we obtain \eqref{eqn_appendix_proof_theorem_approximation_H_2}.

\clearpage
\section{eSIM appendix} \label{appendix_eSIM}
\subsection{Proof of Theorem \ref{theorem_consistency}}
In this section, we use $f^\mu_{\theta}$ to denote the Sinkhorn potential to $\OTgamma({T_\theta}_\sharp\mu, \beta)$.
This allows us to emphasize the continuity of its Fr\'echet derivative w.r.t. the underlying measure $\mu$.
Similarly, we write $\BM_\mu(f, \theta)$ and $\EM_\mu(f, \theta)$ instead of $\BM(f, \theta)$ and $\EM(f, \theta)$, which are used to characterize the fixed point property of the Sinkhorn potential.


To prove Theorem \ref{theorem_consistency}, we need the following lemmas.
\begin{lemma} \label{theorem_continuity_f}
	The Sinkhorn potential $f^\mu_{\theta}$ is Lipschitz continuous with respect to $\mu$: 
	\begin{equation}
	\|f^\mu_{\theta} - f^{\bar{\mu}}_{\theta}\|_{\infty} = \OM(d_{bl}(\mu, \bar{\mu})).
	\end{equation}	
\end{lemma}
\begin{lemma} \label{theorem_continuity_gradient}
	The gradient of the Sinkhorn potential $f^\mu_{\theta}$ is Lipschitz continuous with respect to $\mu$: 
	\begin{equation}
	\|\nabla f^\mu_{\theta} - \nabla f^{\bar{\mu}}_{\theta}\|_{2, \infty} = \OM(d_{bl}(\mu, \bar{\mu})).
	\end{equation}	
\end{lemma}
\begin{lemma} \label{theorem_continuity_Hessian}
	The Hessian of the Sinkhorn potential $f^\mu_{\theta}$ is Lipschitz continuous with respect to $\mu$: 
	\begin{equation}
	\|\nabla^2 f^\mu_{\theta} - \nabla^2 f^{\bar{\mu}}_{\theta}\|_{op, \infty} = \OM(d_{bl}(\mu, \bar{\mu})).
	\end{equation}	
\end{lemma}
\begin{lemma} \label{theorem_continuity_Df}
	The Fr\'echet derivative of the Sinkhorn potential $f^\mu_{\theta}$ w.r.t. the parameter $\theta$, i.e. 
	$Df^\mu_{\theta}$, is Lipschitz continuous with respect to $\mu$: 
	\begin{equation}
		\|Df^\mu_{\theta} - Df^{\bar{\mu}}_{\theta}\|_{op} = \OM(d_{bl}(\mu, \bar{\mu})).
	\end{equation}	
\end{lemma}
Once we have these lemmas, we can prove \ref{theorem_consistency} in the same way as the proof of \ref{thm_computability_of_SIM} in Appendix \ref{appendix_proof_thm_computability_of_SIM}.

\subsection{Proof of Lemma \ref{theorem_continuity_f}}
Note that from the definition of the bounded Lipschitz distance, we have
\begin{align}
d_{bl}(\alpha, \bar{\alpha}) =&\ \sup_{\|\xi\|_{bl}\leq 1} |\langle \xi, \alpha\rangle - \langle \xi, \bar \alpha \rangle| = \sup_{\|\xi\|_{bl}\leq 1} |\langle \xi\circ T_\theta, \mu\rangle - \langle \xi\circ T_\theta, \bar \mu \rangle| \notag\\
\leq&\ \sup_{\|\xi\|_{bl}\leq 1} \|\xi\circ T_\theta\|_{bl} \cdot d_{bl}(\mu, \bar\mu) \leq G_T\cdot d_{bl}(\mu, \bar\mu), \label{eqn_appendix_d_bl_alpha_d_bl_mu}
\end{align}
where we use $\|\xi\circ T_\theta\|_{lip} \leq G_T$ from Assumption \ref{ass_Lipschitz_continuity_T}.

We have Lemma \ref{theorem_continuity_f} by combining the above results with the following lemma.
\begin{lemma}
	Under Assumption \ref{ass_bounded_c} and Assumption \ref{ass_bounded_infty_c_gradient}, the Sinkhorn potential is Lipschitz continuous with respect to the bounded Lipschitz metric: Given measures $\alpha$, $\alpha'$ and $\beta$, we have
	\begin{align*}
	\|f_{\alpha, \beta} - f_{\alpha', \beta}\|_\infty \leq G_{bl}  d_{bl}(\alpha', \alpha)\quad \mathrm{and} \quad
	\|g_{\alpha, \beta} - g_{\alpha', \beta'}\|_\infty \leq G_{bl} d_{bl}(\alpha', \alpha).
	\end{align*}
	where $G_{bl} = {2\gamma\exp(2M_c/\gamma)G'_{bl}}/{(1-\lambda^2)}$ with $G'_{bl} = \max\{\exp(3M_c/\gamma), {2G_c\exp(3M_c/\gamma)}/{\gamma}\}$ and $\lambda = \frac{\exp(M_c/\gamma) - 1}{\exp(M_c/\gamma) + 1}$.
\end{lemma}
\begin{proof}
	Let $(f, g)$ and $(f', g')$ be the Sinkhorn potentials to $\OTgamma(\alpha, \beta)$ and $\OTgamma(\alpha', \beta)$ respectively.\\
	Denote $u \defi \exp(f/\gamma)$, $v \defi \exp(g/\gamma)$ and $u' \defi \exp(f'/\gamma)$, $v' \defi \exp(g'/\gamma)$.
	From Lemma \ref{lemma_sinkhorn_potential_bound}, $u$ is bounded in terms of the $L^\infty$ norm:
	\begin{equation*}
	\|u\|_\infty = \max_{x\in\XM} |u(x)| = \max_{x\in\XM} \exp(f/\gamma) \leq \exp(2M_c/\gamma),
	\end{equation*} 
	which also holds for $v, u', v'$.
	Additionally, from Lemma \ref{lemma_lipschitz_sinkhorn_potential}, $\nabla u$ exists and $\|\nabla u\|$ is bounded:
	\begin{equation*}
	\max_x \|\nabla u(x)\| = 	\max_x \frac{1}{\gamma}|u(x)|\|\nabla f(x)\|\leq \frac{1}{\gamma}\|u(x)\|_\infty\max_x\|\nabla f(x)\|\leq 
	\frac{G_c\exp(2M_c/\gamma)}{\gamma}.
	\end{equation*}
	Define the mapping $A_{\alpha} \mu \defi 1/(L_\alpha \mu)$ with 
	\begin{equation*}
	L_\alpha \mu = \int_\XM l(\cdot, y)\mu(y)\dB \alpha(y),
	\end{equation*}
	where $l(x, y) \defi \exp(-c(x, y)/\gamma)$. 
	From Assumption \ref{ass_bounded_c}, we have $\|l\|_\infty\leq\exp(M_c/\gamma)$ and from Assumption \ref{ass_bounded_infty_c_gradient} we have $\|\nabla_x l(x, y)\|\leq \exp(M_c/\gamma)\frac{G_c}{\gamma}$.
	From the optimality condition of $f$ and $g$, we have $v = A_{\alpha} u$ and $u = A_{\beta} v$. Similarly, $v' = A_{\alpha'} u'$ and $u' = A_{\beta} v'$.
	Recall the definition of the Hilbert metric in \eqref{eqn_hilbert_metric}.
	Note that $d_H(\mu, \nu) = d_H(1/\mu, 1/\nu)$ if $\mu(x)>0$ and $\nu(x)>0$ for all $x\in\XM$ and hence $d_H(L_\alpha\mu, L_\alpha\nu) = d_H(A_\alpha\mu, A_\alpha\nu)$.
	We recall the result in \eqref{eqn_contraction_under_hilbert_metric} using the above notations.
	\begin{lemma}[Birkhoff-Hopf Theorem \cite{lemmens2012nonlinear}, see Lemma B.4 in \cite{NIPS2019_9130}] 
		\label{lemma_Birkhoff-Hopf}
		Let $\lambda = \frac{\exp(M_c/\gamma) - 1}{\exp(M_c/\gamma) + 1}$ and $\alpha\in\MM_1^+(\XM)$. Then for every $u, v\in\CM(\XM)$, such that $u(x)>0, v(x)>0$ for all $x\in\XM$, we have
		\begin{equation*}
		d_H(L_\alpha u, L_\alpha v)\leq \lambda d_H(u, v).
		\end{equation*}
	\end{lemma}
	Note that $$\|\log\mu-\log\nu \|_\infty\leq d_H(\mu, \nu) = \|\log\mu - \log \nu\|_\infty + \|\log\nu - \log \mu\|_\infty\leq2\|\log\mu-\log\nu \|_\infty.$$
	In the following, we derive upper bound for $d_H(\mu, \nu)$ and use such bound to analyze the Lipschitz continuity of the Sinkhorn potentials $f$ and $g$.\\
	Construct $\tilde{v} \defi A_{\alpha} u'$.
	Using the triangle inequality (which holds since $v(x), v'(x), \tilde{v}(x) >0$ for all $x\in\XM$), we have
	\begin{align*}
	d_H(v, v')\leq d_H(v, \tilde{v}) + d_H(\tilde{v}, v') \leq
	\lambda d_H(u, u') + d_H(\tilde{v}, v'),
	\end{align*}
	where the second inequality is due to Lemma \ref{lemma_Birkhoff-Hopf}.
	Note that $u' = A_{\beta} v'$.
	Apply Lemma \ref{lemma_Birkhoff-Hopf} again to obtain
	\begin{equation*}
	d_H(u, u') \leq \lambda d_H(v, v').
	\end{equation*}
	Together, we obtain 
	\begin{equation*}
	d_H(v, v') \leq \lambda^2d_H(v, v') + d_H(\tilde{v}, v') + \lambda d_H(\tilde{u}, u') \leq \lambda^2d_H(v, v') + d_H(\tilde{v}, v'),
	\end{equation*}
	which leads to
	\begin{equation*}
	d_H(v, v') \leq \frac{1}{1- \lambda^2}[d_H(\tilde{v}, v')].
	\end{equation*}
	
	To bound $d_H(\tilde{v}, v')$, observe the following:
	\begin{align}
	d_H(v', \tilde v) =& d_H(L_{\alpha'} u', L_{\alpha} u') \leq 2\|\log L_{\alpha'} u' - \log L_{\alpha} u'\|_\infty \notag\\
	=& 2\max_{x\in\XM}| \nabla \log(a_x) ([L_{\alpha'} u'](x) - [L_{\alpha} u'](x))| = 2\max_{x\in\XM} \frac{1}{a_x} |[L_{\alpha'} u'](x) - [L_{\alpha} u'](x)|\notag\\
	\leq& 2\max\{\|1/L_{\alpha'} u'\|_\infty, \|1/L_{\alpha} u'\|_\infty\}\|L_{\alpha'} u' - L_{\alpha} u'\|_\infty \label{appendix_proof_i},
	\end{align}
	where $a_x\in[[L_{\alpha'} u'](x), [L_{\alpha} u'](x)]]$ in the second line is from the mean value theorem.
	Further, in the inequality we use $\max\{\|1/L_{\alpha} u'\|_\infty, \|1/L_{\alpha} u'\|_\infty\} = \max\{\|A_{\alpha'} u'\|_\infty, \|A_{\alpha} u'\|_\infty\} \leq \exp(2M_c/\gamma)$.
	Consequently, all we need to bound is the last term $\|L_{\alpha'} u' - L_{\alpha} u'\|_\infty$.
	
	We first note that $\forall x\in\XM$, $\|l(x, \cdot)u'(\cdot)\|_{bl}<\infty$: In terms of $\|\cdot\|_\infty$
	\begin{equation*}
	\|l(x, \cdot)u'(\cdot)\|_\infty \leq \|l(x, \cdot)\|_\infty\|u'\|_\infty\leq \exp(3M_c/\gamma) <\infty.
	\end{equation*}
	In terms of $\|\cdot\|_{lip}$, we bound
	\begin{align*}
	\|l(x, \cdot)u'(\cdot)\|_{lip} &\leq \|l(x, \cdot)\|_\infty\|u'\|_{lip} + \|l(x, \cdot)\|_{lip}\|u'\|_{\infty}\\
	&\leq \exp(M_c/\gamma)\frac{G_c\exp(2M_c/\gamma)}{\gamma} + \exp(M_c/\gamma)\frac{G_c}{\gamma}\exp(2M_c/\gamma) = \frac{2G_c\exp(3M_c/\gamma)}{\gamma}< \infty.
	\end{align*}	
	Together we have $\|l(x, y)u'(y)\|_{bl} \leq \max\{\exp(3M_c/\gamma), \frac{2G_c\exp(3M_c/\gamma)}{\gamma}\}$.
	From the definition of the operator $L_{\alpha}$, we have
	\begin{align*}
	\|L_{\alpha'} u' - L_{\alpha} u'\|_\infty = \max_x |\int_\XM l(x, y)u'(y)\dB\alpha'(y) - \int_\XM l(x, y)u'(y)\dB\alpha(y)|
	\leq \|l(x, y)u'(y)\|_{bl} d_{bl}(\alpha', \alpha).
	\end{align*}
	All together we derive
	\begin{equation*}
	d_H(v', v) \leq \frac{2\exp(2M_c/\gamma)\|l(x, y)u'(y)\|_{bl}}{1-\lambda^2}\cdot  d_{bl}(\alpha', \alpha) \quad(\lambda = \frac{\exp(M_c/\gamma) - 1}{\exp(M_c/\gamma) + 1}).
	\end{equation*}
	Further, since $d_H(v', v) \geq \|\log v'-\log v \|_\infty =  \frac{1}{\gamma}\|f'- f \|_\infty$, we have the result:
	\begin{equation*}
	\|f'- f \|_\infty\leq \frac{2\gamma\exp(2M_c/\gamma)\|l(x, y)u'(y)\|_{bl}}{1-\lambda^2}  \cdot d_{bl}(\alpha', \alpha).
	\end{equation*}
	Similar argument can be made for $\|g'- g \|_\infty$.
\end{proof}

\begin{lemma}[Boundedness of the Sinkhorn Potentials] \label{lemma_sinkhorn_potential_bound}
	Let $(f, g)$ be the Sinkhorn potentials of problem \eqref{eqn_OTepsilon_dual} and assume that there exists $x_o\in\XM$ such that $f(x_o) = 0$ (otherwise shift the pair by $f(x_o)$). Then, under Assumption \ref{ass_bounded_c}, $\|f\|_\infty \leq 2M_c$ and $\|g\|_\infty \leq 2M_c$.
\end{lemma}
Next, we analyze the Lipschitz continuity of the Sinkhorn potential $f_{\alpha,\beta}(x)$ with respect to the input $x$.

Assumption \ref{ass_bounded_infty_c_gradient} implies that $\nabla_x c(x,y)$ exists and for all $x, y\in\XM, \|\nabla_x c(x,y)\|\leq G_c$.
It further ensures the Lipschitz-continuity of the Sinkhorn potential.
\begin{lemma}[Proposition 12 of \cite{feydy2019interpolating}]
	\label{lemma_lipschitz_sinkhorn_potential}
	Under Assumption \ref{ass_bounded_infty_c_gradient}, for a fixed pair of measures $(\alpha, \beta)$, the corresponding Sinkhorn potential $f:\XM\rightarrow\RBB$ is $G_c$-Lipschitz continuous, i.e. for $x_1, x_2 \in \XM$
	\begin{equation}
	|f_{\alpha, \beta}(x_1) - f_{\alpha, \beta}(x_2)|\leq G_c\|x_1 - x_2\|.
	\end{equation}
	Further, the gradient $\nabla f_{\alpha, \beta}$ exists {at every point $x \in \XM$}, and $\|\nabla f_{\alpha, \beta}(x)\|\leq G_c, \forall x\in\XM$.
\end{lemma}

\begin{lemma}\label{lemma_lipschitz_sinkhorn_potential_gradient}
	Under Assumption \ref{ass_bounded_infty_c_hessian}, for a fixed pair of measures $(\alpha, \beta)$, the gradient of the corresponding Sinkhorn potential $f:\XM\rightarrow\RBB$ is Lipschitz continuous,
	\begin{equation}
	\|\nabla f(x_1) - \nabla f(x_2)\|\leq L_f\|x_1 - x_2\|,
	\end{equation}
	where $L_f \defi \frac{4G_c^2}{\gamma}+L_c$.
\end{lemma}
\subsection{Proof of Lemma \ref{theorem_continuity_gradient}}
We have Lemma \ref{theorem_continuity_gradient} by combining \eqref{eqn_appendix_d_bl_alpha_d_bl_mu} with the following lemma.
\begin{lemma}[Lemma \ref{theorem_continuity_gradient} restated]
	Under Assumption \ref{ass_bounded_c} and Assumption \ref{ass_bounded_infty_c_gradient}, the gradient of the Sinkhorn potential is Lipschitz continuous with respect to the bounded Lipschitz metric: Given measures $\alpha$, $\alpha'$ and $\beta$, we have
	\begin{align*}
	\|\nabla f_{\alpha, \beta} - \nabla f_{\alpha', \beta}\|_\infty = \OM\big(d_{bl}(\alpha', \alpha)\big)
	\end{align*}
\end{lemma}
\begin{proof}
	From the optimality condition of the Sinkhorn potentials, one have that
	\begin{equation} \label{eqn_optimality_condition_sinkhorn_potential}
		\int_\XM h_{\alpha, \beta}(x, y)\dB \beta(y) = 1, \textrm{with}\ h_{\alpha, \beta}(x, y) \defi \exp\left(\frac{1}{\gamma}\big(f_{\alpha, \beta}(x) + g_{\alpha, \beta}(y) - c(x, y)\big)\right).
	\end{equation}
	Taking gradient w.r.t. $x$ on both sides of the above equation, the expression of $\nabla f_{\alpha, \beta}$ writes 
	\begin{align}
		\nabla f_{\alpha, \beta}(x) = \frac{\int_\XM h_{\alpha, \beta}(x, y) \nabla_x c(x, y) \dB\beta(y)}{\int_{\XM} h_{\alpha, \beta}(x, y)\dB\beta(y)} = \int_\XM h_{\alpha, \beta}(x, y)\nabla_x c(x, y) \dB\beta(y) \label{eqn_sinkhorn_potential_gradient_x}.
	\end{align}
	We have that $\forall x, y,$ $h_{\alpha, \beta}(x)$ is Lipschitz continuous w.r.t. $\alpha$, which is due to the boundedness of $f_{\alpha, \beta}(x)$, $g_{\alpha, \beta}(y)$ and the ground cost $c$, and Lemma \ref{theorem_continuity_f}.
	Further, since $\|\nabla_x c(x, y)\|$ is bounded from Assumption \ref{ass_bounded_infty_c_gradient} we have the Lipschitz continuity of $\nabla f_{\alpha, \beta}$ w.r.t. $\alpha$, i.e.
	\begin{equation*}
	\|\nabla f_{\alpha, \beta}(x) - \nabla f_{\alpha', \beta}(x)\| = \OM\big(d_{bl}(\alpha', \alpha)\big).
	\end{equation*}
\end{proof}
\subsection{Proof of Lemma \ref{theorem_continuity_Hessian}}
We have Lemma \ref{theorem_continuity_Hessian} by combining \eqref{eqn_appendix_d_bl_alpha_d_bl_mu} with the following lemma.
\begin{lemma}[Lemma \ref{theorem_continuity_Hessian} restated]
	Under Assumptions \ref{ass_bounded_c}-\ref{ass_bounded_infty_c_hessian}, the Hessian of the Sinkhorn potential is Lipschitz continuous with respect to the bounded Lipschitz metric: Given measures $\alpha$, $\alpha'$ and $\beta$, we have
	\begin{align*}
	\|\nabla^2 f_{\alpha, \beta} - \nabla^2 f_{\alpha', \beta}\|_{op, \infty} = \OM\big(d_{bl}(\alpha', \alpha)\big)
	\end{align*}
\end{lemma}
\begin{proof}
	Taking gradient w.r.t. $x$ on both sides of \eqref{eqn_sinkhorn_potential_gradient_x}, the expression of $\nabla^2 f_{\alpha, \beta}$ writes
	\begin{align*}\label{eqn_sinkhorn_potential_Hessian_x}
		\nabla^2 f_{\alpha, \beta}(x) = \int_\XM \frac{1}{\gamma}h_{\alpha, \beta}(x, y)(\nabla f_{\alpha, \beta}(x) - \nabla_x c(x, y))[\nabla_{x} c(x, y)]^\top + h_{\alpha, \beta}(x, y)\nabla^2_{xx} c(x, y) \dB\beta(y).
	\end{align*}
	From the boundedness of $h_{\alpha, \beta}$, $\nabla f_{\alpha, \beta}$ and $\nabla_x c$, and the Lipschitz continuity of $h_{\alpha, \beta}$ and $\nabla f_{\alpha, \beta}$ w.r.t. $\alpha$, we have that the first integrand of $\nabla^2 f_{\alpha, \beta}$ is Lipschitz continuous w.r.t. $\alpha$.
	Further, combining the boundedness of $\|\nabla^2_{xx} c(x, y)\|$ from Assumption \ref{ass_bounded_infty_c_hessian} and the Lipschitz continuity of $h_{\alpha, \beta}$ w.r.t. $\alpha$, we have the Lipschitz continuity of $\nabla^2 f_{\alpha, \beta}(x)$, i.e.
	\begin{equation*}
	\|\nabla^2 f_{\alpha, \beta}(x) - \nabla^2 f_{\alpha', \beta}(x)\| = \OM\big(d_{bl}(\alpha', \alpha)\big).
	\end{equation*}
\end{proof}

\subsection{Proof of Lemma \ref{theorem_continuity_Df}}
	The optimality of the Sinkhorn potential $f^\mu_{\theta}$ can be restated as
	\begin{equation} \label{eqn_appendix_fix_point_B_f}
		f^\mu_{\theta} = \BM_\mu(f^\mu_{\theta}, \theta),
	\end{equation}
	where we recall the definition of $\BM_\mu$ in \eqref{eqn_main_fix_point}
	\begin{equation}
		\BM_\mu(f, \theta) = \AM\big(\AM(f, {T_{\theta}}_\sharp \mu), \beta_\mu\big).
	\end{equation}
	Note that it is possible that $\beta_\mu$ depends on $\mu$, which is the case in $\OTgamma(\alpha_\theta, \alpha_{\theta^t})$ as $\beta_\mu = \alpha_{\theta^t} = {T_{\theta^t}}_\sharp \mu$.\\
	Under Assumption \ref{ass_bounded_c}, let $\lambda = \frac{e^{M_c/\gamma} - 1}{e^{M_c/\gamma} + 1}$.
	By repeating the above fixed point iteration \eqref{eqn_appendix_fix_point_B_f} $l = \lceil\log_{\lambda}\frac{1}{3}\rceil/2$ times, we have that
	\begin{equation} \label{eqn_appendix_fix_point_E_f}
		f^\mu_{\theta} = \EM_\mu(f^\mu_{\theta}, \theta),
	\end{equation}
	where $\EM_\mu(f, \theta) = \BM_\mu^l(f, \theta) = \BM_\mu\big(\cdots\BM_\mu(f, \theta)\cdots, \theta\big)$ is the $l$ times composition of $\BM_\mu$ in its first variable.
	We have from \eqref{eqn_contraction_EM}
	\begin{equation}\label{eqn_contraction_of_E}
		||D^1\EM_\mu(f, \theta)\|_{op} \leq \frac{2}{3},
	\end{equation}
	where we recall for a (linear) operator $\CM:\CM(\XM)\rightarrow\CM(\XM)$, $\|\CM\|_{op} \defi \max_{f\in\CM(\XM)} \frac{\|\CM f\|_\infty}{\|f\|_\infty}$.

	Let $h\in\RBB^d$ be any direction.
	Taking Fr\'echet derivative w.r.t. $\theta$ on both sides of \eqref{eqn_appendix_fix_point_E_f}, we derive
	\begin{equation} \label{eqn_appendix_fix_point_DE_f}
		Df^\mu_{\theta}[h] = D_1\EM_\mu(f^{\mu}_{\theta}, \theta)\big[Df^\mu_{\theta}[h]\big] + D_2\EM_\mu(f^\mu_{\theta}, \theta)[h].
	\end{equation}
	Using the triangle inequality, we bound
	\begin{equation} \label{eqn_proof_main}
	\begin{aligned}
	& \|Df^\mu_{\theta}[h] - Df^{\bar{\mu}}_{\theta}[h]\|_\infty\\
	\leq\quad & \|D_1\EM_\mu(f^{\mu}_{\theta}, \theta)\big[Df^\mu_{\theta}[h]\big]  - D_1\EM_{\bar{\mu}}(f_{\theta, \bar{\mu}}, \theta)\big[ Df^{\bar{\mu}}_{\theta}[h]\big]\|_\infty \\
	\quad & +  \|D_2\EM_\mu(f^\mu_{\theta}, \theta)[h] - D_2\EM_{\bar{\mu}}(f^{\bar{\mu}}_{\theta}, \theta)[h]\|_\infty\\
	\leq\quad & \|D_1\EM_{\mu}(f^{\mu}_{\theta}, \theta)\big[Df^\mu_{\theta}[h]\big]  - D_1\EM_{\mu}(f^{\mu}_{\theta}, \theta) \big[Df^{\bar{\mu}}_{\theta}[h]\big]\|_\infty && \textcircled{1}\\ 
	\quad&+ \|D_1\EM_{\mu}(f_{\theta, {\mu}}, \theta) \big[Df^{\bar{\mu}}_{\theta}[h]\big]  - D_1\EM_{\mu}(f_{\theta, \bar{\mu}}, \theta)\big[ Df^{\bar{\mu}}_{\theta}[h]\big]\|_\infty && \textcircled{2}\\
	\quad& + \|D_1\EM_{\mu}(f_{\theta, \bar{\mu}}, \theta) \big[Df^{\bar{\mu}}_{\theta}[h]\big]  - D_1\EM_{\bar{\mu}}(f_{\theta, \bar{\mu}}, \theta) \big[Df^{\bar{\mu}}_{\theta}[h]\big]\|_\infty && \textcircled{3}\\
	\quad&+ \|D_2\EM_{\mu}(f^\mu_{\theta}, \theta)[h] - D_2\EM_{\bar{\mu}}(f^{\bar{\mu}}_{\theta}, \theta)[h]\|_\infty. && \textcircled{4}
	\end{aligned}
	\end{equation}
	The following subsections analyze $\textcircled{1}$ to $\textcircled{4}$ individually. In summary, we have
	\begin{equation}
		\textcircled{1} \leq \frac{2}{3} \|Df^\mu_{\theta}[h] - Df^{\bar{\mu}}_{\theta}[h]\|_\infty,
	\end{equation}
	and \textcircled{2}, \textcircled{3}, \textcircled{4} are all of order $\OM(d_{bl}(\mu, \bar{\mu})\cdot\|h\|)$.
	Therefore we conclude
	\begin{equation}
		\frac{1}{3}\|Df^\mu_{\theta}[h] - Df^{\bar{\mu}}_{\theta}[h]\|_\infty = \OM(d_{bl}(\mu, \bar{\mu})\cdot\|h\|) \Rightarrow \|Df^\mu_{\theta} - Df^{\bar{\mu}}_{\theta}\|_{op} = \OM(d_{bl}(\mu, \bar{\mu})).
	\end{equation}
	
	\subsubsection{Bounding \textcircled{1}}
	From the linearity of $D_1\EM_\mu(f^{\mu}_{\theta}, \theta)$ and \eqref{eqn_contraction_of_E}, we bound
	\begin{align*}
		\textcircled{1}	=\ & \|D_1\EM_\mu(f^{\mu}_{\theta}, \theta)\big[Df^\mu_{\theta}[h] - Df^{\bar{\mu}}_{\theta}[h]\big]\|_\infty \\
		\leq\ & \|D_1\EM_\mu(f^{\mu}_{\theta}, \theta)\|_{op}\|Df^\mu_{\theta}[h] - Df^{\bar{\mu}}_{\theta}[h]\|_\infty \leq \frac{2}{3} \|Df^\mu_{\theta}[h] - Df^{\bar{\mu}}_{\theta}[h]\|_\infty.
	\end{align*}
	
	\subsubsection{Bounding \textcircled{2}}
	From Lemma \ref{lemma_appendix_lipschitz_continuity_D_1_B}, we know that $D_1\BM_\mu(f, \theta)$ is Lipschitz continuous w.r.t. its first variable:
	\begin{equation} \label{eqn_appendix_bound_2}
		\|D_1\BM_\mu(f, \theta)  - D_1\BM_{\mu}(f', \theta)\|_{op} = \OM(\|f - f'\|_\infty).
	\end{equation}
	
	Recall that $\EM_\mu(f, \theta) = \BM_\mu^l(f, \theta)$.
	Using the chain rule of the Fr\'echet derivative, we have
	\begin{equation}
		D_1\EM_{\mu}(f, \theta) = D_1\BM_\mu^l(f, \theta) = D_1\BM_{\mu}\big(\BM_{\mu}^{l-1}(f, \theta), \theta\big) \circ D_1\BM_{\mu}^{l-1}(f, \theta).
	\end{equation}
	Consequently, we can bound \textcircled{2} in a recursive way: for any two functions $f, f'\in\CM(\XM)$
	\begin{align*}
		&\|D_1\BM_{\mu}^l(f, \theta) - D_1\BM_{\mu}^l(f', \theta)\|_{op}\\
		=&\ \|D_1\BM_{\mu}\big(\BM_{\mu}^{l-1}(f, \theta), \theta\big) \circ D_1\BM_{\mu}^{l-1}(f, \theta) - D_1\BM_{\mu}\big(\BM_{\mu}^{l-1}(f', \theta), \theta\big) \circ D_1\BM_{\mu}^{l-1}(f', \theta)\|_{op}\\
		\leq&\ \|D_1\BM_{\mu}\big(\BM_{\mu}^{l-1}(f, \theta), \theta\big)\circ \big(D_1\BM_{\mu}^{l-1}(f, \theta) - D_1\BM_{\mu}^{l-1}(f', \theta)\big)\|_{op} \\
		&\ +\|\bigg(D_1\BM_{\mu}\big(\BM_{\mu}^{l-1}(f, \theta), \theta\big) - D_1\BM_{\mu}\big(\BM_{\mu}^{l-1}(f', \theta), \theta\big)\bigg)\circ D_1\BM_{\mu}^{l-1}(f', \theta)\|_{op} \\
		\leq&\ \|D_1\BM_{\mu}\big(\BM_{\mu}^{l-1}(f, \theta), \theta\big)\|_{op} \|D_1\BM_{\mu}^{l-1}(f, \theta) - D_1\BM_{\mu}^{l-1}(f', \theta)\|_\infty\\
		&\ + \OM(\|\BM_{\mu}^{l-1}(f, \theta) - \BM_{\mu}^{l-1}(f', \theta)\|_\infty\cdot\|D_1\BM_{\mu}^{l-1}(f', \theta)\|_{op})\\
		=&\ \OM(\|f - f'\|_\infty) + \|D_1\BM_{\mu}^{l-1}(f, \theta) - D_1\BM_{\mu}^{l-1}(f', \theta)\|_\infty,
	\end{align*}
	where in the first inequality we use the triangle inequality, in the second inequality, we use the definition of $\|\cdot\|_{op}$ and \eqref{eqn_appendix_bound_2}, and in the last equality we use \eqref{eqn_appendix_bound_2} and the fact that $\BM^k$ is Lipschitz continuous with respect its first argument for any finite $k$ (see Lemma \ref{appendix_lemma_property_3}).
	Besides, since $f^\mu_{\theta}$ is continuous with respect to $\mu$ (see Lemma \ref{theorem_continuity_f}), we have 
	\begin{equation}
		\|D_1\BM^l(f^{\mu}_{\theta}, \theta) - D_1\BM^l(f^{\bar \mu}_{\theta}, \theta)\|_{op} = \OM(d_{bl}(\mu, \bar \mu)).
	\end{equation}
	
	We then show that $\|Df^{\bar{\mu}}_{\theta}[h]\|_\infty = \OM(\|h\|_\infty)$: Using \eqref{eqn_appendix_fix_point_DE_f}, we have that
	\begin{equation*}
		\|Df^{\bar{\mu}}_{\theta}[h]\|_\infty \leq \frac{2}{3} \|Df^{\bar{\mu}}_{\theta}[h]\|_\infty + \|D_2\EM_\mu(f^\mu_{\theta}, \theta)[h]\|_\infty \Rightarrow \|Df^{\bar{\mu}}_{\theta}[h]\|_\infty \leq 3\|D_2\EM_\mu(f^\mu_{\theta}, \theta)\|_{op}\|[h]\|_\infty.
	\end{equation*}
	Lemma \ref{appendix_lemma_property_4} shows that $\|D_2\EM_\mu(f^\mu_{\theta}, \theta)\|_{op}$ is bounded and therefore we have
	\begin{equation} \label{eqn_Df_h_bounded_by_h}
		\|Df^{\bar{\mu}}_{\theta}[h]\|_\infty = \OM(\|h\|_\infty).
	\end{equation}
	
	Combining the above results, we obtain
	\begin{equation*}
		\textcircled{2} \leq \|D_1\BM^l(f^{\mu}_{\theta}, \theta) - D_1\BM^l(f^{\bar \mu}_{\theta}, \theta)\|_{op} \|Df^{\bar{\mu}}_{\theta}[h]\|_\infty = \OM(d_{bl}(\mu, \bar \mu)\cdot\|h\|_\infty).
	\end{equation*}
	\subsubsection{Bounding \textcircled{3}}
	Denote $\omega_y(x) = \exp(-\frac{c(x, y)}{\gamma})\exp(\bar f(x)/\gamma)$.
	Assume that $\|\bar f\|_\infty\leq M_c$ and $\|\nabla \bar f\|_{2,\infty}\leq G_f$. Then we have for any $y\in\XM$, 
	\begin{equation}\label{eqn_appendix_omega_y_bl_norm}
		\|\omega_y\|_\infty \leq \exp(M_c/\gamma), \quad \|\nabla \omega_y\|_{2,\infty}\leq\exp(M_c/\gamma)(G_c + G_f)/\gamma.
	\end{equation}
	Therefore, $\|\omega_y\|_{bl} = \max\{\exp(M_c/\gamma), \exp(M_c/\gamma)(G_c + G_f)/\gamma\}$ is bounded (recall the definition of bounded Lipschitz norm in Theorem \ref{theorem_consistency}).
	Besides, for any $y\in\XM$, $\omega_y(x)$ is positive and bounded away from zero
	\begin{equation}\label{eqn_appendix_omega_y_strictly_positive}
		\omega_y(x) \geq \exp(-2M_c/\gamma).
	\end{equation}
	
	For a fixed measure $\kappa$ and $g\in\CM(\XM)$, we compute that 
	\begin{equation} \label{eqn_appendix_bound_3_1}
	D_1\AM(\bar f, \kappa)[g] = \frac{\int_\XM\omega_y(x)g(x)\dB\kappa(x)}{\int_\XM\omega_y(x)\dB\kappa(x)}.
	\end{equation}
	This expression allows us to bound for two measures $\kappa$ and $\kappa'$
	\begin{align*}
		&\ \|\big(D_1\AM(\bar f, \kappa) - D_1\AM(\bar f, \kappa')\big)[g]\|_\infty = \|\frac{\int_\XM \omega_y(x)g(x)\dB\kappa(x)}{\int_\XM \omega_y(x)\dB\kappa(x)} - \frac{\int_\XM \omega_y(x)g(x)\dB\kappa'(x)}{\int_\XM \omega_y(x)\dB\kappa'(x)}\|_\infty\\
		\leq&\ \|\frac{\int_\XM \omega_y(x)g(x)\dB\kappa(x)}{\int_\XM \omega_y(x)\dB\kappa(x)} - \frac{\int_\XM \omega_y(x)g(x)\dB\kappa(x)}{\int_\XM \omega_y(x)\dB\kappa'(x)}\|_\infty  + \|\frac{\int_\XM \omega_y(x)g(x)\dB\kappa(x)}{\int_\XM \omega_y(x)\dB\kappa'(x)} - \frac{\int_\XM \omega_y(x)g(x)\dB\kappa'(x)}{\int_\XM \omega_y(x)\dB\kappa'(x)}\|_\infty.
	\end{align*}
	We now bound these two terms individually.
	For the first term, we have
	\begin{align*}
		&\ \|\frac{\int_\XM \omega_y(x)g(x)\dB\kappa(x)}{\int_\XM \omega_y(x)\dB\kappa(x)} - \frac{\int_\XM \omega_y(x)g(x)\dB\kappa(x)}{\int_\XM \omega_y(x)\dB\kappa'(x)}\|_\infty\\
		\leq&\ \|\int_\XM \omega_y(x)g(x)\dB\kappa(x)\|_\infty\|\frac{\int_\XM \omega_y(x)\left[\dB\kappa(x) - \dB\kappa'(x)\right]}{\int_\XM \omega_y(x)\dB\kappa(x)\int_\XM \omega_y(x)\dB\kappa'(x)}\|_\infty\\
		\leq&\ \|\omega_y\|_\infty\cdot\|g\|_\infty\cdot\|\omega_y(x)\|_{bl}\cdot d_{bl}(\kappa, \kappa')\cdot\exp(4M_c/\gamma) = \OM(\|g\|_\infty\cdot d_{bl}(\kappa, \kappa')),
	\end{align*}
	where we use \eqref{eqn_appendix_omega_y_bl_norm} and \eqref{eqn_appendix_omega_y_strictly_positive} in the last equality.
	For the second term, we bound
	\begin{align*}
		\|\frac{\int_\XM \omega_y(x)g(x)\dB\kappa(x)}{\int_\XM \omega_y(x)\dB\kappa(x)} - \frac{\int_\XM \omega_y(x)g(x)\dB\kappa'(x)}{\int_\XM \omega_y(x)\dB\kappa(x)}\|_\infty
		\leq\ \|\frac{\int_\XM \omega_y(x)g(x)[\dB\kappa(x) - \dB\kappa'(x)]}{\int_\XM \omega_y(x)\dB\kappa(x)} \|_\infty \\
		\leq \exp(M_c/\gamma)\cdot\|\omega_y(x)\|_{bl}\cdot\|g\|_{bl}\cdot d_{bl}(\kappa, \kappa') = \OM(\|g\|_{bl}\cdot d_{bl}(\kappa, \kappa')). 
	\end{align*}
	Combining the above inequalities, we have
	\begin{equation} \label{eqn_proof_bound_3}
		\|\big(D_1\AM(\bar f, \kappa) - D_1\AM(\bar f, \kappa')\big)[g]\|_\infty = \OM(\|g\|_{bl}\cdot d_{bl}(\kappa, \kappa')).
	\end{equation}

	Denote $\alpha = {T_\theta}_\sharp\mu$ and $\bar \alpha = {T_\theta}_\sharp \bar\mu$.
	From the chain rule of the Fr\'echet derivative, we compute
	\begin{align*}
		&\ \|\big(D_1\BM_{\mu}(f, \theta) - D_1\BM_{\bar \mu}(f, \theta)\big)[g]\|_\infty \\
		=&\ \big\|\bigg(D_1\AM\big(\AM(f, \alpha), \beta_{\mu}\big)\circ D_1\AM(f, \alpha) - D_1\AM\big(\AM(f, \bar \alpha), \beta_{\bar\mu}\big)\circ D_1\AM(f, \bar \alpha)\bigg) [g]\big\|_\infty\\
		\leq&\ \big\|D_1\AM\big(\AM(f, \alpha), \beta_{\mu}\big)\big[\big( D_1\AM(f, \alpha)  -  D_1\AM(f, \bar \alpha)\big) [g]\big]\big\|_\infty \\
		&+ \big\|\bigg(D_1\AM\big(\AM(f, \alpha), \beta_{\mu}\big) - D_1\AM\big(\AM(f, \alpha), \beta_{\bar\mu}\big)\bigg) \big[D_1\AM(f, \bar \alpha) [g]\big]\big\|_\infty\\
		&+ \big\|\bigg(D_1\AM\big(\AM(f, \alpha), \beta_{\bar \mu}\big) - D_1\AM\big(\AM(f, \bar \alpha), \beta_{\bar\mu}\big)\bigg) \big[D_1\AM(f, \bar \alpha) [g]\big]\big\|_\infty.
	\end{align*}
	We now bound these three terms one by one.\\
	For the first term, use \eqref{eqn_contraction_of_E} to derive
	\begin{align*}
		\big\|D_1\AM\big(\AM(f, \alpha), \beta_{\mu}\big)\big[\big( D_1\AM(f, \alpha)  -  D_1\AM(f, \bar \alpha)\big) [g]\big]\big\|_\infty \\
		\leq \|D_1\AM(f, \alpha) [g] - D_1\AM(f, \bar \alpha) [g]\|_\infty &= \OM(\|g\|_{bl}\cdot d_{bl}(\alpha, \bar{\alpha})),
	\end{align*}
	where we use {$\|D_1\AM\big(\AM(f, \alpha), \beta_\mu\big)\|_{op} \leq 1$} \eqref{eqn_proof_bound_0} and \eqref{eqn_proof_bound_3} in the second equality.
	
	Combining the above result with \eqref{eqn_appendix_d_bl_alpha_d_bl_mu} gives
	\begin{equation*}
		\big\|D_1\AM\big(\AM(f, \alpha), \beta_{\mu}\big)\big[\big( D_1\AM(f, \alpha)  -  D_1\AM(f, \bar \alpha)\big) [g]\big]\big\|_\infty = \OM(\|g\|_{bl}\cdot d_{bl}(\mu, \bar\mu)).
	\end{equation*}
	For the second term, use \eqref{eqn_proof_bound_3} to derive
	\begin{align*}
		\big\|\bigg(D_1\AM\big(\AM(f, \alpha), \beta_{\mu}\big) - D_1\AM\big(\AM(f, \alpha), \beta_{\bar\mu}\big)\bigg) \big[D_1\AM(f, \bar \alpha) [g]\big]\big\|_\infty \\
		= \OM(\|D_1\AM(f, \bar \alpha) [g]\|_{bl}\cdot d_{bl}(\beta_{\mu}, \beta_{\bar \mu})).
	\end{align*}
	We now bound $\|D_1\AM(f, \bar \alpha) [g]\|_{bl}$.
	From \eqref{eqn_proof_bound_0}, we have that $\|D_1\AM(f, \bar \alpha) [g]\|_{\infty}\leq \|g\|_\infty$.
	Besides, note that $D_1\AM(f, \bar \alpha)[g]$ is a function mapping from $\XM$ to $\RBB$ and recall the expression of $D_1\AM(f, \bar \alpha)[g]$ in \eqref{eqn_appendix_bound_3_1}.
	To show that $D_1\AM(f, \bar \alpha)[g](y)$ is Lipschitz continuous w.r.t. $y$, we use the similar argument as \eqref{eqn_Lipschitz_continuous_fraction_of_functions}:
	Under Assumption \ref{ass_bounded_c} and assume that $\|f\|_\infty\leq M_c$,
	the numerator and denominator of \eqref{eqn_Lipschitz_continuous_fraction_of_functions} are both Lipschitz continuous w.r.t. $y$ and bounded; the denominator is positive and bounded away from zero.
	Consequently, we can bound for any $y\in\XM$
	\begin{equation} \label{eqn_appendix_D1AMg_bound}
		{\|\nabla_y D_1\AM(f, \bar \alpha) [g]	(y)\| \leq 2\exp(4M_c/\gamma)\|g\|_\infty\cdot G_c,}
	\end{equation}
	and therefore
	\begin{equation*}
		\big\|\bigg(D_1\AM\big(\AM(f, \alpha), \beta_{\mu}\big) - D_1\AM\big(\AM(f, \alpha), \beta_{\bar\mu}\big)\bigg) \big[D_1\AM(f, \bar \alpha) [g]\big]\big\|_\infty = \OM(\|g\|_\infty\cdot d_{bl}(\beta_{\mu}, \beta_{\bar \mu})).
	\end{equation*}
	For the third term, first note that we can use \eqref{eqn_appendix_d_bl_alpha_d_bl_mu} and the mean value theorem to bound 
	{\begin{equation} \label{eqn_appendix_AM_continuous_mu}
	\|\AM(f, \alpha) - \AM(f, \bar \alpha)\|_\infty = \OM(\max_{y\in\XM}\|\omega_y\|_{bl}\cdot d_{bl}(\alpha, \bar \alpha)) = \OM(d_{bl}(\mu, \bar{\mu})).
	\end{equation}}
	Hence, we use Lemma \ref{lemma_appendix_bound_2_lemma} to derive
	\begin{align*}
		\big\|\bigg(D_1\AM\big(\AM(f, \alpha), \beta_{\bar \mu}\big) - D_1\AM\big(\AM(f, \bar \alpha), \beta_{\bar\mu}\big)\bigg) \big[D_1\AM(f, \bar \alpha) [g]\big]\big\|_\infty\\
		 = \OM(\|\AM(f, \alpha) - \AM(f, \bar \alpha)\|_\infty\cdot \|D_1\AM(f, \bar \alpha) [g]\|_\infty) = \OM(\|g\|_\infty\cdot d_{bl}(\mu, \bar{\mu})),
	\end{align*}
	where we use {\eqref{eqn_appendix_AM_continuous_mu}} and the fact that $\|D_1\AM(f, \bar \alpha)\|_{op}$ is bounded in the last equality.\\
	Combing the above three results, we have 
	\begin{equation} \label{eqn_appendix_bound_3}
		\|\big(D_1\BM_{\mu}(f, \theta) - D_1\BM_{\bar \mu}(f, \theta)\big)[g]\|_\infty = \OM(\|g\|_{bl}\cdot d_{bl}(\mu, \bar{\mu})).
	\end{equation}

	Recall that $\EM_\mu(f, \theta) = \BM_\mu^l(f, \theta)$.
	Using the chain rule of the Fr\'echet derivative, we have
	\begin{equation}
D_1\EM_{\mu}(f, \theta) = D_1\BM_\mu^l(f, \theta) = D_1\BM_{\mu}\big(\BM_{\mu}^{l-1}(f, \theta), \theta\big) \circ D_1\BM_{\mu}^{l-1}(f, \theta).
	\end{equation}
	Denote $g = Df^{\bar{\mu}}_{\theta}[h]$. We can bound \textcircled{3} in the following way:
	\begin{align*}
	\textcircled{3}=&\ \|D_1\BM_{\mu}\big(\BM_{\mu}^{l-1}(f, \theta), \theta\big)\big[D_1\BM^{l-1}_{\mu}(f, \theta)[g]\big] - D_1\BM_{\bar\mu}\big(\BM_{\bar \mu}^{l-1}(f, \theta), \theta\big)[D_1\BM^{l-1}_{\bar\mu}(f, \theta)[g]]\|_\infty\\
	\leq&\ \|D_1\BM_{\mu}\big(\BM_{\mu}^{l-1}(f, \theta), \theta\big)\big[\big(D_1\BM^{l-1}_{\mu}(f, \theta) - D_1\BM^{l-1}_{\bar\mu}(f, \theta)\big)[g]\big]\|_\infty \\
	&\ +\|\bigg(D_1\BM_{\mu}\big(\BM_{\mu}^{l-1}(f, \theta), \theta\big) - D_1\BM_{\mu}\big(\BM_{\bar \mu}^{l-1}(f, \theta), \theta\big)\bigg)\big[D_1\BM^{l-1}_{\bar\mu}(f, \theta)[g]\big]\|_\infty \\
	&\ +\|\bigg(D_1\BM_{\mu}\big(\BM_{\bar \mu}^{l-1}(f, \theta), \theta\big) - D_1\BM_{\bar \mu}\big(\BM_{\bar \mu}^{l-1}(f, \theta), \theta\big)\bigg) \big[D_1\BM^{l-1}_{\bar\mu}(f, \theta)[g]\big]\|_\infty\\
	\leq&\ \|D_1\BM_{\mu}\big(\BM_{\mu}^{l-1}(f, \theta), \theta\big)\|_{op} \|\big(D_1\BM^{l-1}_{\mu}(f, \theta) - D_1\BM^{l-1}_{\bar\mu}(f, \theta)\big)[g]\|_\infty && \#1\\
	&\ + \OM(\|\BM_{\mu}^{l-1}(f, \theta) - \BM_{\bar \mu}^{l-1}(f, \theta)\|_\infty\cdot\|D_1\BM_{\bar \mu}^{l-1}(f, \theta)[g]\|_\infty) && \#2\\
	&\ + \OM(\|D_1\BM^{l-1}_{\bar\mu}(f, \theta)[g]\|_{bl}\cdot d_{bl}(\mu, \bar\mu)), && \#3
	\end{align*}
	where in the first inequality we use the triangle inequality, in the second inequality we use the definition of $\|\cdot\|_{op}$, \eqref{eqn_appendix_bound_2} and \eqref{eqn_appendix_bound_3}.
	We now analyze the R.H.S. of the above inequality one by one.
	For the first term, use $\|D_1\BM_{\mu}\big(\BM_{\mu}^{l-1}(f, \theta), \theta\big)\|_{op}\leq 1$ and then use \eqref{eqn_appendix_bound_3}. We have
	\begin{equation*}
		\#1 \leq \|\big(D_1\BM_{\mu}(f, \theta) - D_1\BM_{\bar\mu}(f, \theta)\big)[g]\|_\infty = \OM(\|g\|_{bl}\cdot d_{bl}(\mu, \bar{\mu})).
	\end{equation*}
	For the second term, note that $\BM_\mu^k$ is the composition of the terms $\AM(f, \alpha)$ and $\AM(f, \beta_\mu)$. Using a similar argument like \eqref{eqn_appendix_AM_continuous_mu}, for any finite $k$, we have
	\begin{equation*}
		\|\BM_{\mu}^{l-1}(f, \theta) - \BM_{\bar \mu}^{l-1}(f, \theta)\|_\infty = \OM(d_{bl}(\mu, \bar{\mu})).
	\end{equation*}
	Together with the fact that  $\|D_1\BM(f, \theta)\|_{op}\leq 1$, we have
	\begin{equation*}
		\#2 = \OM(\|g\|_{\infty}\cdot d_{bl}(\mu, \bar{\mu})).
	\end{equation*}
	Finally, for the third term, note that $\BM_\mu$ is the composition of the terms $\AM(f, \alpha)$ and $\AM(f, \beta_\mu)$. Using a similar argument like \eqref{eqn_appendix_D1AMg_bound} to bound
	\begin{equation*}
		\#3 = \OM(\|g\|_{\infty}\cdot d_{bl}(\mu, \bar{\mu})).
	\end{equation*}
	Combining these three results, we have
	\begin{equation} \label{eqn_appendix_bound_3_last}
		\textcircled{3} = \|\big(D_1\BM_{\mu}^l(f, \theta) - D_1\BM_{\bar \mu}^l(f, \theta)\big)[g] \|_\infty = \OM(\|g\|_{bl}\cdot d_{bl}(\mu, \bar{\mu})).
	\end{equation}
	
	We now bound $\|Df^{\bar{\mu}}_{\theta}[h]\|_{bl}$ ($g = Df^{\bar{\mu}}_{\theta}[h]$). 
	From the fixed point definition of the Sinkhorn potential in \eqref{eqn_appendix_fix_point_B_f}, we can compute the Fr\'echet derivative $Df_\theta^\mu$ by
	\begin{equation}
		Df_\theta^\mu = D_1\AM\big(\AM(f_\theta^\mu, \alpha_\theta), \beta_\mu\big)\circ D_1\AM(f_\theta^\mu, \alpha_\theta)\circ Df_\theta^\mu + D_1\AM\big(\AM(f_\theta^\mu, \alpha_\theta), \beta_\mu\big) \circ D_2 \tilde \AM(f_\theta^\mu, \theta),
	\end{equation}
	where we recall $\tilde{\AM}(f, \theta) \defi \AM(f, \alpha_\theta)$.
	For any direction $h\in\RBB^d$ and any $y\in\XM$, $Df_\theta^\mu[h]$ is a function with its gradient bounded by
	\begin{align*}
		\|\nabla_y Df_\theta^\mu[h](y)\| \leq \|\nabla_y \bigg(D_1\AM\big(\AM(f_\theta^\mu, \alpha_\theta), \beta_\mu\big)\bigg[D_1\AM(f_\theta^\mu, \alpha_\theta)\big[Df_\theta^\mu[h]\big]\bigg]\bigg)(y)\| &&\#1\\
		+ \|\nabla_y \left(D_1\AM\big(\AM(f_\theta^\mu, \alpha_\theta), \beta_\mu\big) \big[D_2 \tilde \AM(f_\theta^\mu, \theta)[h]\big]\right)(y)\|. &&\#2
	\end{align*}
	We now bound the R.H.S. individually:\\
	For $\#1$, take $\bar f = \AM[f, \alpha_\theta]$, $\kappa = \beta_\mu$ and $g = D_1\AM(f_\theta^\mu, \alpha_\theta)\big[Df_\theta^\mu[h]\big]$ in \eqref{eqn_appendix_bound_3_1}.
	Using \eqref{eqn_appendix_D1AMg_bound} and \eqref{eqn_Df_h_bounded_by_h}, we have
	\begin{equation}
		\#1 = \OM(\|g\|_\infty) = \OM(\|Df_\theta^\mu[h]\|_\infty) = \OM(\|h\|).
	\end{equation}
	For $\#2$, take $\bar f = \AM[f, \alpha_\theta]$, $\kappa = \beta_\mu$ and $g = D_2 \tilde \AM(f_\theta^\mu, \theta)[h]$ in \eqref{eqn_appendix_bound_3_1}. Using \eqref{eqn_appendix_D1AMg_bound} and \eqref{eqn_appendix_D2A_bounded}, we have
	\begin{equation}
		\#2 = \OM(\|g\|_\infty) = \OM(\|D_2 \tilde \AM(f_\theta^\mu, \theta)[h]\|_\infty) = \OM(\|h\|).
	\end{equation}
	Combining these two bounds, we have 
	\begin{equation} \label{eqn_appendix_Df_h_bl}
		\|Df_\theta^\mu[h]\|_{bl} = \OM(\|h\|).
	\end{equation}
	By plugging the above result to \eqref{eqn_appendix_bound_3_last}, we bound
	\begin{equation}
	\textcircled{3} = \|\big(D_1\BM_{\mu}^l(f, \theta) - D_1\BM_{\bar \mu}^l(f, \theta)\big)[g] \|_\infty = \OM(d_{bl}(\mu, \bar\mu)\cdot\|h\|).
	\end{equation}

	\subsubsection{Bounding \textcircled{4}}
	We have from the triangle inequality
	\begin{equation} \label{eqn_appendix_bounding_4}
	\textcircled{4} \leq \|D_2\EM_\mu(f^\mu_{\theta}, \theta)[h] - D_2\EM_{\bar{\mu}}(f^\mu_{\theta}, \theta)[h]\|_\infty + \|D_2\EM_{\bar \mu}(f^{\mu}_{\theta}, \theta)[h] - D_2\EM_{\bar{\mu}}(f^{\bar{\mu}}_{\theta}, \theta)[h]\|_\infty.
	\end{equation}
	We analyze these two terms on the R.H.S..
	
	For the first term of \eqref{eqn_appendix_bounding_4}, use the chain rule of Fr\'echet derivative to compute
	\begin{equation}
	D_2\EM_\mu(f, \theta)[h] = D_1\BM_{\mu}\big(\BM_{\mu}^{l-1}(f, \theta), \theta\big)\big[D_2\BM^{l-1}_\mu(f, \theta)[h]\big] + D_2\BM_\mu\big(\BM_{\mu}^{l-1}(f, \theta), \theta\big)[h].
	\end{equation}
	Consequently, we can bound 
	\begin{align*}
	&\|\big(D_2\EM_\mu(f, \theta) - D_2\EM_{\bar \mu}(f, \theta)\big)[h]\|_\infty \\
	\leq &\| D_1\BM_{\mu}\big(\BM_{\mu}^{l-1}(f, \theta), \theta\big) \big[D_2\BM^{l-1}_\mu(f, \theta)[h]\big] - D_1\BM_{\bar\mu}\big(\BM_{\bar \mu}^{l-1}(f, \theta), \theta\big) \big[D_2\BM^{l-1}_{\bar\mu}(f, \theta)[h]\big] \|_\infty && \#1 \\
	&+ \|D_2\BM_{\mu}\big(\BM_{\mu}^{l-1}(f, \theta), \theta\big)[h] - D_2\BM_{\bar\mu}\big(\BM_{\bar \mu}^{l-1}(f, \theta), \theta\big)[h]\|_\infty. && \#2
	\end{align*}
	We analyze \#1 and \#2 individually.
	\paragraph{Bounding \#1.} We first note that $\AM(f, \alpha)$ is Lipschitz continuous w.r.t. $\alpha$ (see also \eqref{eqn_appendix_AM_continuous_mu}):
	\begin{equation}
	\|\AM(f, \alpha) - \AM(f, \alpha')\|_\infty \leq \exp(2M_c/\gamma)\cdot \|\omega_y\|_{bl} \cdot d_{bl}(\alpha, \alpha') = \OM(d_{bl}(\alpha, \alpha')),
	\end{equation}
	where in the equality we use \eqref{eqn_appendix_omega_y_bl_norm}.
	As $\BM_\mu^k$ is the composition of $\AM$, it is Lipschitz continuous with respect to $\mu$ for finite $k$.
	Note that the boundedness of $\|f\|_\infty$ and $\|\nabla f\|_\infty$ remains valid after the operator $\BM$ (Lemma \ref{lemma_appendix_bounded_function_sequence_from_sinkhorn_knopp} and (i) of Lemma \eqref{lemma_appendix_sinkhorn_potential_boundedness}).
	We then bound
	\begin{align*}
	\#1 \leq&\ \| D_1\BM_{\mu}\big(\BM_{\mu}^{l-1}(f, \theta), \theta\big) \big[\big(D_2\BM^{l-1}_\mu(f, \theta) -  D_2\BM^{l-1}_{\bar\mu}(f, \theta)\big)[h]\big] \|_\infty \\
	&\ + \|\bigg(D_1\BM_{\mu}\big(\BM_{\mu}^{l-1}(f, \theta), \theta\big) - D_1\BM_{\mu}\big(\BM_{\bar \mu}^{l-1}(f, \theta), \theta\big)\bigg) \big[D_2\BM^{l-1}_{\bar\mu}(f, \theta)[h]\big]\|_\infty \\
	&\ + \|\bigg(D_1\BM_{\mu}\big(\BM_{\bar \mu}^{l-1}(f, \theta), \theta\big) - D_1\BM_{\bar \mu}\big(\BM_{\bar \mu}^{l-1}(f, \theta), \theta\big)\bigg) \big[D_2\BM^{l-1}_{\bar\mu}(f, \theta)[h]\big]\|_\infty \\
	\leq&\ \|D_1\BM_{\mu}\big(\BM_{\mu}^{l-1}(f, \theta), \theta\big)\|_{op}\| D_2\BM^{l-1}_\mu(f, \theta)[h] - D_2\BM^{l-1}_{\bar\mu}(f, \theta)[h]\|_\infty \\
	&\ + \OM(\|\BM_{\mu}^{l-1}(f, \theta) - \BM_{\bar\mu}^{l-1}(f, \theta)\|_\infty\cdot\|D_2\BM^{l-1}_{\bar\mu}(f, \theta)[h]\|_\infty) \\
	&\ + \OM(d_{bl}(\mu, \bar \mu)\cdot \|D_2\BM^{l-1}_{\bar\mu}(f, \theta)[h]\|_\infty) \\
	\leq&\ \| D_2\BM^{l-1}_\mu(f, \theta)[h] - D_2\BM^{l-1}_{\bar\mu}(f, \theta)[h]\|_\infty + \OM(d_{bl}(\mu, \bar \mu)\cdot\|h\|),
	\end{align*}
	where in the second inequality we use the definition of $\|\cdot\|_{op}$, \eqref{eqn_appendix_bound_2} and \eqref{eqn_appendix_bound_3}, and in the last inequality we use the fact that $\|D_1\BM_{\mu}(f, \theta)\|_{op}\leq 1$, $\BM_{\mu}^k$ is Lipschitz continuous with respect to $\mu$ for finite $k$ (see the discussion above) and that $\|D_2\BM^{l-1}_{\bar\mu}(f, \theta)\|_{op}$ is bounded (see Lemma \ref{appendix_lemma_property_4}.
	\paragraph{Bounding \#2.} 
	To make the dependences of $\AM$ on $\theta$ and $\mu$ explicit, we denote $$\hat \AM(f, \theta, \mu) = \AM(f, {T_\theta}_\sharp\mu).$$	
	To bound the second term, we first establish that for any $k\geq 0$, $\nabla \BM_{\mu}^{k+1}(f, \theta)$ is Lipschitz continuous w.r.t. $\mu$, i.e.
	\begin{equation} \label{eqn_appendix_nabla_BM_k_continuous_mu}
	\|\nabla \BM_{\mu}^{k+1}(f, \theta) - \nabla \BM_{\bar \mu}^{k+1}(f, \theta)\|_{2, \infty} = \OM(d_{bl}(\mu, \bar \mu)),
	\end{equation}
	as follows: First note that $\nabla \hat \AM(f, \theta, \mu)$ is Lipschitz continuous w.r.t. $\mu$, i.e.
	\begin{equation} \label{eqn_appendix_nabla_AM_continuous_mu}
	\|\nabla \hat \AM(f, \theta, \mu)(y) - \nabla \hat \AM(f, \theta, \bar \mu)(y)\| = \OM(d_{bl}(\mu, \bar \mu)).
	\end{equation}
	This is because for any $y\in\XM$ (note that $\hat \AM(f, \theta, \mu)(\cdot):\XM\rightarrow\RBB$ is a function of $y$),
	\begin{align*}
	&\ \|\nabla \hat \AM(f, \theta, \mu)(y) - \nabla \hat \AM(f, \theta, \bar \mu)(y)\|\\
	=&\ \|\frac{\int_{\XM} \omega_y(x)\nabla_{1} c(y, x) \dB \alpha_\theta(x)}{\int_{\XM} \omega_y(x)\dB \alpha_\theta(x)} - \frac{\int_{\XM} \omega_y(x)\nabla_{1} c(y, x) \dB \bar \alpha_\theta(x)}{\int_{\XM} \omega_y(x)\dB \bar \alpha_\theta(x)}\|\\
	\leq&\  \|\frac{\int_{\XM} \omega_y(x)\nabla_{1} c(y, x) \big(\dB \alpha_\theta(x) - \dB \bar \alpha_\theta(x)\big)}{\int_{\XM} \omega_y(x)\dB \alpha_\theta(x)}\|\\
	&\ +\|\int_{\XM} \omega_y(x)\nabla_{1} c(y, x) \dB \bar \alpha_\theta(x)\|\cdot\|\frac{\int_{\XM} \omega_y(x)\big(\dB \alpha_\theta(x) - \dB \bar \alpha_\theta(x)\big)}{\int_{\XM} \omega_y(x)\dB \alpha_\theta(x)\int_{\XM} \omega_y(x)\dB \bar \alpha_\theta(x)}\| \\
	=&\ {\OM(d_{bl}(\mu, \bar \mu))}.
	\end{align*}
	Here in the last equality, we use the facts that $\|\omega_y(\cdot)\nabla_{1} c(y, \cdot)\|_{bl}$ and $\|\omega_y\|_{bl}$ are bounded, and $\int_{\XM} \omega_y(x)\dB \alpha_\theta(x)$ is strictly positive and bounded away from zero.
	Recall that $\BM_{\mu}(f, \theta) = \AM(\hat \AM(f, \theta, \mu), \beta_{\mu})$.
	We can then prove \eqref{eqn_appendix_nabla_BM_k_continuous_mu} by bounding
	\begin{align*}
	&\ \|\nabla \BM_{\mu}^{k+1}(f, \theta) - \nabla \BM_{\bar \mu}^{k+1}(f, \theta)\| \\
	=&\  \|\nabla \AM(\hat \AM(\BM^{k}_{\mu}(f, \theta), \theta, \mu), \beta_\mu) - \nabla \AM(\hat \AM(\BM^{k}_{\bar \mu}(f, \theta), \theta, \bar \mu), \bar \beta_\mu)\|\\
	\leq&\  \|\nabla  \AM(\hat \AM(\BM^{k}_{\mu}(f, \theta), \theta, \mu), \beta_\mu) - \nabla  \AM(\hat \AM(\BM^{k}_{\mu}(f, \theta), \theta, \mu), \bar \beta_\mu)\| && \&1\\
	&\ + \|\nabla  \AM(\hat \AM(\BM^{k}_{\mu}(f, \theta), \theta, \mu), \bar \beta_\mu) - \nabla  \AM(\hat \AM(\BM^{k}_{\mu}(f, \theta), \theta, \bar \mu), \bar \beta_\mu)\| && \&2\\
	&\ + \|\nabla  \AM(\hat \AM(\BM^{k}_{\mu}(f, \theta), \theta, \bar \mu), \bar \beta_\mu) - \nabla  \AM(\hat \AM(\BM^{k}_{\bar \mu}(f, \theta), \theta, \bar \mu), \bar \beta_\mu)\| && \&3\\
	=&\  \OM(d_{bl}(\mu, \bar \mu)) 
	\end{align*}
	Here we bound \&1 using \eqref{eqn_appendix_nabla_AM_continuous_mu}, the Lipschitz continuity of $\nabla {\AM}$ w.r.t. its second variable; we bound \&2 using the Lipschitz continuity of $\nabla \hat \AM$ w.r.t. its first variable and \eqref{eqn_appendix_AM_continuous_mu}, the Lipschitz continuity of $\hat \AM$ w.r.t. $\mu$; we bound \&3 using \eqref{eqn_appendix_AM_continuous_mu}, the Lipschitz continuity of $\hat \AM$ w.r.t. $\mu$, and the fact that $\BM_{\mu}^k$ is the composition of the terms $\AM(f, \alpha)$ and $\AM(f, \beta_\mu)$.
	
	We then establish that $D_2\BM_\mu(f, \theta)$ is Lipschitz continuous w.r.t. $\mu$.
	\begin{assumption} \label{ass_nabla_theta_nabla_z_T_bounded}
		$\|\nabla_z[\nabla_{\theta}T_{\theta}(z)]\|_{op}$ is bounded
	\end{assumption}
	\begin{lemma}\label{lemma_appendix_bounding_4}
		Assume that $\|f\|_\infty\leq M_c$, $\|\nabla f\|_{2,\infty}\leq G_f$, $\|\nabla^2 f\|_{op,\infty}\leq L_f$
		Under Assumptions \ref{ass_Lipschitz_continuity_T}, \ref{ass_nabla_theta_nabla_z_T_bounded} and \ref{ass_bounded_c}, we have
		\begin{equation} \label{eqn_appendix_bound_4_2}
		\|D_2\BM_\mu(f, \theta) - D_2\BM_{\bar \mu}(f, \theta)\|_{op} = \OM(d_{bl}(\mu, \bar{\mu})).
		\end{equation}
	\end{lemma}
	\begin{proof}
		Denote $\omega_y(x) = \exp\left(\frac{-c(x, y) + f(x)}{\gamma}\right)$ and $$\phi_y(z) = [\nabla_{\theta}T_{\theta}(z)]^\top\left[-\nabla_1 c(T_\theta(z), y) + \nabla f(T_\theta(z))\right]$$ where $\nabla_{\theta}T_{\theta}(z)$ denotes the Jacobian matrix of $T_\theta(z)$ with respect to $\theta$.
		
		The Fr\'echet derivative $D_2\hat \AM(f, \theta,\mu)[h]$ can be computed by
		\begin{equation} \label{eqn_appendix_D2_A_h}
		D_2\hat \AM(f, \theta,\mu)[h] = \frac{\int_\XM \omega_y\big(T_{\theta}(z)\big) \langle \phi_y(z), h\rangle\dB\mu(z)}{\int_\XM \omega_y\big(T_{\theta}(z)\big) \dB\mu(z)}.
		\end{equation}
		Recall that $\|f\|_\infty\leq M_c$, $\|\nabla f\|_{2,\infty}\leq G_f$.
		Using the above expression we can bound
		\begin{align*}
		&\ \|\big(D_2\hat \AM(f, \theta,\mu) - D_2\hat \AM(f, \theta, \bar\mu)\big)[h]\|_\infty \notag\\
		=&\ \big\|\frac{\int_\XM \omega_y\big(T_{\theta}(z)\big) \langle \phi_y(z), h\rangle\dB\mu(z)}{\int_\XM \omega_y\big(T_{\theta}(z)\big) \dB\mu(z)} - \frac{\int_\XM \omega_y\big(T_{\theta}(z)\big) \langle \phi_y(z), h\rangle\dB\bar{\mu}(x)}{\int_\XM \omega_y\big(T_{\theta}(z)\big) \dB\bar{\mu}(x)}\big\|_\infty \notag\\
		\leq&\ \big\|\frac{\int_\XM \omega_y\big(T_{\theta}(z)\big) \langle \phi_y(z), h\rangle\dB\mu(z)}{\int_\XM \omega_y\big(T_{\theta}(z)\big) \dB\mu(z)} - \frac{\int_\XM \omega_y\big(T_{\theta}(z)\big) \langle \phi_y(z), h\rangle\dB\bar{\mu}(x)}{\int_\XM \omega_y\big(T_{\theta}(z)\big) \dB\mu(z)}\big\|_\infty \notag\\
		&\ + \big\|\frac{\int_\XM \omega_y\big(T_{\theta}(z)\big) \langle \phi_y(z), h\rangle\dB\bar{\mu}(x)}{\int_\XM \omega_y\big(T_{\theta}(z)\big) \dB\mu(z)} - \frac{\int_\XM \omega_y\big(T_{\theta}(z)\big) \langle \phi_y(z), h\rangle\dB\bar{\mu}(x)}{\int_\XM \omega_y\big(T_{\theta}(z)\big) \dB\bar{\mu}(x)}\big\|_\infty \notag\\
		=&\ \big\|\frac{\int_\XM \omega_y\big(T_{\theta}(z)\big) \langle \phi_y(z), h\rangle\left[\dB\mu(z) - \dB\bar{\mu}(x)\right]}{\int_\XM \omega_y\big(T_{\theta}(z)\big) \dB\mu(z)}\big\|_\infty \notag\\
		&\ + \big\|\frac{\int_\XM \omega_y\big(T_{\theta}(z)\big) \langle \phi_y(z), h\rangle\dB\bar{\mu}(x) \int_\XM \omega_y\big(T_{\theta}(z)\big) [\dB\bar{\mu}(x) -  \dB\mu(z)]}{\int_\XM \omega_y\big(T_{\theta}(z)\big) \dB\mu(z)\int_\XM \omega_y\big(T_{\theta}(z)\big) \dB\bar{\mu}(x)}\big\|_\infty \notag\\
		\leq&\ \exp(2M_c/\gamma)\cdot\|\omega_y\big(T_{\theta}(z)\big) \langle \phi_y(z), h\rangle\|_{bl}\cdot d_{bl}(\mu, \bar\mu) \notag\\
		&\  + \exp(5M_c/\gamma)\cdot\|\phi_y\|_\infty\cdot\|h\|_\infty\cdot\|\omega_y\|_{bl}\cdot d_{bl}(\mu, \bar\mu). \notag
		\end{align*}
		For the first term, note that $\|\omega_y\big(T_{\theta}(z)\big) \langle \phi_y(z), h\rangle\|_{bl}\leq \|\omega_y\|_{bl}\cdot\|\phi_y\|_{bl}\cdot\|h\|_\infty $ and $\|\omega_y\|_{bl}$ is bounded (see \eqref{eqn_appendix_omega_y_bl_norm}). We just need to bound $\|\phi_y\|_{bl}$.
		Under Assumption \ref{ass_Lipschitz_continuity_T} that $\|\nabla_{\theta}T_{\theta}(z)\|_{op}\leq G_T$, we clearly have that  $\|\phi_y\|_{\infty}$ is bounded.
		For $\|\phi_y\|_{lip}$, compute that 
		\begin{align*}
		\nabla_z \phi_y(z) = \nabla_z[\nabla_{\theta}T_{\theta}(z)]\times_1 \left[-\nabla_1 c(T_\theta(z), y) + \nabla f(T_\theta(z))\right] \\
		+ \nabla_{\theta}T_{\theta}(z)^\top \left[-\nabla^2_{11} c(T_\theta(z), y) + \nabla^2 f(T_\theta(z))\right]\nabla_{\theta}T_{\theta}(z).
		\end{align*}
		Recall that $\|\nabla^2 f(x)\|_{op}$ is bounded.
		Consequently, under Assumption \ref{ass_nabla_theta_nabla_z_T_bounded}, we can see that $\|\nabla_z \phi_y(z)\|$ is bounded.
		Together, $\|\phi_y\|_{bl}$ is bounded.
		As a result, we have
		\begin{equation}
		\|\big(D_2\hat \AM(f, \theta,\mu) - D_2\hat \AM(f, \theta, \bar\mu)\big)[h]\|_\infty = \OM(d_{bl}(\mu, \bar{\mu})\cdot \|h\|). \label{eqn_appendix_bound_4_1}
		\end{equation}

		Based on the above result, we can further bound
		\begin{align*}
		&\ \|\big(D_2\BM_\mu(f, \theta) - D_2\BM_{\bar \mu}(f, \theta)\big)[h]\|_\infty \\
		= &\ \|\bigg(D_1\AM\big(\hat \AM(f, \theta,\mu), \beta\big)\circ D_2 \hat \AM(f, \theta, \mu) - D_1\AM\big(\hat \AM(f, \theta, \bar\mu), \bar\beta\big)\circ D_2\hat \AM(f, \theta, \bar\mu)\bigg) [h]\|_\infty\\
		\leq&\ \|D_1\AM\big(\hat \AM(f, \theta,\mu), \beta\big)\big[\big( D_2 \hat \AM(f, \theta, \mu) - D_2\hat \AM(f, \theta, \bar\mu)\big) [h]\big]\|_\infty &&\#\#1\\
		&\ + \|\bigg(D_1\AM\big(\hat \AM(f, \theta,\mu), \beta\big) - D_1\AM\big(\hat \AM(f, \theta, \mu), \bar\beta\big)\bigg) \big[D_2\hat \AM(f, \theta, \bar\mu) [h]\big]\|_\infty &&\#\#2\\
		&\ + \|\bigg(D_1\AM\big(\hat \AM(f, \theta,\mu), \bar \beta\big) - D_1\AM\big(\hat \AM(f, \theta, \bar\mu), \bar\beta\big)\bigg) \big[D_2\hat \AM(f, \theta, \bar\mu) [h]\big]\|_\infty. &&\#\#3
		\end{align*}
		For the first term, use $\|D_1\AM\big(\hat \AM(f, \theta,\mu), \beta\big)\|_{op}\leq 1$ \eqref{eqn_proof_bound_0} and \eqref{eqn_appendix_bound_4_1} to bound
		\begin{align*}
		\#\#1
		\leq \|D_2 \hat \AM(f, \theta, \mu) [h] - D_2\hat \AM(f, \theta, \bar\mu) [h]\|_\infty
		= \OM(d_{bl}(\mu, \bar{\mu})\cdot \|h\|).
		\end{align*}
		For the second term, recall the expression of $D_2\hat \AM(f, \theta, \bar\mu) [h]$ in \eqref{eqn_appendix_D2_A_h}. 
		Under Assumption \ref{ass_bounded_c} and assume that $\|f\|_\infty\leq M_c$, one can see that $\|D_2\hat \AM(f, \theta, \bar\mu) [h]\|_{bl} = \OM(\|h\|)$.
		Further, use \eqref{eqn_proof_bound_3} and $d_{bl}(\beta, \bar{\beta}) = \OM\big(d_{bl}(\mu, \bar{\mu})\big)$ from \eqref{eqn_appendix_d_bl_alpha_d_bl_mu} to bound
		\begin{align*}
		\#\#2 = \OM(\|D_2\hat \AM(f, \theta, \bar\mu) [h]\|_{bl}\cdot d_{bl}(\beta, \bar{\beta})) = \OM(\|h\|\cdot d_{bl}(\mu, \bar{\mu})).
		\end{align*}
		For the third term, use Lemma \ref{lemma_appendix_bound_2_lemma} to bound
		\begin{align*}
		\#\#3
		= \OM(\|D_2\hat \AM(f, \theta, \bar\mu) [h]\|_\infty\cdot\|\hat \AM(f, \theta,\mu) - \hat \AM(f, \theta, \bar\mu)\|_\infty)
		= \OM(d_{bl}(\mu, \bar{\mu}) \cdot \|h\|),
		\end{align*}
		where we use $\|D_2\hat \AM(f, \theta, \bar\mu) [h]\|_\infty = \OM(\|h\|)$ and \eqref{eqn_appendix_AM_continuous_mu}.
		Altogether, we have 
		\begin{equation}
		\|D_2\BM_\mu(f, \theta)[h] - D_2\BM_{\bar \mu}(f, \theta)[h]\|_\infty = \OM(d_{bl}(\mu, \bar{\mu}) \cdot \|h\|).
		\end{equation}
	\end{proof}
	
	We are now ready to bound \#2.
	\begin{align*}
	\#2 \leq&\ \|D_2\BM_{\mu}\big(\BM_{\mu}^{l-1}(f, \theta), \theta\big)[h] - D_2\BM_{\mu}\big(\BM_{\bar \mu}^{l-1}(f, \theta), \theta\big)[h]\|_\infty \\
	&\ + \|D_2\BM_{\mu}\big(\BM_{\bar \mu}^{l-1}(f, \theta), \theta\big)[h] - D_2\BM_{\bar\mu}\big(\BM_{\bar \mu}^{l-1}(f, \theta), \theta\big)[h]\|_\infty\\
	=&\ \OM(\|\BM_{\mu}^{l-1}(f, \theta) - \BM_{\bar \mu}^{l-1}(f, \theta)\|_\infty + \|\nabla \BM_{\mu}^{l-1}(f, \theta) - \nabla \BM_{\bar \mu}^{l-1}(f, \theta)\|_{2, \infty}) \\
	&\ + \OM(d_{bl}(\mu, \bar \mu)\cdot\|h\|) \\
	=&\	\OM(d_{bl}(\mu, \bar \mu)\cdot\|h\|),
	\end{align*}
	where we use Lemma \ref{appendix_lemma_property_5} and \eqref{eqn_appendix_bound_4_2} \eqref{eqn_appendix_AM_continuous_mu} in the first equality.
	\paragraph{Combining \#1 and \#2.}
	Combining the above results, we yield
	\begin{equation*}
	\|D_2\BM^l_\mu(f, \theta)[h] - D_2\BM^l_{\bar \mu}(f, \theta)[h]\|_\infty\leq \| D_2\BM^{l-1}_\mu(f, \theta)[h] - D_2\BM^{l-1}_{\bar\mu}(f, \theta)[h]\|_\infty + \OM(d_{bl}(\mu, \bar \mu)\cdot\|h\|_\infty),
	\end{equation*}
	which, via recursion, implies that (recall that $D_2\EM_\mu(f, \theta)[h] = D_2\BM^l_\mu(f, \theta)[h]$)
	\begin{equation}
	\|D_2\EM_\mu(f, \theta)[h] - D_2\EM_{\bar \mu}(f, \theta)[h]\|_\infty = \OM(d_{bl}(\mu, \bar \mu)\cdot\|h\|).
	\end{equation}
	
	To bound the second term of \eqref{eqn_appendix_bounding_4}, compute the expression of $D_2\EM_{\bar \mu}(f, \theta)[h]$ via the chain rule: 
	\begin{equation} \label{eqn_appendix_bounding_4_2}
		D_2\EM_{\bar \mu}(f, \theta)[h] = D_1\BM_{\bar \mu}\big(\BM_{\bar \mu}^{l-1}(f, \theta), \theta\big)\big[D_2 \BM_{\bar \mu}^{l-1}(f, \theta)[h]\big] + D_2\BM_{\bar \mu}\big(\BM_{\bar \mu}^{l-1}(f, \theta), \theta\big)[h].
	\end{equation}
	Recall that $\EM_{\bar \mu}(f, \theta) = \BM_{\bar \mu}^l(f, \theta)$.
	We then show in an inductive manner that the second term of \eqref{eqn_appendix_bounding_4} is of order $\OM(d_{bl}(\mu, \bar \mu)\cdot\|h\|)$:
	For any finite $k\geq 1$, 
	\begin{equation}\label{eqn_appendix_proof_bounding_4_inductive}
		\|D_2\BM_{\bar \mu}^k(f^{\mu}_{\theta}, \theta)[h] - D_2\BM_{\bar \mu}^k(f^{\bar{\mu}}_{\theta}, \theta)[h]\|_\infty = \OM(d_{bl}(\mu, \bar \mu)\cdot\|h\|).
	\end{equation}
	For the base case when $l=1$, we only have the second term of \eqref{eqn_appendix_bounding_4_2} in $D_2\EM_{\bar \mu}(f, \theta)[h]$.
	Consequently, from Lemma \ref{appendix_lemma_property_5}, we have
	\begin{equation}\label{eqn_appendix_proof_bounding_4_inductive_base}
		\begin{aligned}
			\|D_2\BM_{\bar \mu}\big(\BM_{\bar \mu}^{l-1}(f_{\theta}^{\mu}, \theta), \theta\big) - D_2\BM_{\bar \mu}\big(\BM_{\bar \mu}^{l-1}(f_{\theta}^{\bar \mu}, \theta), \theta\big)\|_{op} \\
			 = \OM(\|\BM_{\bar \mu}^{l-1}(f_{\theta}^{\mu}, \theta) - \BM_{\bar \mu}^{l-1}(f_{\theta}^{\bar \mu}, \theta)\|_\infty + \|\nabla \BM_{\bar \mu}^{l-1}&(f_{\theta}^{\mu}, \theta) - \nabla \BM_{\bar \mu}^{l-1}(f_{\theta}^{\bar \mu}, \theta)\|_{2, \infty})
			 = \OM(d_{bl}(\mu, \bar \mu)),
		\end{aligned}
	\end{equation}
	where we use \eqref{eqn_appendix_nabla_BM_k_continuous_mu} in the second equality.\\
	Now assume that for $l=k$ the statement \eqref{eqn_appendix_proof_bounding_4_inductive} holds.
	For any two function $f, f'\in\CM(\XM)$, we bound 
	\begin{align*}
		&\ \|D_2\BM_{\bar \mu}^k(f, \theta)[h] - D_2\BM_{\bar \mu}^k(f', \theta)[h]\|_\infty\\
		\leq&\  \|D_1\BM_{\bar \mu}\big(\BM_{\bar \mu}^{l-1}(f, \theta), \theta\big)\big[D_2 \BM_{\bar \mu}^{l-1}(f, \theta)[h] - D_2 \BM_{\bar \mu}^{l-1}(f', \theta)[h]\big]\|_\infty \\
		&\ + \|\bigg(D_1\BM_{\bar \mu}\big(\BM_{\bar \mu}^{l-1}(f, \theta), \theta\big) - D_1\BM_{\bar \mu}\big(\BM_{\bar \mu}^{l-1}(f', \theta), \theta\big)\bigg)\big[D_2 \BM_{\bar \mu}^{l-1}(f', \theta)[h]\big]\|_\infty\\
		&\ + \|\bigg(D_2\BM_{\bar \mu}\big(\BM_{\bar \mu}^{l-1}(f, \theta), \theta\big) - D_2\BM_{\bar \mu}\big(\BM_{\bar \mu}^{l-1}(f', \theta), \theta\big)\bigg)[h]\|_\infty.\\
		\leq&\ \|\big(D_2 \BM_{\bar \mu}^{l-1}(f, \theta) - D_2 \BM_{\bar \mu}^{l-1}(f', \theta)\big)[h]\|_\infty && \|D_1\BM_{\bar \mu}(f, \theta)\|_{op}\leq 1\\
		&\ + \OM(\|\BM_{\bar \mu}^{l-1}(f, \theta) - \BM_{\bar \mu}^{l-1}(f', \theta)\|_\infty\cdot\|D_2 \BM_{\bar \mu}^{l-1}(f', \theta)[h]\|_\infty) && \mathrm{Lemma\ \ref{lemma_appendix_lipschitz_continuity_D_1_B}}\\
		&\ + \OM(d_{bl}(\mu, \bar \mu)\cdot\|h\|). && \eqref{eqn_appendix_proof_bounding_4_inductive_base}\\
		= &\ \OM((\|f - f'\|_\infty + \|\nabla f - \nabla f'\|_{2, \infty})\cdot\|h\|) && \mathrm{Lemma\ \ref{lemma_appendix_lipschitz_D2EM}} \\
		&\ \OM((\|f - f'\|_\infty)\cdot\|h\|) \\
		&\ \OM(d_{bl}(\mu, \bar \mu)\cdot\|h\|).
	\end{align*}
	Plug in $f = f_\theta^\mu$ and $f' = f_\theta^{\bar \mu}$ and use Lemmas \ref{theorem_continuity_f} and \ref{theorem_continuity_gradient}.
	We prove the statement \eqref{eqn_appendix_proof_bounding_4_inductive} holds for $l = k+1$.
	Consequently, we have that
	\begin{equation}
		\|D_2\EM_{\bar \mu}(f^{\mu}_{\theta}, \theta)[h] - D_2\EM_{\bar{\mu}}(f^{\bar{\mu}}_{\theta}, \theta)[h]\|_\infty = \OM(d_{bl}(\mu, \bar \mu)\cdot\|h\|).
	\end{equation}
	
	In conclusion, we have
	\begin{equation}
		\textcircled{4} = \OM(d_{bl}(\mu, \bar \mu)\cdot\|h\|).
	\end{equation}

\clearpage
\section{Experiment Details} \label{appendix_experiment}
We use the generator from DC-GAN \cite{radford2015unsupervised}.
And the adversarial ground cost $c_\xi$ in the form of
\begin{equation}
	c_\xi(x, y) = \|\phi_\xi(x) - \phi_\xi(y)\|_2^2,
\end{equation}
where $\phi_\xi:\RBB^q \rightarrow\RBB^{\hat q}$ is an encoder that maps the original data point (and the generated image) to a higher dimensional space ($\hat q > q$).
We pick $\phi_\xi$ to be an CNN with a similar structure as the discriminator of DC-GAN except that we discard the last layer which was used for classification.
Specifically, the networks used are given in Table \ref{table_encoder} and \ref{table_generator}.

We set the step size $\beta$ of SiNG to be $30$ and set the maximum allow Sinkhorn divergence in each iteration to be $0.1$. Note that the step size is set after the normalization in \eqref{eqn_update_direction}.
For Adam, RMSprop, and AMSgrad, we set all of their initial step sizes to be $1.0\times e^{-3}$, which is in general recommended by the GAN literature.
The minibatch sizes of both the real images and the generated images for each iteration are set to $3000$.
We uniformly set the $\gamma$ parameter in the objective (recall that $\FM(\alpha_\theta) = \SM_{c_\xi}(\alpha_\theta, \beta)$) and the constraint to $100$.

The code is in \url{https://github.com/shenzebang/Sinkhorn_Natural_Gradient}.

\begin{table}
	\caption{Structure of the encoder}
	\centering
	\begin{tabular}{c c c}
		Layer (type) &              Output Shape   &      Param \#\\
		Conv2d-1    & [-1, 64, 32, 32]  & 4,800     \\
		LeakyReLU-2   & [-1, 64, 32, 32]  & 0         \\
		Conv2d-3    & [-1, 128, 16, 16] & 204,800   \\
		BatchNorm2d-4  & [-1, 128, 16, 16] & 256       \\
		LeakyReLU-5   & [-1, 128, 16, 16] & 0         \\
		Conv2d-6    & [-1, 256, 8, 8]   & 819,200   \\
		BatchNorm2d-7  & [-1, 256, 8, 8]   & 512       \\
		LeakyReLU-8   & [-1, 256, 8, 8]   & 0         \\
		Conv2d-9    & [-1, 512, 4, 4]   & 3,276,800 \\
		BatchNorm2d-10 & [-1, 512, 4, 4]   & 1,024     \\
		LeakyReLU-11  & [-1, 512, 4, 4]   & 0
	\end{tabular}
\label{table_encoder}
\end{table} 

\begin{table}
	\caption{Structure of the generator}
	\centering
	\begin{tabular}{c c c}
		   Layer (type)    & Output Shape     & Param \# \\
		ConvTranspose2d-1  & [-1, 256, 4, 4]  & 262,144  \\
		  BatchNorm2d-2    & [-1, 256, 4, 4]  & 512      \\
		      ReLU-3       & [-1, 256, 4, 4]  & 0        \\
		ConvTranspose2d-4  & [-1, 128, 8, 8]  & 524,288  \\
		  BatchNorm2d-5    & [-1, 128, 8, 8]  & 256      \\
		      ReLU-6       & [-1, 128, 8, 8]  & 0        \\
		ConvTranspose2d-7  & [-1, 64, 16, 16] & 131,072  \\
		  BatchNorm2d-8    & [-1, 64, 16, 16] & 128      \\
		      ReLU-9       & [-1, 64, 16, 16] & 0        \\
		ConvTranspose2d-10 & [-1, 3, 32, 32]  & 3,072    \\
		     Tanh-11       & [-1, 3, 32, 32]  & 0
	\end{tabular}
\label{table_generator}
\end{table}

\clearpage
\section{PyTorch Implementation} \label{appendix_pytorch}
In this section, we focus on the empirical version of SiNG, where we approximate the gradient of the function $F$ by a minibatch stochastic gradient and approximate SIM by eSIM.
In this case, all components involved in the optimization procedure can be represented by finite dimensional vectors.

It is known that the stochastic gradient admits an easy implementation in PyTorch.
However, at the first sight, the computation of eSIM is quite complicated as it requires to construct two sequences $f^t$ and $g^t$ to estimate the Sinkhorn potential and the Fr\'echet derivative.
As we discussed earlier, it is well known that we can solve the inversion of a p.s.d. matrix via the Conjugate Gradient (CG) method with only matrix-vector-product operations. In particular, in this case, we no longer need to explicitly form eSIM in the computer memory.
Consequently, to implement the empirical version of SiNG using CG and eSIM, one can resort to the auto-differential mechanism provided by PyTorch:
First, we use existing PyTorch package like geomloss\footnote{https://www.kernel-operations.io/geomloss/} to compute the tensor $\fB$ representing the Sinkhorn potential $f_\theta^\epsilon$. Note the the sequence $f^t$ is constructed implicitly by calling geomloss.
We then use the ".detach()" function in PyTorch to maintain only the value of the $\fB$ while discarding all of its "grad\_fn" entries.
We then enable the "autograd" mechanism is PyTorch and run several loops of Sinkhorn mapping $\AM(f, \alpha_\theta)$ ($\AM(f, \alpha_{\theta^t})$) so that the output tensor now records all the dependence on the parameter $\theta$ via the implicitly constructed computational graph.
We can then easily compute the matrix-vector-product use the Pearlmutter’s algorithm (Pearlmutter, 1994).

\clearpage
\end{document}